%% file: main.tex
\documentclass[twoside,11pt]{article}

\usepackage{algorithm}

\usepackage[preprint]{jmlr2e}  

\usepackage[T1]{fontenc}
\usepackage[utf8]{inputenc}
\usepackage{enumitem}
\usepackage{mathtools}
\usepackage{amssymb}
\usepackage{booktabs}

\let\proof\relax

\usepackage{amsthm}

\usepackage{graphicx}
\usepackage{booktabs}
\usepackage{xcolor}
\usepackage{algorithmic}
\usepackage{xurl}  
\usepackage[capitalise,noabbrev]{cleveref}

\makeatletter
\@for\@tempa:=theorem,proposition,lemma,corollary,definition,example,remark\do{%
  \expandafter\let\csname\@tempa\endcsname\relax
  \expandafter\let\csname end\@tempa\endcsname\relax
  \expandafter\let\csname c@\@tempa\endcsname\@undefined
  \expandafter\let\csname the\@tempa\endcsname\@undefined}
\makeatother

\theoremstyle{plain}
\newtheorem{theorem}{Theorem}[section]
\newtheorem{proposition}[theorem]{Proposition}
\newtheorem{lemma}[theorem]{Lemma}
\newtheorem{corollary}[theorem]{Corollary}
\theoremstyle{definition}
\newtheorem{definition}[theorem]{Definition}
\newtheorem{example}[theorem]{Example}
\theoremstyle{remark}
\newtheorem{remark}[theorem]{Remark}
\makeatletter
\@for\@tempa:=proposition,lemma,corollary,definition,example,remark\do{%
  \@ifundefined{c@\@tempa}{\newcounter{\@tempa}}{}}
\makeatother
\crefname{algorithm}{Algorithm}{Algorithms}

\input{macros}

\ShortHeadings{Two-Sample Testing via Generative Processes}{English, Fukumizu and Suzuki}
\firstpageno{1}

\begin{document}

\title{Two-Sample Testing via Generative Processes}

\author{\name{Eshant English\textsuperscript{1,3}\thanks{Contact email: eshant.english@outlook.co.nz, eshantenglish@g.ecc.u-tokyo.ac.jp}}, \name{Kenji Fukumizu\textsuperscript{2}}, \name{Taiji Suzuki}\textsuperscript{1,4} \\
  \\
  \addr{\textsuperscript{1}The University of Tokyo, Tokyo, Japan} \\
  \addr{\textsuperscript{2}Institute of Statistical Mathematics, Japan} \\
  \addr{\textsuperscript{3}Hasso Plattner Institute for Digital Engineering, Germany} \\
  \addr{\textsuperscript{4}RIKEN Centre for Advanced Intelligence Project (AIP), Tokyo, Japan}
}

\maketitle

\begin{abstract}
Deciding whether two samples come from the same distribution is a classical problem in statistics, and generative transport offers a new way to approach it.
We build a stochastic interpolant directly between the two samples and observe that, for a symmetric schedule, its law is invariant under the time reflection $t \mapsto 1-t$ whenever the two distributions coincide.
We therefore test whether the marginals at times $t$ and $1-t$ agree, using their Jensen--Shannon divergence.
Both marginals are explicit mixtures over all cross pairs of observations, so nothing is learned, and permutation calibration gives exact finite-sample level.
For Gaussian noise, this divergence equals a time integral that pairs the reflection defects of the velocity field and of the score, so the test compares transport dynamics rather than endpoints alone.
With a narrow-plus-broad noise design, the test attains the minimax separation rate $n^{-2s/(4s+d)}$ over bounded, compactly supported densities whose difference has Sobolev smoothness $s > 3d/4$, with no lower bound on the densities.
Fusing a dyadic grid of noise scales through their permutation ranks, without sample splitting, preserves exact level and adapts to unknown $s$ at an iterated-logarithmic cost.
Empirically, the test matches or outperforms state-of-the-art kernel two-sample tests.
\end{abstract}

\vspace{1cm}
\begin{keywords}
  two-sample testing, stochastic interpolants, generative models, permutation tests,
  Jensen--Shannon divergence, minimax separation rates, adaptive testing
\end{keywords}
\clearpage
\newpage
\tableofcontents
\newpage

\input{sections/01_introduction}
\input{sections/02_preliminaries}
\input{sections/03_method}
\input{sections/04_theory}
\input{sections/05_experiments}

\input{sections/06_related_work}
\input{sections/07_discussion}

\acks{We extend our gratitude to Prof. Masato Taki for helpful discussions.  We gratefully acknowledge funding by  JST CREST (PMJCR2015), JSPS KAKENHI (24K02905, 25H01107), JST ERATO (JPMJER2601), the Deutsche Forschungsgemeinschaft (DFG, German Research Foundation) – project number
459422098, and the National Research Foundation, Singapore and the Ministry of Digital Development and Information under the AI Visiting Professorship Programme (award number AIVP-2024-004). Any opinions, findings and conclusions or recommendations expressed in this material are those of the author(s) and do not reflect the views of the National Research Foundation, Singapore and the Ministry of Digital Development and Information.}

\appendix
\input{appendix/A_experiments}

\clearpage
\vskip 0.2in
\bibliography{aistats2027_conference}

\end{document}

%% file: macros.tex
\newcommand{\R}{\mathbb{R}}
\newcommand{\E}{\mathbb{E}}
\renewcommand{\P}{\mathbb{P}}
\newcommand{\Id}{\mathbb{I}}
\newcommand{\dd}{\mathrm{d}}
\newcommand{\ind}{\mathbf{1}}
\newcommand{\JS}{\mathrm{JS}}
\newcommand{\KL}{\mathrm{KL}}
\newcommand{\TV}{\mathrm{TV}}
\newcommand{\LC}{\mathrm{LC}}
\DeclareMathOperator{\supp}{supp}
\DeclareMathOperator{\tr}{tr}
\DeclareMathOperator{\Var}{Var}
\DeclareMathOperator{\Cov}{Cov}
\DeclareMathOperator{\diag}{diag}
\DeclareMathOperator{\Bin}{Bin}
\DeclareMathOperator{\Beta}{Beta}

\newcommand{\norm}[1]{\lVert #1 \rVert}
\newcommand{\abs}[1]{\lvert #1 \rvert}

\newcommand{\Hnull}{\mathcal{H}_0}

\newcommand{\Dzero}{\mathcal{D}_0}
\newcommand{\Done}{\mathcal{D}_1}

\newcommand{\Cs}{\mathcal{C}_s(L,M,R)}
\newcommand{\rhohat}{\widehat{\rho}}
\newcommand{\muhat}{\widehat{\mu}}
\newcommand{\fhat}{\widehat{f}}
\newcommand{\phat}{\widehat{p}}
\newcommand{\pfuse}{\widehat{p}_{\mathrm{fuse}}}

\newcommand{\Arank}{\mathcal{A}}
\newcommand{\Tobs}{\widehat{T}}
\newcommand{\Tmc}{\widetilde{T}}
\newcommand{\Bil}{\mathcal{B}}
\newcommand{\Gram}{\mathcal{K}}
\newcommand{\Hil}{\mathbb{V}}


%% file: sections/01_introduction.tex
\section{Introduction}
\label{sec:intro}
In the natural sciences, we often learn about an object through the laws it obeys rather than by observing it directly: to decide whether two particles are of the same kind, a physicist sends both through the same field and checks whether they follow the same trajectory.
Statistics faces a similar problem. Given independent datasets
$\Dzero = \{x_0^{(i)}\}_{i=1}^{n_0}$ and $\Done = \{x_1^{(j)}\}_{j=1}^{n_1}$
drawn from densities $\rho_0$ and $\rho_1$, the two-sample problem asks whether
$\Hnull\colon \rho_0 = \rho_1$, while bounding the probability of a false rejection
\citep{gretton2012kernel, jmlr_mmdagg}. Over smoothness classes, the problem is well
understood: the optimal separation rate is known \citep{ingster1993asymptotically}, and
permutation and kernel tests attain it \citep{kim_minimax_2022, jmlr_mmdagg}.

Unlike the physicist, we have no laws of motion telling us how data from a given distribution ought to behave, but dynamical-transport generative models show that we can write such laws down
ourselves \citep{grathwohl_ffjord_2018, bunne_schrodinger_2023}. Diffusion models
\citep{song_score-based_2020}, flow matching \citep{lipman_flow_2022, liu_flow_2022} and stochastic
interpolants \citep{albergo_building_2022, albergo_stochastic_2025} define a sample-level
evolution $(x_t)_{t\in[0,1]}$ that carries $\rho_0$ to $\rho_1$. These laws are chosen by
design, and once chosen, the path becomes an object of inference that can carry
information its endpoints do not \citep{english_path_2026}. If the two distributions
coincide, an optimal transport between them is the identity, so any movement would be
evidence of a difference; but such a transport is ill-posed to estimate from finite data.
We therefore fix the transport law in advance and build a bridge directly between the two
samples.

The central observation is a symmetry. Picture a film of samples from $\rho_0$ moving
towards samples from $\rho_1$, alongside a second film in which samples from $\rho_1$ move
towards samples from $\rho_0$. If $\rho_0 = \rho_1$, the second film is the first played
backwards, so the frame at time $t$ in one film matches the frame at $1-t$ in the other.
For the stochastic interpolant $x_t = \alpha_t x_0 + \beta_t x_1 + \gamma_t z$ with
independent endpoints and mirrored coefficients $\alpha_{1-t} = \beta_t$ and
$\gamma_{1-t} = \gamma_t$, exchanging the endpoints is the same as reversing time, so under
the null the marginal density satisfies $\rho_t = \rho_{1-t}$ for every $t$
(\cref{prop:reflection}). \cref{fig:reflection} illustrates this: under the null, the path of marginals is a mirror
image of itself about $t = 1/2$, and every probability-flow path retraces itself, returning
to its starting point; under an alternative, both the frames and the paths separate from
their mirror images.
Conditional on the data, these marginals are explicit
mixtures over all $n_0 n_1$ cross pairs, so nothing needs to be learned: we compare a
frame near the start of the path with its mirror frame through their Jensen--Shannon
divergence, and calibrate the statistic by permutation.

\textbf{Contributions:}
\textbf{(i)} We propose a two-sample test, which compares the empirical marginals of a single stochastic interpolant at the mirrored times $t$ and $1-t$; both marginals are explicit mixtures over all cross pairs (\cref{sec:method}).
\textbf{(ii)} For Gaussian noise, we show that the reflected divergence is exactly a time integral pairing the reflection defects of the velocity and the score (\cref{thm:velocity-score}), so the statistic measures the transport dynamics rather than the endpoints alone.
\textbf{(iii)} With a designed narrow-plus-broad noise, the test attains the separation rate $n^{-2s/(4s+d)}$ for $s > 3d/4$, without any lower bound on the densities, and a matching lower bound shows that this rate is minimax (\cref{thm:minimax,thm:lower-bound}).
\textbf{(iv)} Fusing a dyadic grid of noise scales through a soft maximum of permutation ranks gives a test with exact finite-sample level (\cref{thm:fusion-level}) that attains $(\sqrt{\log \log n}/n)^{2s/(4s+d)}$ without knowledge of $s$ (\cref{lem:rank-moments,thm:adaptive}).

We discuss related work in \cref{app:related}.

\begin{figure}[!t]
  \centering
  \includegraphics[width=\textwidth]{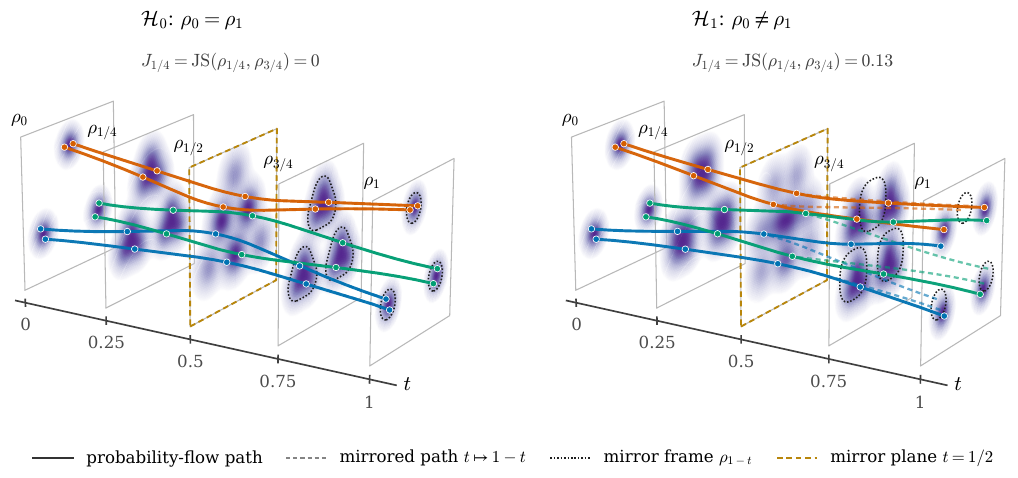}
  \caption{Time reflection of a stochastic interpolant built directly between $\rho_0$ and
    $\rho_1$ using our construction with the linear schedule $x_t = (1-t)\,x_0 + t\,x_1 + \sigma t(1-t)\,z$, $\sigma = 1.5$ and Gaussian $z$. 
    Slices show the exact marginals $\rho_t$. Solid curves
    are probability-flow paths from six starting points in $\rho_0$, coloured by their starting
    mode; dashed curves mirror the first half of each path under $t \mapsto 1-t$, and dotted
    contours on the slices $t = 3/4$ and $t = 1$ outline the mirror frames $\rho_{1/4}$ and
    $\rho_0$. \textbf{Top:} under $\Hnull$, $\rho_t = \rho_{1-t}$
    (\cref{prop:reflection}), each mirrored distribution coincides
    with the actual one. \textbf{Bottom:} under $\mathcal{H}_1$, the
    mirrored paths end in the outline of $\rho_0$ while the actual paths end in $\rho_1$. All densities and velocity fields are exact Gaussian-mixture
    expressions; nothing is learned.}
  \label{fig:reflection}
\end{figure}

%% file: sections/02_preliminaries.tex
\section{Preliminaries}
\label{sec:prelim}
\subsection{Stochastic Interpolants}
\label{sec:prelim-si}

We work within the stochastic interpolant framework of \citet{albergo_stochastic_2025}, which unifies flow-based and diffusion-based generative models under a single construction.
Let $\rho_0$ and $\rho_1$ be probability densities on $\R^d$, let $(x_0, x_1) \sim \rho_0 \otimes \rho_1$ be drawn from the independent coupling, and let $z$ be a latent variable, independent of $(x_0, x_1)$, with density $\varphi$.
Unlike in most generative applications, $z$ need not be Gaussian here, and the freedom to choose its law will matter in \cref{sec:minimax}.

The stochastic interpolant is the process
\begin{equation}
\label{eq:interpolant}
    x_t = \alpha_t\, x_0 + \beta_t\, x_1 + \gamma_t\, z, \qquad t \in [0,1],
\end{equation}
where the coefficients $\alpha_t, \beta_t, \gamma_t$ are continuously differentiable and satisfy the boundary conditions
\begin{equation}
\label{eq:boundary}
    \begin{gathered}
    \alpha_0 = \beta_1 = 1, \qquad \alpha_1 = \beta_0 = 0, \\
    \gamma_0 = \gamma_1 = 0, \qquad \gamma_t > 0 \quad \text{for } t \in (0,1).
    \end{gathered}
\end{equation}
These conditions guarantee that $x_{t=0} = x_0$ and $x_{t=1} = x_1$, so the time-dependent density $\rho_t$ of $x_t$ interpolates between $\rho_0$ and $\rho_1$.
The latent term $\gamma_t z$ smooths $\rho_t$ at intermediate times; setting $\gamma_t \equiv 0$ recovers the deterministic interpolants underlying flow matching \citep{lipman_flow_2022} and rectified flow \citep{liu_flow_2022}.
A standard choice is the linear schedule $\alpha_t = 1 - t$ and $\beta_t = t$, with $\gamma_t = \sigma \alpha_t \beta_t$ for an amplitude $\sigma > 0$.
We use it only as an example; all statements and derivations below are written for general coefficients, and each result states which properties of the coefficients it needs.

For a density $\mu$ on $\R^d$ and $c > 0$, we write $\mu^{[c]}(x) = c^{-d} \mu(x/c)$ for the density of $c\,y$ when $y \sim \mu$, and we use the same notation for signed measures, with $\mu^{[0]}$ the point mass at the origin when $\mu$ is a probability measure.
Since the three terms in \eqref{eq:interpolant} are independent, the marginal density is a convolution,
\begin{equation}
\label{eq:marginal}
    \rho_t = \rho_0^{[\alpha_t]} * \rho_1^{[\beta_t]} * \varphi^{[\gamma_t]}.
\end{equation}
\Cref{eq:interpolant} specifies the path at the level of individual samples: conditional on $(x_0, x_1, z)$, the trajectory $t \mapsto x_t$ is fully determined, and its law defines the path $(\rho_t)_{t \in [0,1]}$.
The conditional velocity along this path is available in closed form,
\begin{equation}
\label{eq:cond-velocity}
    \dot{x}_t = \dot{\alpha}_t\, x_0 + \dot{\beta}_t\, x_1 + \dot{\gamma}_t\, z,
\end{equation}
where dots denote time derivatives.
However, $\dot{x}_t$ depends on the endpoints and the latent, which are not accessible from $x_t$ alone.
Averaging over them gives the velocity field and the probability current,
\begin{equation}
\label{eq:fields}
    v_t(x) = \E[\dot{x}_t \mid x_t = x], \qquad j_t = \rho_t v_t,
\end{equation}
which, together with the score $\nabla \log \rho_t$, are linked by the continuity equation $\partial_t \rho_t + \nabla \cdot j_t = 0$ \citep{albergo_stochastic_2025}.

\subsection{Two-Sample Testing}
\label{sec:prelim-testing}

Consider independent datasets $\mathcal{D}_0 = \{x_0^{(i)}\}_{i=1}^{n_0}$ and
$\mathcal{D}_1 = \{x_1^{(j)}\}_{j=1}^{n_1}$ in $\mathbb{R}^d$, drawn i.i.d.\ from unknown
densities $\rho_0$ and $\rho_1$, and the hypotheses $\mathcal{H}_0\colon \rho_0 = \rho_1$
versus $\mathcal{H}_1\colon \rho_0 \neq \rho_1$. A test is a possibly randomised decision
rule $\Delta(\mathcal{D}_0, \mathcal{D}_1) \in \{0,1\}$, where $\Delta = 1$ rejects
$\mathcal{H}_0$; its internal randomness, such as the permutations and Monte Carlo draws of
Section~\ref{sec:method}, is independent of the data, and $\mathbb{P}_{\rho_0,\rho_1}$
denotes the joint law of the data and this randomness. The test has a level
$\kappa \in (0,1)$ if
\begin{equation}
  \sup_{\rho_0 \in \mathcal{P}(\mathbb{R}^d)}
  \mathbb{P}_{\rho_0,\rho_0}\bigl[\Delta(\mathcal{D}_0,\mathcal{D}_1) = 1\bigr] \le \kappa,
  \label{eq:level}
\end{equation}
and power $1-\delta$ against $(\rho_0,\rho_1)$ if
$\mathbb{P}_{\rho_0,\rho_1}[\Delta = 0] \le \delta$. Our level guarantees hold at every
finite sample size, whereas power is analysed as the sample sizes grow.

\subsection{Smoothness Classes and Separation Rates}
\label{sec:prelim-minimax}

To quantify power, we follow the minimax framework of \citet{ingster1993asymptotically}.
We use the Fourier convention $\widehat{g}(\xi) = \int e^{-\mathrm{i} \xi \cdot x} g(x) \, \dd x$ and the Sobolev norm $\norm{g}_{H^s}^2 = (2\pi)^{-d} \int (1 + \abs{\xi}^2)^s \abs{\widehat{g}(\xi)}^2 \, \dd \xi$.
For $L, M, R > 0$ and $s > 0$, let $B_L = \{x \colon \abs{x} \le L\}$, and let $\Cs$ be the class of pairs $(\rho_0, \rho_1)$ of probability densities on $\R^d$ such that
\begin{equation}
\label{eq:class}
\begin{gathered}
    \supp \rho_0 \cup \supp \rho_1 \subseteq B_L, \\
    0 \le \rho_0, \rho_1 \le M, \qquad \norm{\rho_0 - \rho_1}_{H^s} \le R .
\end{gathered}
\end{equation}
Following previous works \citet{jmlr_mmdagg}, only the difference is required to be smooth, and no individual or pooled density is assumed to be bounded below.
The class contains alternatives only when $M \abs{B_L} > 1$, where $\abs{B_L}$ is the volume of $B_L$.

For rates, we take equal sample sizes $n_0 = n_1 = n$.
A sequence of tests $\Delta_n$ of level $\kappa$ attains the separation rate $r_n$ over the class if there is a constant $C$ such that, for all sufficiently large $n$, the type II error is at most $\delta$ uniformly over all pairs in $\Cs$ with $\norm{\rho_0 - \rho_1}_2 \ge C r_n$.
The benchmark separation rate for this problem is $n^{-2s/(4s+d)}$: no test can do better (\cref{thm:lower-bound}; see also \citealp{ingster1993asymptotically}), and permutation and kernel tests are known to attain it \citep{kim_minimax_2022, jmlr_mmdagg}.

\textbf{Notation:}
We write $\abs{\cdot}$ for the Euclidean norm, $\norm{\cdot}_p$ for $L^p$ norms, and $G_h$ for the centred Gaussian density with covariance $h^2 \Id_d$, so that $G_h = G_1^{[h]}$.
For two sequences, $a_n \lesssim b_n$ means $a_n \le C b_n$ for a constant $C$, and $a_n \asymp b_n$ means that both $a_n \lesssim b_n$ and $b_n \lesssim a_n$.
Constants $C, c$ may change from line to line; they never depend on $n$ or on the pair $(\rho_0, \rho_1)$ within the class. Also, the unsubscripted probabilities are taken over all random variables involved.

%% file: sections/03_method.tex
\section{Testing via Generative Processes}
\label{sec:method}

\subsection{Reflection Symmetry}
\label{sec:reflection}
\begin{definition}[Symmetric schedule]
\label{def:symmetric}
The coefficients in \eqref{eq:interpolant} form a symmetric schedule if $\alpha_{1-t} = \beta_t$ and $\gamma_{1-t} = \gamma_t$ for all $t \in [0,1]$.
\end{definition}

Equivalently, $\beta_{1-t} = \alpha_t$.
The linear schedule with $\gamma_t = \sigma \alpha_t \beta_t$ is symmetric; the power guarantees of \cref{sec:theory} additionally need the channel conditions of \cref{def:channel}.
To compare the interpolant from $\rho_0$ to $\rho_1$ with the one that runs in the opposite direction, we write $\rho_t[\rho_0, \rho_1]$ and $j_t[\rho_0, \rho_1]$ for the marginal and the current at time $t$ of the interpolant whose endpoints are drawn from $\rho_0$ and $\rho_1$, and drop the brackets when the direction is clear.

\begin{proposition}[Reflection symmetry]
\label{prop:reflection}
Assume a symmetric schedule, and let $\rho_0$, $\rho_1$ and the law of $z$ be arbitrary.
\begin{enumerate}
    \item[(i)] For every $t \in [0,1]$, $\rho_{1-t}[\rho_0, \rho_1] = \rho_t[\rho_1, \rho_0]$.
    In particular, $\rho_0 = \rho_1$ implies $\rho_t = \rho_{1-t}$ for every $t$.
    \item[(ii)] If $\mathbb{E}|\dot x_t|<\infty$ and $\mathbb{E}|\dot x_{1-t}|<\infty$ (for instance, if
  $x_0$, $x_1$ and $z$ have finite first moments), then $j_{1-t}[\rho_0,\rho_1]=-j_t[\rho_1,\rho_0]$.
  Under $\mathcal H_0$, therefore, $j_{1-t}=-j_t$, and $v_{1-t}=-v_t$ wherever $\rho_t>0$.
    \item[(iii)] Conversely, if $\rho_{t_k} = \rho_{1-t_k}$ along a sequence $t_k \downarrow 0$, then $\rho_0 = \rho_1$.
\end{enumerate}
\end{proposition}

\begin{proof}
(i) By the symmetry of the schedule,
\begin{equation*}
    x_{1-t} = \alpha_{1-t} x_0 + \beta_{1-t} x_1 + \gamma_{1-t} z = \beta_t x_0 + \alpha_t x_1 + \gamma_t z .
\end{equation*}
Since $x_0 \sim \rho_0$, $x_1 \sim \rho_1$ and $z$ are mutually independent, the right-hand side has the law of the interpolant at time $t$ whose endpoints are drawn from $\rho_1$ and $\rho_0$, which proves $\rho_{1-t}[\rho_0, \rho_1] = \rho_t[\rho_1, \rho_0]$.
If $\rho_0 = \rho_1$, the two directions coincide, so $\rho_{1-t} = \rho_t$.
No symmetry of the law of $z$ is used.

(ii) Differentiating $\alpha_{1-t} = \beta_t$, $\beta_{1-t} = \alpha_t$ and $\gamma_{1-t} = \gamma_t$ gives $\dot{\alpha}_{1-t} = -\dot{\beta}_t$, $\dot{\beta}_{1-t} = -\dot{\alpha}_t$ and $\dot{\gamma}_{1-t} = -\dot{\gamma}_t$.
Hence
\begin{equation*}
    \dot{x}_{1-t} = -\big(\dot{\beta}_t x_0 + \dot{\alpha}_t x_1 + \dot{\gamma}_t z\big),
\end{equation*}
so the pair $(x_{1-t}, \dot{x}_{1-t})$ for the interpolant from $\rho_0$ to $\rho_1$ has the law of $(x'_t, -\dot{x}'_t)$ for the interpolant from $\rho_1$ to $\rho_0$.
When $\mathbb{E}|\dot x_t|<\infty$ and $\mathbb{E}|\dot x_{1-t}|<\infty$, the currents at both times
are the vector measures $j_u(x)\,\mathrm{d}x=\mathbb{E}[\dot x_u;\,x_u\in\mathrm{d}x]$,
$u\in\{t,1-t\}$, so $j_{1-t}[\rho_0,\rho_1]=-j_t[\rho_1,\rho_0]$. Both conditions are needed:
$\dot x_t$ weights $(x_0,x_1)$ by $(\dot\alpha_t,\dot\beta_t)$, whereas $\dot x_{1-t}$ weights them
by $(-\dot\beta_t,-\dot\alpha_t)$, so a heavy-tailed endpoint can make one integrable and not the
other. Under the null the two conditions coincide, because $x_0$ and $x_1$ are exchangeable.
Under the null the two directions coincide, which gives $j_{1-t} = -j_t$ and, where $\rho_t = \rho_{1-t} > 0$, $v_{1-t} = -v_t$.

(iii) By continuity of the coefficients and \eqref{eq:boundary}, $\alpha_t \to 1$, $\beta_t \to 0$ and $\gamma_t \to 0$ as $t \downarrow 0$, so $x_t \to x_0$ and $x_{1-t} \to x_1$ almost surely for every finite-valued random vector $z$.
Almost sure convergence implies weak convergence of the laws, so $\rho_{t_k} \to \rho_0$ and $\rho_{1-t_k} \to \rho_1$ weakly.
Equal measures along the sequence have equal weak limits, hence $\rho_0 = \rho_1$.
\end{proof}

Part~(i) uses neither Gaussian noise nor a symmetric law for $z$. Part~(iii) identifies the
null only as the compared times approach the endpoints, and at the midpoint the comparison
is empty, so we compare mirrored frames away from it; \cref{sec:signal} quantifies the signal
retained there.

\subsection{The Reflected Jensen--Shannon Divergence}
\label{sec:reflected-js}

All logarithms are natural.
For probability densities $g_0$ and $g_1$, the Jensen--Shannon divergence is
\begin{equation}
\label{eq:js}
\begin{aligned}
    \JS(g_0, g_1) = \frac{1}{2} \int g_0 \log \frac{2 g_0}{g_0 + g_1} \, \dd x& + \frac{1}{2} \int g_1 \log \frac{2 g_1}{g_0 + g_1} \, \dd x.
\end{aligned}
\end{equation}
With $m = (g_0 + g_1)/2$ and $F(z) = \{(1+z)\log(1+z) + (1-z)\log(1-z)\}/2$, it can be written as
\begin{equation}
\label{eq:js-F}
    \JS(g_0, g_1) = \int m \, F\Big(\frac{g_0 - g_1}{2m}\Big) \dd x,
\end{equation}
where the integrand is zero wherever $m = 0$, and
\begin{equation}
\label{eq:F-bounds}
    \frac{z^2}{2} \le F(z) \le \log 2, \qquad \abs{z} \le 1 .
\end{equation}
Thus, the divergence lies between a weighted quadratic form and $\log 2$, a fact we use repeatedly.

For a symmetric schedule and $t \in (0, 1/2)$, the reflected Jensen--Shannon divergence is
\begin{equation}
\label{eq:reflected-js}
    J_t(\rho_0, \rho_1) = \JS(\rho_t, \rho_{1-t}).
\end{equation}
By \cref{prop:reflection}, $J_t = 0$ for every $t$ under $\Hnull$.
Our first result shows that $J_t$ is a statement about the dynamics of the path and not only about two of its frames.

\begin{theorem}[Velocity--score identity]
\label{thm:velocity-score}
Assume that the endpoint laws $\rho_0$ and $\rho_1$, which need not have densities, have bounded support, and that $z \sim \mathcal{N}(0, \Id_d)$.
Then, for $t \in (0, 1/2)$,
\begin{equation}
\label{eq:vs-derivative}
    \frac{\dd J_t}{\dd t} = \frac{1}{2} \int \frac{\rho_t \rho_{1-t}}{\rho_t + \rho_{1-t}} \, (v_t + v_{1-t}) \cdot \nabla \log \frac{\rho_t}{\rho_{1-t}} \, \dd x,
\end{equation}
and consequently
\begin{equation}
\label{eq:vs-identity}
    J_t(\rho_0, \rho_1) = -\frac{1}{2} \int_t^{1/2} \!\! \int \frac{\rho_u \rho_{1-u}}{\rho_u + \rho_{1-u}} \, (v_u + v_{1-u}) \cdot \nabla \log \frac{\rho_u}{\rho_{1-u}} \, \dd x \, \dd u .
\end{equation}
\end{theorem}

For the proofs in this subsection, a sufficient standing assumption is that the endpoints are bounded, $z$ is standard Gaussian, and the coefficients are continuously differentiable.
We work on a compact time interval inside $(0,1)$, on which $\gamma_t$ is bounded away from zero.
The densities are then smooth and positive, velocities and scores have at most polynomial growth, and Gaussian tails justify differentiation under the integral and integration by parts.
The same proofs apply whenever the stated derivatives and integrals exist, and the boundary terms vanish.

\begin{lemma}[Continuity and reflection]
\label{lem:continuity}
Under these assumptions, $\partial_t \rho_t = -\nabla \cdot (\rho_t v_t)$.
If $q_t = \rho_{1-t}$ is viewed as a density indexed by increasing $t$, its velocity is $-v_{1-t}$, and
\begin{equation*}
    \partial_t q_t = +\nabla \cdot (q_t v_{1-t}).
\end{equation*}
\end{lemma}

\begin{proof}
For a compactly supported smooth test function $g$, the tower property gives
\begin{equation*}
    \frac{\dd}{\dd t} \E g(x_t) = \E\big[\nabla g(x_t) \cdot \dot{x}_t\big] = \int \nabla g \cdot v_t \, \rho_t \, \dd x .
\end{equation*}
Integration by parts gives the continuity equation in the sense of distributions, and hence classically, since the densities are smooth.
The reflected formula is the chain rule; its minus sign at the level of the velocity is essential.
\end{proof}

\begin{proof}[Proof of \cref{thm:velocity-score}]
Write $r = \rho_t$, $q = \rho_{1-t}$ and $m = (r + q)/2$.
Differentiating $\JS(r, q) = \frac{1}{2} \int r \log r + \frac{1}{2} \int q \log q - \int m \log m$ gives
\begin{equation*}
    \frac{\dd J_t}{\dd t} = \frac{1}{2} \int \partial_t r \, \log \frac{r}{m} \, \dd x + \frac{1}{2} \int \partial_t q \, \log \frac{q}{m} \, \dd x ,
\end{equation*}
because the terms linear in $\partial_t r$ and $\partial_t q$ cancel, both densities having unit mass.
Moreover,
\begin{equation*}
    \nabla \log \frac{r}{m} = \frac{q}{r + q} \, \nabla \log \frac{r}{q}, \qquad
    \nabla \log \frac{q}{m} = -\frac{r}{r + q} \, \nabla \log \frac{r}{q} .
\end{equation*}
Substituting $\partial_t r = -\nabla \cdot (r v_t)$ and $\partial_t q = +\nabla \cdot (q v_{1-t})$ from \cref{lem:continuity} and integrating by parts gives
\begin{equation*}
    \frac{\dd J_t}{\dd t} = \frac{1}{2} \int r v_t \cdot \frac{q}{r + q} \nabla \log \frac{r}{q} \, \dd x + \frac{1}{2} \int q v_{1-t} \cdot \frac{r}{r + q} \nabla \log \frac{r}{q} \, \dd x,
\end{equation*}
which is \eqref{eq:vs-derivative}.
Since $J_{1/2} = \JS(\rho_{1/2}, \rho_{1/2}) = 0$, integrating from $t$ to $1/2$ proves \eqref{eq:vs-identity}.
In terms of currents, the integrand of \eqref{eq:vs-derivative} is $\frac{1}{2} (\rho_{1-t} j_t + \rho_t j_{1-t}) / (\rho_t + \rho_{1-t}) \cdot \nabla \log (\rho_t / \rho_{1-t})$; the identity does not replace this weighted sum by $j_t + j_{1-t}$.
\end{proof}

Both factors in the integrand are reflection defects, which vanish at every time under the
null by \cref{prop:reflection}; under an alternative, $J_t$ accumulates their weighted
pairing over the half path from $t$ to the midpoint. The integrand can take either sign,
so it is a signed pairing rather than a squared norm, and the power guarantees of
Section~4 do not use the identity.
\subsection{The Empirical Statistic}
\label{sec:empirical}
Let $\rhohat_0 = n_0^{-1} \sum_i \delta_{x_0^{(i)}}$ and $\rhohat_1 = n_1^{-1} \sum_j \delta_{x_1^{(j)}}$ be the empirical measures of the two datasets.
Running the interpolant between them gives explicit marginals: by \eqref{eq:marginal} and the symmetry of the schedule,
\begin{align}
\label{eq:emp-marginals}
    \rhohat_t(x) &= \frac{1}{n_0 n_1} \sum_{i=1}^{n_0} \sum_{j=1}^{n_1} \varphi^{[\gamma_t]}\big(x - \alpha_t x_0^{(i)} - \beta_t x_1^{(j)}\big), \\
    \rhohat_{1-t}(x) &= \frac{1}{n_0 n_1} \sum_{i=1}^{n_0} \sum_{j=1}^{n_1} \varphi^{[\gamma_t]}\big(x - \beta_t x_0^{(i)} - \alpha_t x_1^{(j)}\big). \notag
\end{align}
The notation is consistent at the endpoints: by \eqref{eq:boundary}, the empirical marginals at $t = 0$ and $t = 1$ are the empirical measures themselves.
For $t \in (0, 1/2)$, the test statistic is
\begin{equation}
\label{eq:statistic}
    \Tobs_t = \JS\big(\rhohat_t, \rhohat_{1-t}\big).
\end{equation}
These are the exact marginals of the interpolant conditional on the empirical endpoint
laws. The density formula uses all $n_0 n_1$ cross pairs, but these pairs are not
independent data points, and the statistical sample sizes remain $n_0$ and $n_1$. For
Gaussian noise, \cref{thm:velocity-score} applies exactly to this empirical bridge, whose
velocity and score are explicit (\cref{rem:empirical-bridge}), so no learned function approximation enters.

\begin{remark}[The empirical bridge]
\label{rem:empirical-bridge}
Given a pair $(i, j)$, the latent is determined by the position, $z = (x - c_{ij,t})/\gamma_t$, so \eqref{eq:cond-velocity} gives the empirical current
\begin{equation*}
    \widehat{j}_t(x) = \frac{1}{n_0 n_1} \sum_{i,j} \Big[\dot{\alpha}_t x_0^{(i)} + \dot{\beta}_t x_1^{(j)} + \frac{\dot{\gamma}_t}{\gamma_t} (x - c_{ij,t})\Big] \varphi^{[\gamma_t]}(x - c_{ij,t}),
\end{equation*}
with velocity $\widehat{v}_t = \widehat{j}_t / \rhohat_t$ and score $\nabla \log \rhohat_t$.
Finite samples have bounded support, so for Gaussian noise \cref{thm:velocity-score} applies to the empirical bridge without change.
\end{remark}






The time must stay strictly positive: as $t \downarrow 0$, the two empirical frames
concentrate on the two samples, which are almost surely disjoint, so
$\widehat{T}_t \to \log 2$ even under $\Hnull$ (\cref{prop:saturation}). The noise scale
$\gamma_t$ therefore acts as a bandwidth (\cref{sec:theory}), and the label-dependent denominator
$(\rhohat_t + \rhohat_{1-t})/2$ in~\eqref{eq:js-F} is what the analysis of \cref{sec:theory} must
control (\cref{rem:denominator}).

\subsubsection{Saturation at Fixed Sample Size}

\begin{proposition}[Saturation at fixed sample size]
\label{prop:saturation}
If the finite sets $\{x_0^{(i)}\}$ and $\{x_1^{(j)}\}$ are disjoint, then $\lim_{t \downarrow 0} \Tobs_t = \log 2$.
This occurs almost surely for independent samples from continuous laws, including under $\Hnull$.
\end{proposition}

\begin{proof}
Choose disjoint open neighbourhoods $U_0$ of $\{x_0^{(i)}\}$ and $U_1$ of $\{x_1^{(j)}\}$.
As $t \downarrow 0$, every centre $\alpha_t x_0^{(i)} + \beta_t x_1^{(j)}$ converges to $x_0^{(i)}$ and $\gamma_t z \to 0$, so $\rhohat_t(U_0) \to 1$.
The mirrored centres $\beta_t x_0^{(i)} + \alpha_t x_1^{(j)}$ converge to $x_1^{(j)}$, so $\rhohat_{1-t}(U_0) \to 0$.
By the data-processing inequality for the Jensen--Shannon divergence, applied to the indicator of $U_0$, $\Tobs_t$ is at least the divergence between two Bernoulli laws whose success probabilities tend to one and zero, and this divergence tends to $\log 2$.
Since the divergence between arbitrary probability laws is at most $\log 2$, the claim follows.

\end{proof}

\subsubsection{The Pooled-Feature Identity}
\label{app:pooled-feature}

\begin{remark}[A linear contrast with a label-dependent denominator]
\label{rem:denominator}
For equal sample sizes, write $\muhat = (\rhohat_0 + \rhohat_1)/2$ and $\fhat = \rhohat_0 - \rhohat_1$.
Expanding the empirical mixtures gives the exact identities
\begin{align}
\label{eq:emp-contrast}
    \rhohat_t - \rhohat_{1-t} &= \big[\fhat^{[\alpha_t]} * \muhat^{[\beta_t]} - \muhat^{[\alpha_t]} * \fhat^{[\beta_t]}\big] * \varphi^{[\gamma_t]}, \\
\label{eq:emp-average}
    \frac{\rhohat_t + \rhohat_{1-t}}{2} &= \Big[\muhat^{[\alpha_t]} * \muhat^{[\beta_t]} - \frac{1}{4} \fhat^{[\alpha_t]} * \fhat^{[\beta_t]}\Big] * \varphi^{[\gamma_t]}.
\end{align}
The pooled measure $\muhat$ is unchanged by balanced relabelings, so the contrast \eqref{eq:emp-contrast} is linear in the labels, whereas the average \eqref{eq:emp-average}, which forms the denominator of the integrand in \eqref{eq:js-F}, contains a quadratic term that changes with the labels.
This is the main difference from comparing two separate kernel density estimates, whose average is label-invariant, and this is why the analysis in \cref{sec:theory} must control a nonlinear, label-dependent denominator.
One cannot replace \eqref{eq:emp-average} by $\muhat^{[\alpha_t]} * \muhat^{[\beta_t]} * \varphi^{[\gamma_t]}$ and still call the result the same statistic.
\end{remark}

The numerator of the statistic admits a useful linear representation, even for unequal group sizes.
Let $N = n_0 + n_1$ and let $\muhat_N = (n_0 \rhohat_0 + n_1 \rhohat_1)/N$ be the pooled empirical measure.
For a time $t$ with noise density $\varphi$, define the antisymmetric kernel and the pooled feature map
\begin{align*}
    \Psi_t(u, v; x) &= \varphi^{[\gamma_t]}(x - \alpha_t u - \beta_t v) - \varphi^{[\gamma_t]}(x - \beta_t u - \alpha_t v), \\
    \psi_t(v; x) &= \int \Psi_t(u, v; x) \, \muhat_N(\dd u).
\end{align*}

\begin{lemma}[Pooled-feature identity]
\label{lem:pooled-feature}
Writing $\rhohat_1 \psi_t = \int \psi_t(v; \cdot) \, \rhohat_1(\dd v)$ and similarly for $\rhohat_0$,
\begin{equation}
\label{eq:pooled-feature}
    \rhohat_t - \rhohat_{1-t} = \rhohat_1 \psi_t - \rhohat_0 \psi_t, \qquad \muhat_N \psi_t = 0 .
\end{equation}
\end{lemma}

\begin{proof}
Since $\Psi_t(u, v; x) = -\Psi_t(v, u; x)$, the second identity holds, and in $\rhohat_1 \psi_t - \rhohat_0 \psi_t$ the two same-group terms vanish.
The two cross-group terms have coefficients $n_0/N$ and $n_1/N$, which add to one, and their sum is $\iint \Psi_t(u, v; \cdot) \, \rhohat_0(\dd u) \, \rhohat_1(\dd v) = \rhohat_t - \rhohat_{1-t}$.
\end{proof}

Because $\muhat_N$ is invariant under relabelling, \eqref{eq:pooled-feature} shows that the numerator is linear in the labels for a feature map that is fixed once the pooled sample is known.

\begin{lemma}[Conditional second moment of the contrast]
\label{lem:balanced-moment}
Let $\psi_i = \psi_t(V_i; \cdot)$ be elements of a real Hilbert space $\Hil$, where $V_1, \dots, V_N$ are the pooled observations and $\sum_i \psi_i = 0$.
Under a uniform reassignment into groups of sizes $n_0$ and $n_1$, with relabelled empirical measures $\rhohat_0^\pi$ and $\rhohat_1^\pi$,
\begin{equation*}
    \E_\pi \norm{\rhohat_1^\pi \psi - \rhohat_0^\pi \psi}_{\Hil}^2 = \frac{N^2}{n_0 n_1 (N-1)} \cdot \frac{1}{N} \sum_{i=1}^N \norm{\psi_i}_{\Hil}^2 .
\end{equation*}
More generally, for equal group sizes $n$ and arbitrary $\psi_i$, $\E_\pi \norm{n^{-1} \sum_i \varepsilon_i \psi_i}_{\Hil}^2 = \frac{4}{N(N-1)} \sum_i \norm{\psi_i - \bar{\psi}}_{\Hil}^2$, where $\varepsilon_i = \pm 1$ are uniform balanced signs and $\bar{\psi} = N^{-1} \sum_i \psi_i$.
\end{lemma}

\begin{proof}
If $A$ is the uniformly selected first group, its sum is $S = \sum_{i \in A} \psi_i$, and the contrast equals $-N S / (n_0 n_1)$.
The inclusion probabilities $n_0/N$ and $n_0(n_0 - 1)/(N(N-1))$, together with $\sum_{i \ne j} \langle \psi_i, \psi_j \rangle = -\sum_i \norm{\psi_i}^2$, give $\E_\pi \norm{S}^2 = n_0 n_1 \sum_i \norm{\psi_i}^2 / (N(N-1))$.
Multiplication by $N^2/(n_0^2 n_1^2)$ gives the first formula.
For the second, $\E \varepsilon_i \varepsilon_j = -1/(N-1)$ for $i \ne j$ gives $\E_\pi \norm{\sum_i \varepsilon_i \psi_i}^2 = \frac{N}{N-1} \sum_i \norm{\psi_i - \bar{\psi}}^2$, and $n^{-2} = 4/N^2$.
\end{proof}

This lemma concerns the linear density contrast; it is not a variance bound for the nonlinear statistic with denominator \eqref{eq:emp-average}.

\subsection{Calibration by Permutation}
\label{sec:permutation}
Conditional on the unordered pooled sample $\{V_1, \dots, V_N\}$, $N = n_0 + n_1$, the observed labelling is uniformly distributed under $\Hnull$.
We draw $B$ independent uniform relabellings with the group sizes fixed, index the observed labelling by $b = 0$ and the random ones by $b = 1, \dots, B$, and compute the statistic $T^b$ on each row at the same time, noise law and numerical objects.
The permutation p-value is
\begin{equation}
\label{eq:pvalue}
    \phat = \frac{1 + \sum_{b=1}^{B} \ind\{T^b \ge T^0\}}{B + 1},
\end{equation}
and the test rejects $\Hnull$ when $\phat \le \kappa$.


Calibration requires that all pooled preprocessing and all shared numerical randomness be independent of the observed labelling, conditional on the pooled data.
It is consistent with the general theory of random permutation tests
\citep{hemerik_exact_2018}, and its level follows from \cref{thm:fusion-level}.

Times, noise laws and coefficients may be chosen in advance or from the unlabelled pooled sample; choosing them to favour the observed labels, without repeating and calibrating that selection, is not covered.

\subsubsection{Permutation Validity and Power from Pairwise Orderings}


This conditional Monte Carlo argument is consistent with the general theory of random permutation tests \citep{hemerik_exact_2018}.
Besides rank fusion, other row-equivariant aggregations are valid by the same argument; such aggregations have level, but their power requires a separate argument.

\begin{lemma}[Power from pairwise ordering]
\label{lem:pairwise-order}
Let $T^0, T^1, \dots, T^B$ be the final scores used in \eqref{eq:pvalue}, and assume that rows $1, \dots, B$ have the same marginal relation to row $0$.
If $k_B = \lfloor \kappa (B+1) \rfloor \ge 1$, then
\begin{equation*}
    \P(\phat > \kappa) \le \frac{B \, \P(T^1 \ge T^0)}{k_B} .
\end{equation*}
\end{lemma}

\begin{proof}
Let $C_B = \sum_{b=1}^B \ind\{T^b \ge T^0\}$.
Nonrejection means $1 + C_B > \kappa(B+1)$, which implies $C_B \ge k_B$, so Markov's inequality gives $\P(C_B \ge k_B) \le \E C_B / k_B = B \, \P(T^1 \ge T^0) / k_B$.
\end{proof}

\subsection{Fusing Channels by Permutation Ranks}
\label{sec:fusion}

A single channel, that is, a time together with a noise law, sees the data at a single resolution, and \cref{sec:theory} shows that the best resolution depends on the unknown smoothness.
We therefore compute the statistic on several channels $c = 1, \dots, C$ and fuse them.
Let $T_{bc}$ be the statistic of channel $c$ on row $b$, where the same $B$ relabellings are used for every channel.
Within each channel, we replace the statistics by their inclusive upper ranks among all rows,
\begin{equation}
\label{eq:ranks}
    r_{bc} = \frac{1}{B+1} \sum_{j=0}^{B} \ind\{T_{jc} \ge T_{bc}\} \in \Big\{\frac{1}{B+1}, \dots, 1\Big\},
\end{equation}
so that tied rows receive the largest rank in their tied block, and both the observed row and all permuted rows enter every rank.
The rank $r_{0c}$ of the observed row is exactly the permutation p-value \eqref{eq:pvalue} of channel $c$.
For weights $w_c > 0$ with $\sum_c w_c = 1$ and a temperature $\tau > 0$, both fixed in advance, rank fusion computes
\begin{equation}
\label{eq:rank-fusion}
\begin{gathered}
    \Lambda_b = \sum_{c=1}^{C} w_c \, r_{bc}^{-\tau}, \\
    \pfuse = \frac{1}{B+1} \sum_{b=0}^{B} \ind\{\Lambda_b \ge \Lambda_0\},
\end{gathered}
\end{equation}
and rejects $\Hnull$ when $\pfuse \le \kappa$.
Since the logarithm is increasing, $\tau^{-1}\log\Lambda_b$, a weighted soft maximum of
$-\log r_{bc}$, orders the rows in the same way, so rank fusion applies the soft-maximum
aggregation of MMD-FUSE \citep{biggs_mmd-fuse_2023} to permutation ranks, which place all
channels on a common scale; it uses neither MMD statistics nor the power theorem of
MMD-FUSE.

\begin{theorem}[Level of rank fusion]
\label{thm:fusion-level}
Under $\Hnull$, suppose that the channel construction and the shared numerical objects are independent of the observed labels, conditional on the unordered pooled observations, that the same numerical statistic is applied to every row, and that $\tau$ and the weights are chosen in advance.
Then, for every $\kappa \in [0, 1]$,
\begin{equation*}
    \P(\pfuse \le \kappa) \le \frac{\lfloor \kappa (B+1) \rfloor}{B+1} \le \kappa .
\end{equation*}
No independence between channels is required, nor is any assumption of numerical accuracy.
\end{theorem}

\begin{proof}[Proof of \cref{thm:fusion-level}]
Under the null, and conditional on the unordered pooled observations and on label-independent auxiliary randomness, the observed labelling and the $B$ independent uniform balanced labellings are independent and identically distributed.
Their vector-valued statistic rows $(T_{b1}, \dots, T_{bC})$ are therefore exchangeable.
The rank transform and the map from a row of ranks to $\Lambda_b$ commute with permutations of the rows, so $(\Lambda_0, \dots, \Lambda_B)$ is exchangeable.
For every deterministic score vector $g$, at most $\lfloor \kappa (B+1) \rfloor$ rows $b$ satisfy $\#\{j \colon g_j \ge g_b\} \le \kappa (B+1)$.
Hence
\begin{equation*}
    \P(\pfuse \le \kappa) = \frac{1}{B+1} \E \sum_{b=0}^{B} \ind\big\{\#\{j \colon \Lambda_j \ge \Lambda_b\} \le \kappa (B+1)\big\} \le \frac{\lfloor \kappa (B+1) \rfloor}{B+1} .
\end{equation*}
Ties can only decrease the count.
Shared Monte Carlo errors need not be independent, because a common label-independent integration procedure preserves the required exchangeability.
\end{proof}

The theorem also permits more general tuning rules that are equivariant under
permutations of all rows of the statistic. Choosing a temperature to favour the observed
row, and leaving that selection uncalibrated, is not covered.

All power results below treat a fixed temperature. With a single channel,
$\widehat p_{\mathrm{fuse}}$ equals that channel's permutation $p$-value~\eqref{eq:pvalue}.

\subsection{Computation}
\label{sec:computation}
The divergence between two mixtures of $n_0 n_1$ kernels has no closed form, so we
estimate it by importance sampling. For each channel, we draw $U_1, \dots, U_J$ i.i.d.\
from a pooled proposal $q_N$, a mixture over all ordered pairs of pooled observations that
does not depend on the labels and is shared by all rows, and average the integrand
of~\eqref{eq:js-F} divided by $q_N$ (\cref{app:computation}). The summands are
bounded, so Hoeffding's inequality bounds the Monte Carlo error uniformly over rows and
channels (\cref{prop:mc}); since the integration points do not depend on the
labels, \cref{thm:fusion-level} applies at every finite $J$. With pair kernels
evaluated in log space, the dominant cost is $O(CJN^2(d+B))$ for $C$ channels.
\cref{alg:test} summarises the procedure.

\begin{algorithm}[t]
\caption{Generative Process-based Two-Sample Test}
\label{alg:test}
\begin{algorithmic}[1]
\REQUIRE Datasets $\Dzero, \Done$; channels $c = 1, \dots, C$, each with a time $t_c \in (0, 1/2)$ and a noise law; weights $w_c > 0$ with $\sum_c w_c = 1$; temperature $\tau > 0$; level $\kappa$; numbers $B$ and $J$.
\STATE Pool the data into $V_1, \dots, V_N$; draw $B$ uniform relabellings with fixed group sizes, and let row $0$ be the observed labelling.
\FOR{$c = 1$ \TO $C$}
    \STATE Draw $U_1, \dots, U_J$ i.i.d.\ from the pooled proposal \eqref{eq:proposal} of channel $c$.
    \STATE Compute $T_{bc} = \Tmc^b$ by \eqref{eq:mc} for $b = 0, \dots, B$.
    \STATE Compute the inclusive ranks $r_{bc}$ by \eqref{eq:ranks}.
\ENDFOR
\STATE $\Lambda_b \leftarrow \sum_c w_c \, r_{bc}^{-\tau}$ for $b = 0, \dots, B$.
\STATE $\pfuse \leftarrow (B+1)^{-1} \sum_{b=0}^{B} \ind\{\Lambda_b \ge \Lambda_0\}$.
\STATE Reject $\Hnull$ if $\pfuse \le \kappa$.
\end{algorithmic}
\end{algorithm}

\subsubsection{Shared Monte Carlo Integration}
\label{app:mc}
\label{app:computation}
The divergence between two mixtures of $n_0 n_1$ kernels has no closed-form expression, so we integrate by importance sampling, using a proposal that does not depend on the labels and is shared across all rows.
For a channel with time $t$ and noise density $\varphi$, the pooled proposal is
\begin{equation}
\label{eq:proposal}
    q_N(x) = \frac{1}{N^2} \sum_{k, l = 1}^{N} \varphi^{[\gamma_t]}(x - \alpha_t V_k - \beta_t V_l).
\end{equation}
It is sampled by drawing two pooled observations independently with replacement and adding the channel's noise.
It contains all ordered pairs, including self-pairs, and expanding the pooled measures shows that, for every labelling,
\begin{equation}
\label{eq:domination}
    q_N \ge w_N \big(\rhohat_t + \rhohat_{1-t}\big), \qquad w_N = \frac{n_0 n_1}{N^2}.
\end{equation}
For points $U_1, \dots, U_J$ drawn i.i.d.\ from $q_N$ and shared by all rows, the estimate on row $b$ is
\begin{equation}
\label{eq:mc}
    \Tmc^b = \frac{1}{J} \sum_{\ell=1}^{J} \frac{\widehat{m}^b(U_\ell)}{q_N(U_\ell)} \, F\bigg(\frac{\rhohat^b_t(U_\ell) - \rhohat^b_{1-t}(U_\ell)}{2 \widehat{m}^b(U_\ell)}\bigg),
\end{equation}
where $\widehat{m}^b = (\rhohat^b_t + \rhohat^b_{1-t})/2$ and the summand is zero when $\widehat{m}^b(U_\ell) = 0$.
Conditional on the data and the labelling, the expectation is the exact statistic.
By \eqref{eq:domination}, each summand lies in $[0, L_w]$ with $L_w = \log 2 / (2 w_N)$, and Hoeffding's inequality gives a uniform error bound over rows and channels (\cref{prop:mc}).
For equal sample sizes, $w_N = 1/4$ and $L_w = 2 \log 2$.
Since the integration points do not depend on the labels, \cref{thm:fusion-level} applies at every finite $J$.
A finite Monte Carlo average may exceed $\log 2$ even though the exact divergence cannot; we do not clip it, since clipping would change the estimator.

In practice, all pair kernels are evaluated in log space and normalised by their largest value, so that the density ratios entering \eqref{eq:mc} are ratios of sums of the same stabilised weights and common normalisers cancel (\cref{app:mc}).
The dominant cost is $O(C J N^2 (d + B))$ for $C$ channels.
\Cref{alg:test} summarises the procedure, including the rank fusion of \cref{sec:fusion} and the bandwidth grid of \cref{sec:adaptive}.

The proposal \eqref{eq:proposal} is $q_N = (\muhat_N^{[\alpha_t]} * \muhat_N^{[\beta_t]}) * \varphi^{[\gamma_t]}$, and expanding the two pooled measures shows that it contains the term $w_N (\rhohat_t + \rhohat_{1-t})$ for every labelling, which is \eqref{eq:domination}.

\begin{proposition}[Uniform numerical error bound]
\label{prop:mc}
Each summand in \eqref{eq:mc} lies in $[0, L_w]$ with $L_w = \log 2 / (2 w_N)$.
For $C$ channels with $J$ integration points per channel, conditional on the pooled observations and on all labellings,
\begin{equation}
\label{eq:mc-bound}
    \P\Big(\max_{b, c} \big|\Tmc_c^b - T_c^b\big| > \eta\Big) \le 2 C (B+1) \exp\Big(-\frac{2 J \eta^2}{L_w^2}\Big).
\end{equation}
\end{proposition}

\begin{proof}
By \eqref{eq:domination}, $\widehat{m}^b / q_N \le 1/(2 w_N)$.
Combining this with $0 \le F \le \log 2$, we apply Hoeffding's inequality to each row and channel and take a union bound.
Independence between different rows or channels is not needed for the union bound.
\end{proof}

The bound \eqref{eq:mc-bound} concerns the numerical error of the scores; the per-channel conditional Monte Carlo standard errors that an implementation may report are not uncertainties for the final permutation p-value.

\paragraph{Stable evaluation:}
For a query point $x$, compute the pooled ordered-pair log kernels $\ell_{kl} = \log \varphi^{[\gamma_t]}(x - \alpha_t V_k - \beta_t V_l)$ and subtract their largest value.
With $W_{kl} = e^{\ell_{kl} - \max \ell}$, $S = \sum_{k,l} W_{kl}$, and $A_b$ the first group on row $b$,
\begin{equation*}
    \frac{\rhohat^b_t(x)}{q_N(x)} = \frac{N^2}{n_0 n_1} \cdot \frac{\sum_{k \in A_b, l \notin A_b} W_{kl}}{S},
\end{equation*}
and the mirrored ratio uses the opposite orientation.
Common density normalisers cancel, which avoids ratios of two underflowed densities, and compactly supported noise components, whose log density is $-\infty$ off their support, are handled by the convention that empty sums are zero.
Near zero, $F$ should be evaluated through its positive even power series, $F(z) = \sum_{k \ge 1} z^{2k} / (2k(2k-1))$.

%% file: sections/04_theory.tex
\section{Power Guarantees}
\label{sec:theory}
Throughout this section, $n_0 = n_1 = n$, and constants may depend on $d, L, M, R, s$ and
the noise design, but never on $n$ or on the pair $(\rho_0,\rho_1)$.

\input{sections/04a_notation}  

\subsection{What the Reflected Marginals Retain}
\label{sec:signal}
Fix a symmetric schedule and a time $t \in (0, 1/2)$, and write $f = \rho_0 - \rho_1$ and $\mu = (\rho_0 + \rho_1)/2$.
Expanding the population mixtures, exactly as for the empirical laws in \cref{rem:denominator}, gives the decomposition
\begin{equation}
\label{eq:contrast}
    \rho_t - \rho_{1-t} = \big[f^{[\alpha_t]} * \mu^{[\beta_t]} - \mu^{[\alpha_t]} * f^{[\beta_t]}\big] * \varphi^{[\gamma_t]}.
\end{equation}
When $\alpha_t$ is close to one and the noise dominates the small coefficient $\beta_t$,
the first term keeps the endpoint difference at resolution $\gamma_t$, while the second,
reversed term, in which the difference enters only through $\beta_t$, is negligible. The
following definition records the geometry we need.

\begin{definition}[Channel]
\label{def:channel}
For a symmetric schedule, $h \in (0, 1]$ and $\sigma > 0$, a time $t \in (0, 1/2)$ is an $(h, \sigma)$-channel if
\begin{equation*}
    \gamma_t = h, \quad \alpha_t \ge \tfrac{1}{2},  \quad \beta_t>0, \quad \alpha_t + \beta_t \le 1, \quad \gamma_t \ge \tfrac{\sigma}{2} \beta_t .
\end{equation*}
\end{definition}

In particular, $\tfrac12\le\alpha_t<1$ and $0<\beta_t\le1-\alpha_t\le\tfrac12$.
The ratio $\gamma_t / \beta_t$ compares the noise with the small endpoint coefficient.
A large ratio suppresses the reversed term, while a small $\gamma_t$ keeps the resolution fine.

\begin{example}[Linear schedule]
\label{ex:linear}
For $\alpha_t = 1 - t$, $\beta_t = t$ and $\gamma_t = \sigma \alpha_t \beta_t$, we have $\alpha_t + \beta_t = 1$ and $\gamma_t / \beta_t = \sigma \alpha_t \ge \sigma/2$ on $(0, 1/2]$.
Hence, for every $h \in (0, 1]$ with $h < \sigma/4$, the time
\begin{equation}
\label{eq:th}
    t_h = \frac{2h/\sigma}{1 + \sqrt{1 - 4h/\sigma}}
\end{equation}
solves $\gamma_{t_h} = h$ and is an $(h, \sigma)$-channel.
Thus, a sequence of bandwidths is implemented by a sequence of times on the same interpolant, with $\sigma$ fixed: the bandwidth is changed by moving along the path, not by changing $\sigma$.
\end{example}

\begin{theorem}[Signal retained by the reflected marginals]
\label{thm:signal}
Let $z \sim \mathcal{N}(0, \Id_d)$, $s > 0$ and $(\rho_0, \rho_1) \in \Cs$, and write $\omega = \norm{\rho_0 - \rho_1}_2$.
Let $t\in(0,1/2)$ satisfy $\alpha_t>0$, $\beta_t>0$ and $\gamma_t\ge\frac12\sigma_G\beta_t$, where
\begin{equation}
\label{eq:sigma-gauss}
    \sigma_G = \max\bigg\{4L, \ \frac{4 \sqrt{M} K_L \norm{\nabla G_1}_1}{c_0}\bigg\},
\end{equation}
$K_L = (\int_{B_L} \abs{x}^2 \dd x)^{1/2}$ and $c_0 = e^{-1/2}/2$.
Then, with $h = \gamma_t$,
\begin{align}
\label{eq:signal-l2}
    \norm{\rho_t - \rho_{1-t}}_2 &\ge \alpha_t^{-d/2} c_0 \Big[\frac{\omega}{2} - R \Big(\frac{h}{\alpha_t}\Big)^{s}\Big]_+, \\
\label{eq:signal-js}
    J_t(\rho_0, \rho_1) &\ge \frac{e^{-1}}{32 M} \Big[\frac{\omega}{2} - R \Big(\frac{h}{\alpha_t}\Big)^{s}\Big]_+^2 .
\end{align}
If, in addition, $\alpha_t \ge 1/2$, then $J_t(\rho_0, \rho_1) \ge e^{-1} \omega^2 / (256 M) - e^{-1} 2^{2s} R^2 h^{2s} / (32 M)$.
\end{theorem}

We prove a version of \cref{thm:signal} for a general noise kernel, which is reused in \cref{app:minimax}.
Its constants depend on the dimension, the support, the density bound and the smoothness.
The proof cuts the frequencies at $\abs{\xi} \le 1/h$, where the Fourier transform of the mixture at the small coefficient, $\widehat{\mu}(\beta_t \xi)$, stays close to one, and it bounds the reversed term through the ratio $\gamma_t / \beta_t$.
Recall $f = \rho_0 - \rho_1$, $\mu = (\rho_0 + \rho_1)/2$ and $\omega = \norm{f}_2$.

\begin{lemma}[Signal through a channel]
\label{lem:signal-general}
Let $\varphi$ be a probability density on $\R^d$ with $\norm{\nabla \varphi}_1 < \infty$, and let $r_\varphi, c_\varphi > 0$ be such that $\abs{\widehat{\varphi}(\xi)} \ge 2 c_\varphi$ for $\abs{\xi} \le r_\varphi$.
Let $s > 0$ and $(\rho_0, \rho_1) \in \Cs$, and let $\alpha, \beta, h > 0$ satisfy
\begin{equation}
\label{eq:sigma-phi}
    \frac{h}{\beta} \ge \frac{\sigma_\varphi}{2}, \qquad \sigma_\varphi = \max\bigg\{4 L r_\varphi, \ \frac{4 \sqrt{M} K_L \norm{\nabla \varphi}_1}{c_\varphi}\bigg\}.
\end{equation}
Then
\begin{equation}
\label{eq:signal-general}
    \Big\| \big[f^{[\alpha]} * \mu^{[\beta]} - \mu^{[\alpha]} * f^{[\beta]}\big] * \varphi^{[h]} \Big\|_2 \ge \alpha^{-d/2} c_\varphi \Big[\frac{\omega}{2} - R\Big(\frac{h}{\alpha r_\varphi}\Big)^s\Big]_+ .
\end{equation}
\end{lemma}

\begin{proof}
Write the two terms as $u = f^{[\alpha]} * \mu^{[\beta]} * \varphi^{[h]}$ and $v = \mu^{[\alpha]} * f^{[\beta]} * \varphi^{[h]}$.

\emph{Step 1: low frequencies retain the endpoint difference.}
Since $\mu$ is a probability density supported in $B_L$, $\abs{\widehat{\mu}(\eta) - 1} \le \int \abs{e^{-\mathrm{i} \eta \cdot x} - 1} \mu(x) \, \dd x \le L \abs{\eta}$.
On $\abs{\xi} \le r_\varphi / h$, the first requirement in \eqref{eq:sigma-phi} gives $\abs{\widehat{\mu}(\beta \xi)} \ge 1 - L \beta r_\varphi / h \ge 1/2$, and by assumption $\abs{\widehat{\varphi}(h \xi)} \ge 2 c_\varphi$.
Since $\widehat{u}(\xi) = \widehat{f}(\alpha \xi) \widehat{\mu}(\beta \xi) \widehat{\varphi}(h \xi)$, Plancherel's identity and the change of variable $\eta = \alpha \xi$ give
\begin{align*}
    \norm{u}_2^2 &\ge c_\varphi^2 (2\pi)^{-d} \int_{\abs{\xi} \le r_\varphi/h} \abs{\widehat{f}(\alpha \xi)}^2 \, \dd \xi \\
    &= \alpha^{-d} c_\varphi^2 (2\pi)^{-d} \int_{\abs{\eta} \le \alpha r_\varphi / h} \abs{\widehat{f}(\eta)}^2 \, \dd \eta \\
    &\ge \alpha^{-d} c_\varphi^2 \big[\omega^2 - R^2 (h/(\alpha r_\varphi))^{2s}\big].
\end{align*}
Indeed, the omitted tail is at most $(h/(\alpha r_\varphi))^{2s} \norm{f}_{H^s}^2 \le R^2 (h/(\alpha r_\varphi))^{2s}$, because $(1 + \abs{\eta}^2)^s \ge (\alpha r_\varphi / h)^{2s}$ there.
Taking square roots and using $\sqrt{(x^2 - y^2)_+} \ge (x - y)_+$ gives
\begin{equation}
\label{eq:signal-step1}
    \norm{u}_2 \ge \alpha^{-d/2} c_\varphi \big[\omega - R (h/(\alpha r_\varphi))^s\big]_+ .
\end{equation}
This frequency-cutoff argument works for every $s > 0$, and it does not rely on an expansion of $\widehat{\varphi}$ to all orders.

\emph{Step 2: control the reversed small-coefficient term.}
Let $A = \mu^{[\alpha]} * \varphi^{[h]}$.
Since $\int f = 0$,
\begin{equation*}
    v(x) = \int \big[A(x - \beta y) - A(x)\big] f(y) \, \dd y .
\end{equation*}
The $L^2$ translation inequality, Minkowski's inequality and the Cauchy--Schwarz inequality on $B_L$ yield
\begin{equation*}
    \norm{v}_2 \le \beta \norm{\nabla A}_2 \int \abs{y} \abs{f(y)} \, \dd y \le \beta \norm{\nabla A}_2 K_L \, \omega .
\end{equation*}
Young's inequality for the vector-valued convolution gives
\begin{equation*}
    \norm{\nabla A}_2 \le \norm{\mu^{[\alpha]}}_2 \norm{\nabla \varphi^{[h]}}_1 \le \alpha^{-d/2} \sqrt{M} \, \frac{\norm{\nabla \varphi}_1}{h},
\end{equation*}
because $\norm{\mu^{[\alpha]}}_2^2 \le \norm{\mu^{[\alpha]}}_\infty \le \alpha^{-d} M$ and $\norm{\nabla \varphi^{[h]}}_1 = h^{-1} \norm{\nabla \varphi}_1$.
Therefore, by the second requirement in \eqref{eq:sigma-phi},
\begin{equation}
\label{eq:signal-step2}
    \norm{v}_2 \le \alpha^{-d/2} \frac{\beta}{h} \sqrt{M} K_L \norm{\nabla \varphi}_1 \, \omega \le \alpha^{-d/2} \frac{c_\varphi}{2} \, \omega .
\end{equation}
Subtracting \eqref{eq:signal-step2} from \eqref{eq:signal-step1}, and retaining the trivial lower bound zero when necessary, proves \eqref{eq:signal-general}.
\end{proof}

\begin{proof}[Proof of \cref{thm:signal}]
For $\varphi = G_1$, we have $\widehat{G_1}(\xi) = e^{-\abs{\xi}^2/2} \ge e^{-1/2} = 2 c_0$ for $\abs{\xi} \le 1$, so $r_\varphi = 1$, $c_\varphi = c_0$ and $\sigma_\varphi = \sigma_G$.
By assumption, $\alpha_t,\beta_t>0$ and $h/\beta=\gamma_t/\beta_t\ge\sigma_G/2$, so \eqref{eq:contrast} and \cref{lem:signal-general} with $(\alpha, \beta, h) = (\alpha_t, \beta_t, \gamma_t)$ give \eqref{eq:signal-l2}.

\emph{Step 3: pass from the density contrast to the divergence.}
Both $\rho_t$ and $\rho_{1-t}$ are bounded above by $M_\alpha = \alpha_t^{-d} M$: condition on the endpoint with coefficient $\beta_t$ and on the noise, and use the density of the endpoint with coefficient $\alpha_t$.
By \cref{app:notation}, $\JS(\rho_t, \rho_{1-t}) \ge \norm{\rho_t - \rho_{1-t}}_2^2 / (8 M_\alpha)$, and \eqref{eq:signal-l2} gives \eqref{eq:signal-js}, since $c_0^2 / (8M) = e^{-1}/(32 M)$.
Finally, $(x - y)_+^2 \ge x^2/2 - y^2$ for $x, y \ge 0$ and, when $\alpha_t \ge 1/2$, $\alpha_t^{-s} \le 2^s$ give the last bound.
\end{proof}

\begin{corollary}[A bandwidth on a fixed path]
\label{cor:bandwidth}
Under the assumptions of \cref{thm:signal}, if $R (\gamma_t / \alpha_t)^s \le \omega/4$, then $J_t(\rho_0, \rho_1) \ge e^{-1} \omega^2 / (512 M)$.
When $\alpha_t \ge 1/2$, the condition $R (2\gamma_t)^s \le \omega/4$ suffices.
\end{corollary}

\begin{proof}[Proof of \cref{cor:bandwidth}]
If $R (h/\alpha_t)^s \le \omega/4$, the bracket in \eqref{eq:signal-js} is at least $\omega/4$, which gives $J_t \ge e^{-1} \omega^2 / (16 \cdot 32 M) = e^{-1} \omega^2 / (512 M)$.
When $\alpha_t \ge 1/2$, $R (h/\alpha_t)^s \le R (2h)^s$, so $R (2h)^s \le \omega/4$ suffices.
\end{proof}

\begin{remark}
The ratio of the noise to the small endpoint coefficient is $\gamma_t / \beta_t$, which equals $\sigma \alpha_t$ for the linear schedule of \cref{ex:linear}.
Taking this ratio large suppresses the cancellation term $v$, while a small $\gamma_t$ keeps the overall resolution fine.
This is why the signal proof uses a sufficiently large amplitude and a time approaching an endpoint.
It does not show that an arbitrary fixed time or any default implementation of a channel identifies all alternatives.
Constants depend on the ambient dimension, the support, the density bound and the smoothness, and no lower density bound or intrinsic-dimension assumption is hidden in the proof.
\end{remark}

Every $(h, \sigma)$-channel with $\sigma \ge \sigma_G$ satisfies the condition of the theorem.
For the linear schedule of \cref{ex:linear} with $\sigma \ge \sigma_G$, it holds at every $t \in (0, 1/2)$, because $\gamma_t / \beta_t = \sigma \alpha_t \ge \sigma/2$.


\cref{thm:signal} holds for every $s > 0$ and uses no lower bound on the densities, and
the same argument applies to any noise with an integrable gradient whose Fourier transform
is bounded away from zero near the origin (\cref{lem:signal-general}). It concerns the population
divergence; power for the empirical statistic is the subject of \cref{sec:minimax,sec:adaptive}.

\subsection{Minimax Optimality with a Designed Noise}
\label{sec:minimax}


Turning the population signal into power requires controlling the label-dependent
denominator of the empirical divergence (\cref{rem:denominator}). Gaussian noise gives
positive frames but no density floor, and for it we only prove the conservative rate
$n^{-s/(2s+d)}$ (\cref{app:gaussian}). We therefore use the freedom allowed by \eqref{eq:interpolant} to choose the
law of $z$.
The idea is to add to a narrow noise component, which carries the signal at resolution $h$, a broad component that is flat over the region where all bridge centres lie, so that the noise itself supplies a floor.

\textbf{A narrow-plus-broad noise:}
Let $k$ be a nonnegative $C_c^1$ probability density supported in $B_1$, and let $K$ be a nonnegative $C_c^1$ probability density that is equal to a constant $c_K > 0$ on $B_{2L+1}$.
Fix $\lambda \in (0, 1)$ and define, for $0 < h \le 1$,
\begin{equation}
\label{eq:mixture}
    k_h = k^{[h]}, \qquad W_h = (1 - \lambda) k_h + \lambda K .
\end{equation}
The narrow component has scale $h$, and the broad component has a fixed scale.
To realise $W_h$ on the interpolant, we use in \eqref{eq:interpolant} a noise variable $z_h$ with density
\begin{equation}
\label{eq:zh}
    z \longmapsto (1 - \lambda) k(z) + \lambda h^d K(h z).
\end{equation}
At an $(h, \sigma)$-channel, $\gamma_t z_h$ has density $W_h$ exactly, because $\gamma_t = h$.
The same law of $z_h$ is used along the whole path of the channel, so no reference endpoint or second interpolant is introduced; only the law of the noise depends on $h$.
Let $\Omega=B_{L+1}$. Since $\alpha_t,\beta_t\ge0$ and $\alpha_t+\beta_t\le1$, every bridge
centre satisfies $|\alpha_tx+\beta_ty|\le(\alpha_t+\beta_t)L\le L$ for $x,y\in B_L$, so all
bridge centres lie in $B_L$, and for every $x\in\Omega$ and every such centre $c$,
\begin{equation}
\label{eq:floor}
    W_h(x - c) = (1 - \lambda) k_h(x - c) + \lambda c_K .
\end{equation}
This holds sample-wise for every labelling, including for the empirical pair mixtures.
The positive floor $\lambda c_K$ on $\Omega$ comes from the chosen noise, not from an assumption on $\rho_0$ or $\rho_1$.
An explicit choice of $k$ and $K$, with exact samplers built from normal, uniform and beta variables, is given in \cref{app:noise-family}.

\begin{theorem}[Minimax rate with a designed noise]
\label{thm:minimax}
Fix $s > 3d/4$, $L, M, R > 0$, and $\lambda, k, K$ as above.
Fix a level $\kappa \in (0, 1)$, a target type II error $\delta \in (0, 1)$, and a number $B \ge 1$ of independent uniform balanced permutations with $\lfloor \kappa (B+1) \rfloor \ge 1$.
Let $\sigma$ be a sufficiently large constant depending only on $d, L, M$ and $k$, let $h_n \asymp n^{-2/(4s+d)}$, and let $t_n$ be an $(h_n, \sigma)$-channel.
Consider the two-sample test at time $t_n$ with noise $z_{h_n}$ from \eqref{eq:zh} and exact integration.
There are constants $C$ and $n_*$ such that, for all $n \ge n_*$,
\begin{gather*}
    \sup_{\rho_0 = \rho_1} \P(\phat \le \kappa) \le \kappa, 
    \sup_{\substack{(\rho_0, \rho_1) \in \Cs \\ \norm{\rho_0 - \rho_1}_2 \ge C n^{-2s/(4s+d)}}} \P(\phat > \kappa) \le \delta .
\end{gather*}
The level statement holds at every finite sample size.
The same guarantee holds for the shared Monte Carlo statistic \eqref{eq:mc} if, for a sufficiently large constant $C'$,
\begin{equation}
\label{eq:mc-budget}
    J \ge C' n^{8s/(4s+d)} \log \frac{2(B+1)}{\kappa \delta} .
\end{equation}
\end{theorem}

\textbf{Proof idea:} The argument of \cref{thm:signal}, applied to the narrow
noise component, lower-bounds a weighted population signal
$S_h \ge c\,\omega^2 - C R^2 h^{2s}$. The noise floor $\lambda c_K$ lets the divergence of
every row be compared with a weighted quadratic form, whose permutation fluctuations are of
order $e_n = 1/(n h^{d/2})$, and balancing $h^{2s}$ against $e_n$ gives the rate. The
restriction $s > 3d/4$ is a sufficient condition of this uniform-denominator argument, not
a threshold below which the statistic fails (\cref{app:minimax}).

The budget \eqref{eq:mc-budget} is sufficient, not a sharp computational claim.
No test can improve the rate.

\begin{theorem}[Lower bound]
\label{thm:lower-bound}
Let $s > 0$, $L, R > 0$ and $M \abs{B_L} > 1$, and fix $\kappa, \delta \in (0,1)$ with $\kappa + \delta < 1$.
There exist $c > 0$ and $\bar{n}$ such that, for every $n \ge \bar{n}$ and every, possibly randomised, test $\Delta_n$ of level at most $\kappa$ over all null pairs in $\Cs$,
\begin{equation*}
    \sup_{\substack{(\rho_0, \rho_1) \in \Cs \\ \norm{\rho_0 - \rho_1}_2 \ge c\, n^{-2s/(4s+d)}}} \P_{\rho_0, \rho_1}[\Delta_n = 0] \ge \delta .
\end{equation*}
\end{theorem}

\Cref{thm:lower-bound} is a lower bound for all tests, not only for tests built from bridges.
It uses the same class \eqref{eq:class}, whose Sobolev restriction is on the density difference, and equal sample sizes throughout.

\begin{proof}[Proof of \cref{thm:lower-bound}]
The strict inequality $M \abs{B_L} > 1$ allows a smooth density $g_0$ supported strictly inside $B_L$, with $\norm{g_0}_\infty < M$, that is constant at a value $b_0 > 0$ on an interior cube $\mathsf{C}$.
For example, choose a smooth cutoff $\chi \in [0, 1]$ supported inside $B_L$, equal to one on $\mathsf{C}$, with $\int \chi > 1/M$, and normalise it.
The cube can be chosen within the region where a cutoff that approximates the ball indicator is identically 1.

Fix a real function $\upsilon \in C_c^\infty((0,1)^d)$ with $\int \upsilon = 0$ and $\norm{\upsilon}_2 = 1$.
Partition a fixed smaller cube within $\mathsf{C}$ into $m_h \asymp h^{-d}$ disjoint cubes of side $h$, with lower corners $x_j$.
For $\theta \in \{-1, 1\}^{m_h}$ and an amplitude $a > 0$, put
\begin{equation*}
    f_\theta(x) = a h^s \sum_{j=1}^{m_h} \theta_j \, \upsilon\big((x - x_j)/h\big), \qquad \rho_0 = g_0 + f_\theta, \qquad \rho_1 = g_0 .
\end{equation*}
Every perturbation has an integral of zero.
Since the supports are disjoint, $\norm{f_\theta}_\infty \le a h^s \norm{\upsilon}_\infty$.
For all sufficiently small $h$, uniformly over $\theta$, the perturbed densities are therefore nonnegative, at most $M$, and supported in $B_L$.
Moreover,
\begin{equation}
\label{eq:lower-separation}
    \norm{f_\theta}_2^2 = a^2 m_h h^{2s+d} \asymp a^2 h^{2s} .
\end{equation}

We verify the fractional Sobolev constraint directly.
Choose an integer $k > s$.
The disjoint supports of all derivatives imply $\norm{f_\theta}_{H^k} \le C a h^{s-k}$, while $\norm{f_\theta}_2 \le C a h^s$.
Hölder's inequality, applied to the Fourier definition of the Sobolev norm, gives
\begin{equation*}
    \norm{f_\theta}_{H^s} \le \norm{f_\theta}_2^{1 - s/k} \norm{f_\theta}_{H^k}^{s/k} \le C a .
\end{equation*}
Fix $a \le R/C$.
Then all constructed pairs lie in $\Cs$ for sufficiently small $h$.

Under the null, both groups have density $g_0$, and only the first group's law changes under the alternatives.
If $\mathcal{L}_\theta$ denotes the likelihood ratio for all $2n$ observations, independence and the flat baseline on the supports of the perturbations give
\begin{equation*}
    \E_0\big[\mathcal{L}_\theta \mathcal{L}_{\theta'}\big] = \Big(1 + \int \frac{f_\theta f_{\theta'}}{g_0}\Big)^n = \Big(1 + \zeta \sum_{j=1}^{m_h} \theta_j \theta'_j\Big)^n, \qquad \zeta = \frac{a^2 h^{2s+d}}{b_0} .
\end{equation*}
The bracket is positive for all sufficiently small $h$.
Mix uniformly over the signs and write $\bar{\mathcal{L}} = 2^{-m_h} \sum_\theta \mathcal{L}_\theta$.
The inequalities $1 + x \le e^x$ and $\cosh u \le e^{u^2/2}$ imply
\begin{equation}
\label{eq:chi-square}
    \E_0 \bar{\mathcal{L}}^2 \le \{\cosh(n \zeta)\}^{m_h} \le \exp\Big(\frac{m_h n^2 \zeta^2}{2}\Big) \le \exp\big(C a^4 n^2 h^{4s+d}\big) .
\end{equation}
Choose an integer grid resolution so that $h \asymp n^{-2/(4s+d)}$; the exponent is then at most $C' a^4$.
Reducing the fixed positive amplitude $a$ if necessary makes
\begin{equation*}
    \TV\big(\bar{\P}, \P_0\big) \le \tfrac{1}{2} \sqrt{\E_0 \bar{\mathcal{L}}^2 - 1} < 1 - \kappa - \delta ,
\end{equation*}
where $\bar{\P}$ is the uniform mixture of the alternatives.
For every test taking values in $[0, 1]$, its average power over these alternatives is at most its power under the null plus this total variation distance.
Hence, at least one alternative has a type II error greater than $\delta$.
By \eqref{eq:lower-separation}, every alternative lies at separation at least $c\, n^{-2s/(4s+d)}$ for a fixed $c > 0$, which proves the claim.
Independent algorithmic randomisation is already covered by allowing tests with values in $[0, 1]$.
\end{proof}

The nondegeneracy condition matters.
If $M \abs{B_L} < 1$, the class contains no densities, and if $M \abs{B_L} = 1$, its only density is the uniform density on $B_L$, up to null sets, so there are no alternatives.
The lower bound, therefore, does not assert a testing difficulty in either degenerate case.


\input{sections/04b_tools}           

\input{sections/04c_minimax_proof}   

\input{sections/04d_noise_family}    

\subsection{Adaptation by Rank Fusion}
\label{sec:adaptive}
\Cref{thm:minimax} uses the smoothness through $h_n$.
We remove this input by fusing channels over a deterministic bandwidth grid.
Each channel still computes the full empirical reflected divergence over all $n^2$ cross-pairs, and no data splitting, fitted smoothness, or replacement by an endpoint kernel statistic is used.

\textbf{A completely specified bandwidth rule:}
Fix the support bound $L$, which is a population assumption rather than an estimate obtained from the largest observed norm.
Fix also a finite set $\mathcal{N}$ of noise designs, chosen independently of the data and containing at least one compact plateau design \eqref{eq:zh} with a fixed $\lambda \in (0,1)$; Gaussian or Student-$t$ designs may be added, without any claim of sharp power for them.
For equal sample sizes $n \ge 2$, define
\begin{equation}
\label{eq:grid}
\begin{gathered}
    \sigma_n = 8 + \log n, \qquad J_n = \Big\lceil \frac{\log_2 n}{2d} \Big\rceil, \\
    \mathcal{G}_n = \{2^{-j} \colon 0 \le j \le J_n\}, \qquad C_n = \abs{\mathcal{N}} (J_n + 1).
\end{gathered}
\end{equation}
For every bandwidth $h \in \mathcal{G}_n$ and every design, we use an $(h, \sigma_n)$-channel, for instance the time \eqref{eq:th} with $\sigma = \sigma_n$, which is valid because $\sigma_n > 8$ and $h \le 1$.
We compute the statistic matrix of these $C_n$ channels with $B_n$ independent uniform balanced permutations, shared by all channels, and fuse it by \eqref{eq:rank-fusion} with uniform weights $w_c = 1/C_n$ and a temperature $\tau > 0$ fixed in advance.
The growing amplitude $\sigma_n$ avoids estimating the density bound $M$ to set a sufficient amplitude, since for every fixed class, it eventually exceeds the required constant.
This is an asymptotic prescription, not a claim of favourable finite-sample constants; a sufficiently large fixed $\sigma$ also works, with the grid multiplied by $h_{\max} = \min\{1, \sigma/8\}$.



\textbf{From one channel to aggregation:}
For $m \ge 1$ and $\tau > 0$, let $\Arank_m(\tau) = m^{\tau - 1} \sum_{j=1}^{m} j^{-\tau}$.

\begin{lemma}[Rank moments and a sufficient rejection condition]
\label{lem:rank-moments}
For every real statistic matrix, including ties, every channel $c$ satisfies
$(B+1)^{-1} \sum_{b} r_{bc}^{-\tau} \le \Arank_{B+1}(\tau)$.
Consequently, $(B+1)^{-1}\sum_b\Lambda_b\le\Arank_{B+1}(\tau)$ and
\[
  \widehat p_{\mathrm{fuse}}\le\min\bigl\{1,\ \Arank_{B+1}(\tau)/\Lambda_0,\ C\,\Lambda_0^{-1/\tau}\bigr\}.
\]
In particular, if some channel satisfies
\begin{equation}
  r_{0c}\le u_c(B),\qquad
  u_c(B)=\max\Bigl\{\Bigl(\frac{\kappa\,w_c}{\Arank_{B+1}(\tau)}\Bigr)^{1/\tau},\
  \frac{\kappa\,w_c^{1/\tau}}{C}\Bigr\},
  \label{eq:fusion-target}
\end{equation}
then $\widehat p_{\mathrm{fuse}}\le\kappa$; this condition is sufficient, not necessary.
\end{lemma}

\begin{proof}[Proof of \cref{lem:rank-moments}]
Let $m = B + 1$, and sort column $c$ in descending order.
A tied block occupying positions $\ell, \dots, k$ has the common inclusive rank $k/m$, and its contribution $(k - \ell + 1)(k/m)^{-\tau}$ is no larger than $\sum_{j=\ell}^{k} (j/m)^{-\tau}$.
Summing over the blocks and dividing by $m$ gives $m^{-1} \sum_b r_{bc}^{-\tau} \le m^{-1} \sum_{j=1}^{m} (j/m)^{-\tau} = \Arank_m(\tau)$.
The bound for $\Lambda_b$ follows by averaging with the weights, which sum to one.
Finally, each row counted by $\pfuse$ has $\Lambda_b \ge \Lambda_0$, so their number is at most $\sum_b \Lambda_b / \Lambda_0$, which gives $\pfuse \le \Arank_m(\tau) / \Lambda_0$.
For the third bound, suppose $\Lambda_b\ge\Lambda_0$. Since the weights sum to one,
$\Lambda_b\le\max_c r_{bc}^{-\tau}$, so $r_{bc}\le\Lambda_0^{-1/\tau}$ for some $c$. Because
$m\,r_{bc}=\#\{j:T_{jc}\ge T_{bc}\}$, the counting argument in the proof of
Theorem~\ref{thm:fusion-level} shows that, for each $c$, at most
$\lfloor\Lambda_0^{-1/\tau}m\rfloor$ rows have $r_{bc}\le\Lambda_0^{-1/\tau}$. A union bound
over the $C$ channels gives $\#\{b:\Lambda_b\ge\Lambda_0\}\le C\,\Lambda_0^{-1/\tau}m$, that is,
$\widehat p_{\mathrm{fuse}}\le C\,\Lambda_0^{-1/\tau}$. Finally, suppose
\eqref{eq:fusion-target} holds. If $r_{0c}\le(\kappa w_c/\mathcal A_m(\tau))^{1/\tau}$, then
$\Lambda_0\ge w_cr_{0c}^{-\tau}\ge\mathcal A_m(\tau)/\kappa$, and the second bound gives
$\widehat p_{\mathrm{fuse}}\le\kappa$. If $r_{0c}\le\kappa w_c^{1/\tau}/C$, then
$\Lambda_0\ge w_cr_{0c}^{-\tau}\ge(C/\kappa)^{\tau}$, and the third bound gives
$\widehat p_{\mathrm{fuse}}\le\kappa$.
For $0 < \tau < 1$, comparison with an integral gives $\Arank_m(\tau) \le m^{\tau - 1} \int_0^m x^{-\tau} \, \dd x = 1/(1 - \tau)$.
\end{proof}

\begin{theorem}[Adaptive separation rate for rank fusion]
\label{thm:adaptive}
Fix $d, L, M, R$, $\kappa, \delta \in (0,1)$, a temperature $\tau > 0$, and a finite set of designs as above.
Use uniform weights and a number $B_n\le C_n^{K}$ of permutations, for a fixed $K\ge1$, such that, at every channel,
\begin{equation}
\label{eq:fusion-budget}
\begin{gathered}
    k_c = \big\lfloor u_c(B_n) (B_n + 1) \big\rfloor \ge 1, \\
    \P\big\{\Bin\big(B_n, u_c(B_n)/4\big) \ge k_c\big\} \le \delta/4 .
\end{gathered}
\end{equation}
The rank-fusion test, whose tuning uses neither $s$ nor $R$ nor $M$, has level at most $\kappa$ at every sample size.
For every fixed $s > 3d/4$, there are constants $A_s$ and $n_s$ such that, for all $n \ge n_s$,
\begin{equation}
\label{eq:adaptive}
    \sup_{\substack{(\rho_0, \rho_1) \in \Cs \\ \norm{\rho_0 - \rho_1}_2 \ge A_s r_{n,s}^{\mathrm{ad}}}} \P(\pfuse > \kappa) \le \delta,
\end{equation}
where $r_{n,s}^{\mathrm{ad}} = (\sqrt{\log \log n}/n)^{2s/(4s+d)}$.
\end{theorem}

Explicit permutation budgets follow from a Chernoff bound (Lemma~\ref{lem:budget}): with
uniform weights, $B_n$ of order $C_n^{1/\tau}$, $C_n\log C_n$ or $C_n$ suffices for
$\tau<1$, $\tau=1$ or $\tau>1$, where $C_n\asymp\log n$. Because the target
\eqref{eq:fusion-target} never falls below $\kappa C_n^{-1-1/\tau}$, every larger budget is
covered as well: additional permutations never weaken the guarantee.
The shared Monte Carlo version satisfies the same bound with the numerical budgets of \cref{app:adaptive}, and any finite budget retains the level.




\textbf{The price of adaptation:} The factor $\sqrt{\log\log n}$ arises because the
channel nearest the oracle bandwidth must clear a rank threshold that shrinks with the
number $C_n \asymp \log n$ of channels; the temperature changes the permutation budget but
not the statistical rate. Since \cref{thm:lower-bound} concerns known smoothness, we do not
claim that this factor is necessary.

\input{sections/04e_adaptive_proof}  

\input{sections/04f_gaussian}        

%% file: sections/04a_notation.tex
\subsection{Notation and Elementary Facts}
\label{app:notation}

\paragraph{General notation:}
We write $\abs{x}$ for the Euclidean norm, $B_L = \{x \colon \abs{x} \le L\}$, $\abs{B_L}$ for its volume, and $K_L = (\int_{B_L} \abs{x}^2 \dd x)^{1/2}$.
For a function or signed measure $g$ and $c > 0$, $g^{[c]}(x) = c^{-d} g(x/c)$; then $(g * g')^{[c]} = g^{[c]} * g'^{[c]}$ and $\widehat{g^{[c]}}(\xi) = \widehat{g}(c \xi)$.
If $0 \le g \le M$ is a density, then $g^{[c]} \le c^{-d} M$, and convolution with a probability density preserves upper bounds.
We write $G_h = G_1^{[h]}$ for the centred Gaussian density with covariance $h^2 \Id_d$, so that $\norm{\nabla G_h}_1 = h^{-1} \norm{\nabla G_1}_1$ and $\norm{\nabla G_1}_1 = \E\abs{z}$ for $z \sim \mathcal{N}(0, \Id_d)$.
Throughout a proof concerning a fixed time $t \in (0, 1/2)$, we abbreviate $\alpha = \alpha_t$, $\beta = \beta_t$ and $h = \gamma_t$, so that the mirrored frame at $1 - t$ has coefficients $(\beta, \alpha, h)$ under a symmetric schedule.
For the endpoint laws we write $f = \rho_0 - \rho_1$, $\mu = (\rho_0 + \rho_1)/2$ and $\omega = \norm{\rho_0 - \rho_1}_2$, and for the empirical laws $\fhat = \rhohat_0 - \rhohat_1$ and $\muhat = (\rhohat_0 + \rhohat_1)/2$.

More generally, subscripts on $\mathbb{P}$ and $\mathbb{E}$ indicate the distributions of the
random variables involved, as in $\mathbb{E}_{x\sim\rho}[g(x)]$, and unsubscripted
probabilities and expectations are taken over all random variables in the expression.

\paragraph{Divergences:}
For probability densities $g_0, g_1$, let $H^2(g_0, g_1) = \int (\sqrt{g_0} - \sqrt{g_1})^2$ be the squared Hellinger distance, $\LC(g_0, g_1) = \int (g_0 - g_1)^2 / (g_0 + g_1)$ the Le Cam divergence, with integrand zero where $g_0 + g_1 = 0$, and $\KL$ the Kullback--Leibler divergence.
The Hellinger distance $H$ is a metric, and
\begin{equation}
\label{eq:divergence-chain}
    \tfrac{1}{4} H^2(g_0, g_1) \le \JS(g_0, g_1) \le H^2(g_0, g_1) \le \KL(g_0 \,\Vert\, g_1).
\end{equation}
The lower bound follows from $F(0) = F'(0) = 0$ and $F''(z) = 1/(1-z^2) \ge 1$.
Indeed, \eqref{eq:F-bounds} gives $\JS \ge \LC/4$, the bound $\log x \le x - 1$ gives $\JS \le \LC/2$, and $H^2 \le \LC \le 2 H^2$ follows from $(\sqrt{a} + \sqrt{b})^2 \in [a + b, 2(a + b)]$.
The Kullback--Leibler inequality follows from $-\log x \ge 1 - x$ applied to $x = \sqrt{g_1/g_0}$.
If $g_0, g_1 \le M'$, then also $\JS(g_0, g_1) \ge \LC(g_0, g_1)/4 \ge \norm{g_0 - g_1}_2^2 / (8 M')$.

%% file: sections/04b_tools.tex
\subsection{Hellinger and Quadratic Tools}
\label{app:tools}

This subsection collects tools used in \cref{app:minimax,app:adaptive,app:gaussian}.
Throughout, we fix one channel with coefficients $(\alpha, \beta, h)$, where $\alpha, \beta \ge 0$ and $\alpha + \beta \le 1$, so that every bridge centre $\alpha x + \beta y$ with $x, y \in B_L$ lies in $B_L$.
For a noise density $\varphi$, we define the envelope and its integral
\begin{equation}
\label{eq:envelope}
    E_h(x) = \sup_{c \in B_L} \varphi^{[h]}(x - c), \qquad V_{\varphi, h} = \int E_h(x) \, \dd x \ge 1 .
\end{equation}
For the Gaussian, $B_L \subseteq [-L, L]^d$ and the factorisation of $G_h$ give
\begin{equation}
\label{eq:vh-gauss}
    V_{G, h} \le \prod_{k=1}^{d} \int_\R \sup_{\abs{c} \le L} \frac{e^{-(x - c)^2/(2h^2)}}{\sqrt{2\pi} h} \, \dd x = \Big(1 + \sqrt{\tfrac{2}{\pi}} \, \tfrac{L}{h}\Big)^d =: V_h .
\end{equation}
For a compactly supported kernel $\varphi$ supported in $B_1$, $V_{\varphi, h} \le \norm{\varphi}_\infty \abs{B_{L+h}} h^{-d} \le C h^{-d}$ for $h \le 1$.

\subsubsection{Cross-Pair Hellinger Estimation}

\begin{lemma}[Cross-pair Hellinger estimation]
\label{lem:hellinger-pairs}
Let $x_0^{(1)}, \dots, x_0^{(n_0)}, x_1^{(1)}, \dots, x_1^{(n_1)}$ be independent, not necessarily identically distributed, and supported in $B_L$.
Define
\begin{equation*}
    \widehat{r}(x) = \frac{1}{n_0 n_1} \sum_{i,j} \varphi^{[h]}\big(x - \alpha x_0^{(i)} - \beta x_1^{(j)}\big), \qquad r = \E \widehat{r} .
\end{equation*}
Then, with $\nu_n = (n_0 + n_1 - 1)/(n_0 n_1)$,
\begin{equation}
\label{eq:hellinger-pairs}
    \E H^2(\widehat{r}, r) \le \nu_n V_{\varphi, h} .
\end{equation}
\end{lemma}

\begin{proof}
Put $K_{ij}(x) = \varphi^{[h]}(x - \alpha x_0^{(i)} - \beta x_1^{(j)})$.
Every $K_{ij}$ lies between zero and $E_h$, and kernels with disjoint index pairs are independent.
For every pair of overlapping index pairs, including identical ones,
\begin{equation*}
    \Cov(K_{ij}, K_{kl}) \le \E K_{ij} K_{kl} \le \frac{E_h}{2} \big(\E K_{ij} + \E K_{kl}\big).
\end{equation*}
Each index pair overlaps exactly $n_0 + n_1 - 1$ index pairs.
Summing the ordered covariances therefore proves $\Var(\widehat{r}(x)) \le \nu_n E_h(x) r(x)$.
Since $(\sqrt{\widehat{r}} - \sqrt{r})^2 \le (\widehat{r} - r)^2 / r$, Tonelli's theorem bounds the desired expectation by $\nu_n \int E_h$.
This argument explicitly accounts for the dependence between cross-pairs.
\end{proof}

\subsubsection{A Permuted First Moment under Alternatives}

Fix one channel and let $T^\pi = \JS(\rhohat_t^\pi, \rhohat_{1-t}^\pi)$ denote the statistic after a uniformly random reassignment of the pooled observations into groups of sizes $n_0$ and $n_1$, independent of the observations.

\begin{lemma}[A permuted first moment]
\label{lem:permuted-first-moment}
Assume $n_0, n_1 \ge 2$.
For arbitrary endpoint distributions supported in $B_L$,
\begin{equation}
\label{eq:permuted-first-moment}
    \E T^\pi \le 18 \, \nu_n V_{\varphi, h},
\end{equation}
where the expectation includes the observations and the random permutation.
\end{lemma}

\begin{proof}
Let $N = n_0 + n_1$, $w = n_0/N$ and $\mu_w = w \rho_0 + (1 - w) \rho_1$.
Condition on the assignment of original observation indices to the two permuted groups.
The variables remain independent, though not identically distributed.
If $J_P$ original $\rho_0$-observations enter the first group, the average population laws of the two groups are
\begin{equation*}
    A = \theta_A \rho_0 + (1 - \theta_A) \rho_1, \quad \theta_A = \frac{J_P}{n_0}, \qquad A' = \theta_{A'} \rho_0 + (1 - \theta_{A'}) \rho_1, \quad \theta_{A'} = \frac{n_0 - J_P}{n_1} .
\end{equation*}
The conditional means of the two empirical bridge densities are $r_{AA'} = A^{[\alpha]} * A'^{[\beta]} * \varphi^{[h]}$ and $r_{A'A} = A'^{[\alpha]} * A^{[\beta]} * \varphi^{[h]}$, and \cref{lem:hellinger-pairs} applies conditionally to each of them.
Let $c = \mu_w^{[\alpha]} * \mu_w^{[\beta]} * \varphi^{[h]}$.
The triangle inequality for $H$, \eqref{eq:divergence-chain} and data processing give
\begin{equation*}
    H^2(r_{AA'}, r_{A'A}) \le 2 H^2(r_{AA'}, c) + 2 H^2(r_{A'A}, c) \le 4 \big\{\KL(A \,\Vert\, \mu_w) + \KL(A' \,\Vert\, \mu_w)\big\}.
\end{equation*}
Projecting the joint law of a Bernoulli mixture label and its observation, and then using $\KL \le \chi^2$, yields for $0 \le \theta \le 1$
\begin{equation*}
    \KL\big(\theta \rho_0 + (1 - \theta) \rho_1 \,\Vert\, \mu_w\big) \le \frac{(\theta - w)^2}{w(1 - w)} .
\end{equation*}
The hypergeometric distribution of $J_P$ gives
\begin{equation*}
    \E \theta_A = \E \theta_{A'} = w, \qquad \frac{\Var(\theta_A) + \Var(\theta_{A'})}{w (1 - w)} = \frac{n_1/n_0 + n_0/n_1}{N - 1} \le \nu_n .
\end{equation*}
The final inequality is equivalent to $n_0^2 + n_1^2 \le (N-1)^2$, which holds when $n_0, n_1 \ge 2$.
Using $\JS \le H^2$ and the triangle inequality for $H$ along the three links $\rhohat_t^\pi, r_{AA'}, r_{A'A}, \rhohat_{1-t}^\pi$ gives
\begin{equation*}
    \E T^\pi \le \E H^2\big(\rhohat_t^\pi, \rhohat_{1-t}^\pi\big) \le 3 \big\{2 \nu_n V_{\varphi, h} + 4 \nu_n\big\} \le 18 \, \nu_n V_{\varphi, h},
\end{equation*}
since $V_{\varphi, h} \ge 1$.
\end{proof}

\subsubsection{Conditional-Expectation Contraction}

We now fix the population weight
\begin{equation*}
    q_\mu = \big(\mu^{[\alpha]} * \mu^{[\beta]}\big) * \varphi^{[h]}, \qquad \Hil_\mu = L^2\big(\R^d, \dd z / q_\mu(z)\big),
\end{equation*}
and define two conditional density kernels and their difference,
\begin{equation}
\label{eq:features}
\begin{gathered}
    p_\alpha(z \mid y) = \big(\mu^{[\beta]} * \varphi^{[h]}\big)(z - \alpha y), \qquad p_\beta(z \mid y) = \big(\mu^{[\alpha]} * \varphi^{[h]}\big)(z - \beta y), \\
    \psi_y = p_\alpha(\cdot \mid y) - p_\beta(\cdot \mid y), \qquad \Gram_\mu(y, y') = \langle \psi_y, \psi_{y'} \rangle_{\Hil_\mu} .
\end{gathered}
\end{equation}
Both conditional kernels average to $q_\mu$ under $y \sim \mu$, so $\E_\mu \psi_Y = 0$, and by \eqref{eq:contrast} the population contrast and the weighted signal are
\begin{equation}
\label{eq:weighted-signal}
    D = \rho_t - \rho_{1-t} = \int \psi_y \, (\rho_0 - \rho_1)(\dd y), \qquad S^{\mu}_h = \norm{D}_{\Hil_\mu}^2 .
\end{equation}
No positive lower bound on $\mu$ or $q_\mu$ is assumed; for Gaussian noise, $q_\mu$ is strictly positive but decays to zero in the tails.

\begin{lemma}[Conditional-expectation contraction and effective size]
\label{lem:contraction}
Let $C_\mu = \E_\mu[\psi_Y \otimes \psi_Y]$ be the covariance operator on $\Hil_\mu$.
Then
\begin{equation}
\label{eq:contraction}
    \norm{C_\mu}_{\mathrm{op}} \le 4, \qquad \tr C_\mu \le 4 V_{\varphi, h}, \qquad \E_{\mu \otimes \mu} \Gram_\mu(Y, Y')^2 \le 16 V_{\varphi, h} .
\end{equation}
These statements also hold when $\mu$ is any finitely supported probability measure in $B_L$.
\end{lemma}

\begin{proof}
For either conditional kernel $p$, a vector $u \in \Hil_\mu$ satisfies
\begin{equation*}
    \langle u, p(\cdot \mid Y) \rangle_{\Hil_\mu} = \E\Big[\frac{u(Z)}{q_\mu(Z)} \,\Big|\, Y\Big]
\end{equation*}
under the joint law $\mu(\dd y) \, p(z \mid y) \, \dd z$, whose $Z$-marginal is $q_\mu$.
Conditional expectation is an $L^2$ contraction, so the second moment of this inner product is at most $\norm{u}_{\Hil_\mu}^2$.
The difference between the two maps has a norm at most 2, which gives the operator bound of 4.
Each conditional kernel is bounded by $E_h$, because the bridge centres lie in $B_L$.
Hence $\E_\mu p(z \mid Y)^2 \le E_h(z) q_\mu(z)$ and
\begin{equation*}
    \tr C_\mu = \int \frac{\E_\mu \psi_Y(z)^2}{q_\mu(z)} \, \dd z \le 4 \int E_h = 4 V_{\varphi, h} .
\end{equation*}
Finally, $\E_{\mu \otimes \mu} \Gram_\mu(Y, Y')^2 = \tr(C_\mu^2) \le \norm{C_\mu}_{\mathrm{op}} \tr C_\mu$.
All arguments use only bounded support and smoothing, so they apply to a discrete empirical mixing law.
\end{proof}

\begin{corollary}[Signal in the weighted comparison]
\label{cor:weighted-signal}
Under the assumptions of \cref{lem:signal-general} at the channel $(\alpha, \beta, h)$,
\begin{equation}
\label{eq:weighted-signal-bound}
    S^{\mu}_h \ge \frac{c_\varphi^2}{M} \Big[\frac{\omega}{2} - R\Big(\frac{h}{\alpha r_\varphi}\Big)^s\Big]_+^2, \qquad S^{\mu}_h \le 16 .
\end{equation}
For Gaussian noise, $c_\varphi^2 / M = e^{-1}/(4M)$.
\end{corollary}

\begin{proof}
Since $q_\mu \le \alpha^{-d} M$, the lower bound follows from \cref{lem:signal-general}.
For the upper bound, put $\ell = (\rho_0 - \rho_1)/\mu$, with $\ell = 0$ where $\mu = 0$.
Then $\abs{\ell} \le 2$, $\norm{\ell}_{L^2(\mu)}^2 \le 4$, and $D = \E_\mu[\ell(Y) \psi_Y]$.
The feature map has squared operator norm at most four by \cref{lem:contraction}, so $\norm{D}_{\Hil_\mu} \le 4$.
\end{proof}

\subsubsection{Sharp Quadratic Fluctuation Bounds without a Density Floor}

This subsection uses the unknown population weight $q_\mu$ as a device for the proofs; it is not the implemented test.
Throughout, the two sample sizes equal $n \ge 2$, and the data are $\Dzero = \{x_0^{(i)}\}$ and $\Done = \{x_1^{(j)}\}$.
Define the unbiased quadratic form
\begin{equation}
\label{eq:ustat}
\begin{aligned}
    U_\mu = {} & \frac{1}{n(n-1)} \sum_{i \ne j} \Gram_\mu\big(x_0^{(i)}, x_0^{(j)}\big) + \frac{1}{n(n-1)} \sum_{i \ne j} \Gram_\mu\big(x_1^{(i)}, x_1^{(j)}\big) \\
    & - \frac{2}{n^2} \sum_{i,j} \Gram_\mu\big(x_0^{(i)}, x_1^{(j)}\big),
\end{aligned}
\end{equation}
and, for any balanced labelling $b$ of the pooled observations,
\begin{equation*}
    Q_\mu^b = \Big\| (\rhohat^b_0 - \rhohat^b_1) \psi \Big\|_{\Hil_\mu}^2, \qquad \bar{Q}_\mu = \E_\pi\big[Q_\mu^\pi \mid \Dzero, \Done\big],
\end{equation*}
where $(\rhohat^b_0 - \rhohat^b_1) \psi = \int \psi_y \, (\rhohat^b_0 - \rhohat^b_1)(\dd y)$.

\begin{theorem}[Sharp comparison fluctuations]
\label{thm:fluctuations}
The following bounds hold without a density floor, with $V = V_{\varphi, h}$:
\begin{gather}
\label{eq:ustat-moments}
    \E U_\mu = S^{\mu}_h, \qquad \Var(U_\mu) \le \frac{32 S^{\mu}_h}{n} + \frac{768 V}{n^2}, \\
\label{eq:perm-centring}
    Q_\mu^0 - \bar{Q}_\mu = \frac{2(n-1)}{2n-1} U_\mu, \qquad \E \Var_\pi\big(Q_\mu^\pi \mid \Dzero, \Done\big) \le \frac{4096 V}{n^2} .
\end{gather}
In particular, for Gaussian noise, the degenerate fluctuation scale is $\sqrt{V_h}/n$, which is of order $1/(n h^{d/2})$ as $h \downarrow 0$.
\end{theorem}

\begin{proof}
For $g(y) = \langle \psi_y, D \rangle_{\Hil_\mu}$, the first-order part of the two-sample Hoeffding decomposition has variance
\begin{equation*}
    \frac{4}{n} \big\{\Var_{\rho_0} g(X) + \Var_{\rho_1} g(Y)\big\} \le \frac{8}{n} \E_\mu g(Y)^2 \le \frac{32 S^{\mu}_h}{n},
\end{equation*}
using $\rho_0, \rho_1 \le 2\mu$ and $\norm{C_\mu}_{\mathrm{op}} \le 4$.
Because $\rho_0, \rho_1 \le 2\mu$ as measures, every uncentred second-order kernel moment under $\rho_0 \otimes \rho_0$, $\rho_1 \otimes \rho_1$ or $\rho_0 \otimes \rho_1$ is at most $4 \E_{\mu \otimes \mu} \Gram_\mu^2 \le 64 V$.
Centring in both variables is an orthogonal projection and cannot increase this moment.
Orthogonality of the Hoeffding components bounds the remaining variance by
\begin{equation*}
    \frac{256 V}{n(n-1)} + \frac{256 V}{n^2} \le \frac{768 V}{n^2} .
\end{equation*}
The expectation follows by expanding the population quadratic form, since $\E \Gram_\mu(X, X') = \langle \E \psi_X, \E \psi_{X'} \rangle$.

The permutation identity and the variance bound are algebraic and apply to any fixed symmetric matrix $\Gram = (\Gram_{ij})_{i,j \le N}$ with $N = 2n$.
Write $Q = n^{-2} \varepsilon^\top \Gram \varepsilon$ for uniform balanced signs $\varepsilon \in \{-1, 1\}^N$.
The relations $\E \varepsilon_i \varepsilon_j = -1/(N-1)$ for $i \ne j$ show that all diagonal terms cancel in $Q^0 - \E_\pi Q$, and the remaining within-group and cross-group coefficients give the identity in \eqref{eq:perm-centring} exactly.
For the variance, let
\begin{gather*}
    u = -\frac{1}{N-1}, \quad v = \frac{3}{(N-1)(N-3)}, \quad V_\Gram = \sum_{i<j} \Gram_{ij}^2, \\
    A_\Gram = \sum_i \Big(\sum_{j \ne i} \Gram_{ij}\Big)^2, \quad L_\Gram = \sum_{i<j} \Gram_{ij} .
\end{gather*}
The four-distinct-index sign moment is $v$, obtained by multiplying $\sum_i \varepsilon_i = 0$ by three distinct signs and taking expectations.
Classifying pairs of index pairs as identical, overlapping or disjoint gives
\begin{equation*}
    \Var_\pi\big(\varepsilon^\top \Gram \varepsilon\big) = 4 \big\{(1 - 2u + v) V_\Gram + (u - v) A_\Gram + (v - u^2) L_\Gram^2\big\}.
\end{equation*}
Since $u - v < 0$, $v - u^2 > 0$ and $L_\Gram^2 \le N(N-1) V_\Gram / 2$, this is at most $8 N V_\Gram / (N - 3) \le 32 V_\Gram$.
Thus
\begin{equation}
\label{eq:perm-variance}
    \Var_\pi\big(Q \mid \Dzero, \Done\big) \le \frac{32}{n^4} \sum_{i < j} \Gram_{ij}^2 .
\end{equation}
Each distinct pooled pair has a product law dominated by $4 \mu \otimes \mu$, so its squared kernel has expectation at most $64 V$, and there are at most $2n^2$ such pairs, which proves the last bound.
\end{proof}

\begin{remark}[Why the diagonal is not estimated separately]
This argument never assumes a bounded second moment of $\Gram_\mu(Y, Y)$.
Small isolated components can have a large density-ratio leverage without a density floor.
The exact cancellation of the diagonal and the off-diagonal Hilbert-Schmidt bound avoids that unnecessary assumption.
\end{remark}




\subsubsection{Exact Remainders for the Empirical Pool}

For finite signed measures, define the antisymmetric bilinear map
\begin{equation}
\label{eq:bilinear}
    \Bil(\nu_1, \nu_2) = \big(\nu_1^{[\alpha]} * \nu_2^{[\beta]} - \nu_2^{[\alpha]} * \nu_1^{[\beta]}\big) * \varphi^{[h]} .
\end{equation}
For equal sample sizes, \eqref{eq:emp-contrast} and \eqref{eq:contrast} read $\rhohat_t - \rhohat_{1-t} = \Bil(\fhat, \muhat)$ and $\rho_t - \rho_{1-t} = \Bil(f, \mu)$.
Write $\widehat{D} = \Bil(\fhat, \muhat)$ for the empirical contrast, $A_n = \Bil(\fhat, \mu)$ for the contrast with the population pool, and $D = \Bil(f, \mu)$.

\begin{lemma}[Exact remainder]
\label{lem:remainder-identity}
With $\Delta\rho_0 = \rhohat_0 - \rho_0$, $\Delta\rho_1 = \rhohat_1 - \rho_1$ and $\Delta\mu = (\Delta\rho_0 + \Delta\rho_1)/2$,
\begin{equation}
\label{eq:remainder-identity}
    \widehat{D} - A_n = \Bil(f, \Delta\mu) + \Bil(\Delta\rho_0, \Delta\rho_1) .
\end{equation}
\end{lemma}

\begin{proof}
Since $\widehat{D} - A_n = \Bil(\fhat, \Delta\mu)$ and $\fhat = f + \Delta\rho_0 - \Delta\rho_1$, the quadratic remainder is $\Bil(\Delta\rho_0 - \Delta\rho_1, (\Delta\rho_0 + \Delta\rho_1)/2) = \Bil(\Delta\rho_0, \Delta\rho_1)$, because antisymmetry eliminates both same-sample terms.
\end{proof}

\begin{lemma}[Remainder for a permuted row]
\label{lem:permuted-remainder}
Let $\Hil$ be a Hilbert space of functions in which the maps below are defined, and suppose that, for all laws $L, L' \in \{\rho_0, \rho_1\}$, $\E_{X \sim L, Y \sim L'} \norm{\Bil(\delta_X, \delta_Y)}_{\Hil}^2 \le \Gamma_h$.
For a uniform balanced permutation, let $\widehat{D}^\pi = \Bil(\fhat^\pi, \muhat)$ and $A^\pi_n = \Bil(\fhat^\pi, \mu)$.
Then
\begin{equation}
\label{eq:permuted-remainder}
    \E \norm{\widehat{D}^\pi - A_n^\pi}_{\Hil}^2 \le \frac{4 \Gamma_h}{N(N-1)}, \qquad \E \norm{A_n^\pi}_{\Hil}^2 \le \frac{4 \Gamma_h}{N-1} .
\end{equation}
\end{lemma}

\begin{proof}
Let $N = 2n$, and let the independent pooled observations $V_i$ have original laws $L_i \in \{\rho_0, \rho_1\}$.
Then $\muhat - \mu = N^{-1} \sum_j (\delta_{V_j} - L_j)$, and with balanced signs $\varepsilon$, $\widehat{D}^\pi - A^\pi_n = n^{-1} \sum_i \varepsilon_i \psi_i$, where $\psi_i = \Bil(\delta_{V_i}, \muhat - \mu)$.
The balanced-sign formula of \cref{lem:balanced-moment} gives
\begin{equation*}
    \E_\pi \norm{\widehat{D}^\pi - A^\pi_n}^2 = \frac{4}{N(N-1)} \sum_i \norm{\psi_i - \bar{\psi}}^2 \le \frac{4}{N(N-1)} \sum_i \norm{\psi_i}^2 .
\end{equation*}
Condition on $V_i$.
In the expansion $\psi_i = N^{-1} \sum_j \Bil(\delta_{V_i}, \delta_{V_j} - L_j)$, the terms indexed by $j \ne i$ are independent and centred, and the diagonal term is $-N^{-1} \Bil(\delta_{V_i}, L_i)$, because $\Bil(\delta_{V_i}, \delta_{V_i}) = 0$.
Since variances and Jensen's inequality are bounded by raw second moments, $\E \norm{\psi_i}^2 \le N^{-2} \cdot N \Gamma_h = \Gamma_h / N$, and summing over $i$ gives the first bound.
The same balanced-sign formula, applied to the frozen features $\Bil(\delta_{V_i}, \mu)$, whose second moments are at most $\Gamma_h$, gives the second.
\end{proof}

%% file: sections/04c_minimax_proof.tex
\subsection{Proof of the Minimax Upper Bound}
\label{app:minimax}

\textbf{Overview:}
The proof (\cref{app:minimax}) has four steps.
First, the argument of \cref{thm:signal}, applied to the narrow kernel, shows that the weighted population signal $S_h = \frac{1}{8} \int_\Omega D_*^2 / q_*$, with $D_* = \rho_t - \rho_{1-t}$ and $q_* = \mu^{[\alpha_t]} * \mu^{[\beta_t]} * W_h$, satisfies $S_h \ge c\,\omega^2 - C R^2 h^{2s}$.
Second, the empirical pair densities of the observed and permuted rows concentrate around their means on $\Omega$ at the rate $\xi_n \asymp \sqrt{\log n / (n h^d)}$, and they never fall below the floor $\lambda c_K$.
Third, using $z^2/2 \le F(z) \le (z^2/2)(1 + C z^2)$ on $\Omega$ and a Hellinger bound outside it, the actual divergence of each row is compared with a weighted quadratic form, up to a relative error of order $\min\{\omega, 1\} + \xi_n$ for the observed row and $\xi_n$ for permuted rows.
Fourth, the quadratic forms admit an exact permutation-centring identity, under which the observed form exceeds its permutation mean by a fixed multiple of $S_h$, while the permuted forms fluctuate at the scale $e_n = 1/(n h^{d/2})$.
The bandwidth $h_n$ balances the bias $h^{2s}$ against $e_n$, which gives the rate.
The numerator errors are negligible when $s > d/2$, and the denominator errors when $s > 3d/4$.
The restriction $s > 3d/4$ is therefore a sufficient condition from the present uniform-denominator argument; it is not a lower bound showing that the statistic fails at smaller $s$.

\paragraph{Setting:}
Throughout this subsection, the sample sizes are equal to $n$, $N = 2n$, and we fix an $(h, \sigma)$-channel at time $t$, with $\alpha = \alpha_t \in [1/2, 1]$, $\beta = \beta_t$, $\alpha + \beta \le 1$, $h = \gamma_t \le 1$ and $\beta / h \le 2/\sigma$.
The noise is $z_h$ from \eqref{eq:zh}, so that the additive noise at the two mirrored times has density $W_h$, and $\Omega = B_{L+1}$.
Let $\Bil_h$ be the bilinear map \eqref{eq:bilinear} with $\varphi^{[h]}$ replaced by $W_h$, and define the population pair density, contrast and pair average
\begin{gather*}
    q_* = \big(\mu^{[\alpha]} * \mu^{[\beta]}\big) * W_h, \qquad D_* = \Bil_h(f, \mu) = \rho_t - \rho_{1-t}, \\
    m_* = \frac{\rho_t + \rho_{1-t}}{2} = q_* - \frac{1}{4} \big(f^{[\alpha]} * f^{[\beta]}\big) * W_h .
\end{gather*}
For the observed row $b = 0$ and the permuted rows $b = 1, \dots, B$, we write $\widehat{D}^b = \rhohat^b_t - \rhohat^b_{1-t}$ and $\widehat{m}^b = (\rhohat^b_t + \rhohat^b_{1-t})/2$, so that $T^b = \JS(\rhohat^b_t, \rhohat^b_{1-t})$.
We use the scales
\begin{equation*}
    v_n = \frac{1}{n h^d}, \qquad e_n = \frac{1}{n h^{d/2}}, \qquad \theta_n = \frac{v_n}{\sqrt{e_n}} = (n h^{3d/2})^{-1/2} .
\end{equation*}
On $\Omega$, the broad part of every bridge density equals $\lambda c_K$ by \eqref{eq:floor}, and the narrow part of the contrast is supported in $\Omega$, because $k_h$ is supported in $B_h \subseteq B_1$.
Consequently, $q_*$ and $m_*$ lie in $[m_0, M_0]$ on $\Omega$, with $m_0 = \lambda c_K$ and $M_0 = (1 - \lambda) 2^d M + \lambda c_K$, and on $\Omega$ the contrast $D_*$ equals $(1 - \lambda) \Bil(f, \mu)$ with the narrow kernel $k_h$ in place of the noise.

\subsubsection{Population Signal for the Compact Narrow Kernel}

Choose $r_0, c_k > 0$ such that $\abs{\widehat{k}(\xi)} \ge c_k$ for $\abs{\xi} \le r_0$, which is possible because $\widehat{k}(0) = 1$ and $\widehat{k}$ is continuous, and put $c_* = c_k / 2$.
A sufficient amplitude condition is
\begin{equation}
\label{eq:sigma-k}
    \sigma \ge \sigma_k = \max\bigg\{4 L r_0, \ \frac{4 \sqrt{M} K_L \norm{\nabla k}_1}{c_*}\bigg\} .
\end{equation}
\Cref{lem:signal-general} with $\varphi = k$, $r_\varphi = r_0$ and $c_\varphi = c_*$ gives
\begin{equation}
\label{eq:signal-narrow}
    \Big\| \big(\rho_0^{[\alpha]} * \rho_1^{[\beta]} - \rho_1^{[\alpha]} * \rho_0^{[\beta]}\big) * k_h \Big\|_2 \ge \alpha^{-d/2} c_* \Big[\frac{\omega}{2} - R\Big(\frac{h}{\alpha r_0}\Big)^s\Big]_+ .
\end{equation}
Let
\begin{equation*}
    S_h = \frac{1}{8} \int_\Omega \frac{D_*^2}{q_*} \, \dd x .
\end{equation*}
Since $q_* \le M_0$ on $\Omega$ and the entire narrow contrast is supported in $\Omega$, \eqref{eq:signal-narrow} and $(x - y)_+^2 \ge x^2/2 - y^2$ give
\begin{equation}
\label{eq:signal-core}
    S_h \ge c\,\omega^2 - C_1 R^2 h^{2s}, \qquad \omega^2 \le C_2 \big(S_h + h^{2s}\big).
\end{equation}
For the second inequality, either $\omega \ge 4 R (2h/r_0)^s$, in which case the bracket in \eqref{eq:signal-narrow} is at least $\omega/4$, or $\omega^2 \le C h^{2s}$.
The constants depend on the fixed class and the noise design, not on $n$.

\subsubsection{Uniform Concentration on the Core}

Here and below, a superscript $0$ denotes the observed labelling and a superscript $\pi$ a uniform balanced permutation.
Conditional on any particular assignment of the original observations to the two groups, the observations remain independent, though not necessarily identically distributed.

\begin{lemma}[Core concentration]
\label{lem:core}
For every fixed $B$ and $\eta > 0$, there is an event of probability at least $1 - \eta$ on which
\begin{gather}
\label{eq:core-observed}
    \norm{\widehat{m}^0 - m_*}_{\infty, \Omega} \le \xi_n, \\
\label{eq:core-permuted}
    \max_{1 \le b \le B} \norm{\widehat{m}^b - q_*}_{\infty, \Omega} + \max_{1 \le b \le B} \norm{\widehat{D}^b}_{\infty, \Omega} \le C \xi_n,
\end{gather}
where, for $h$ of polynomial order in $n$,
\begin{equation}
\label{eq:xi-n}
    \xi_n = C_{B, \eta} \bigg[\sqrt{\frac{\log n}{n h^d}} + \frac{\log n}{n h^d} + n^{-2}\bigg] .
\end{equation}
Every empirical pair-average density is at least $m_0$ on $\Omega$, without conditioning on this event.
\end{lemma}

\begin{proof}
The broad part is exactly $\lambda c_K$ on $\Omega$, so only the narrow part needs concentration.
For a fixed spatial point $x$, condition on the second group and write
\begin{equation*}
    g_Y(u) = \frac{1}{n} \sum_j k_h\big(x - \alpha u - \beta Y_j\big),
\end{equation*}
where $Y_j$ denotes the observations of the second group.
It is bounded by $C h^{-d}$, and, for each first-group observation $X_i$ whose density is at most $M$,
\begin{equation*}
    \E_i\, g_Y(X_i)^2 \le M \int g_Y(u)^2 \, \dd u \le M \alpha^{-d} \norm{k_h}_2^2 \le C h^{-d} .
\end{equation*}
Bernstein's inequality for the independent first-group observations therefore gives an error $C\{\sqrt{u/(n h^d)} + u/(n h^d)\}$ with failure probability at most $2 e^{-u}$.
The conditional mean is an average over the second group of $(\bar{\rho}^{[\alpha]} * k_h)(x - \beta Y_j)$, where $\bar{\rho}$ is the mean of the population densities of the first group.
Each such term is bounded by $\alpha^{-d} M$, and another Bernstein bound gives an error $C\{\sqrt{u/n} + u/n\}$, which is no larger when $h \le 1$.
Thus, this argument applies even after conditioning on an arbitrary permutation assignment.

The spatial Lipschitz constants of the empirical and population narrow densities are at most $C h^{-d-1}$, because $k \in C_c^1$.
A net of mesh $n^{-2} h^{d+1}$ on the fixed ball $\Omega$, followed by a union bound over its polynomially many points and the $2(B+1)$ densities, gives \eqref{eq:xi-n} for the deviations from the conditional population means.

It remains to identify those means for a permuted row.
If $J_P$ of the original $\rho_0$-observations enter its first group, put $a_P = 2 J_P / n - 1$.
The mean endpoint laws of the two groups are $\mu + a_P f/2$ and $\mu - a_P f/2$.
Bilinearity gives their reflected contrast $a_P D_*$ and their pair average $q_* - \frac{a_P^2}{4} (f^{[\alpha]} * f^{[\beta]}) * W_h$.
The hypergeometric variance is $\E a_P^2 = 1/(2n - 1)$, so Chebyshev's inequality and a union bound over the fixed $B$ rows give $\max_b \abs{a_{P,b}} \le C_{B, \eta} n^{-1/2}$ with the required confidence.
All population pair densities are bounded by $M_0$ on $\Omega$, so these deviations are covered by $C \xi_n$.
The observed row has its own population mean $m_*$.
Finally, the pointwise floor follows directly from \eqref{eq:floor} for every component.
\end{proof}

\subsubsection{The Nonlinear Comparison and the Exterior}

Define the weighted empirical quadratic form
\begin{equation}
\label{eq:qhat}
    \widehat{Q}^b = \frac{1}{8} \int_\Omega \frac{(\widehat{D}^b)^2}{q_*} \, \dd x .
\end{equation}
This is a device for the proof; the test computes the full divergence on $\R^d$.
The population pair average satisfies $m_* \le 2 q_*$, because $\abs{f} \le 2\mu$.
On the core, the broad part of $f^{[\alpha]} * f^{[\beta]} * W_h$ vanishes, since $\int f = 0$, and Young's inequality gives $\norm{f^{[\alpha]} * f^{[\beta]} * k_h}_\infty \le \norm{f^{[\alpha]}}_\infty \norm{f^{[\beta]}}_1 \le 2^{d+1} M \abs{B_L}^{1/2} \omega$.
Since $q_* \ge m_0$ on $\Omega$,
\begin{equation}
\label{eq:mstar-ratio}
    \frac{m_*}{q_*} \le 1 + u_\omega, \qquad u_\omega = \min\{C \omega, 1\} \qquad \text{on } \Omega .
\end{equation}
No assumption that $\omega$ tends to zero is made.

\begin{lemma}[Actual divergence compared with its quadratic form]
\label{lem:nonlinear}
On the event of \cref{lem:core}, for sufficiently large $n$,
\begin{equation}
\label{eq:nonlinear}
    T^0 \ge \frac{\widehat{Q}^0}{1 + u_\omega + C \xi_n}, \qquad T^b \le (1 + C \xi_n) \widehat{Q}^b + E^b_{\mathrm{out}}, \quad 1 \le b \le B,
\end{equation}
where $E^b_{\mathrm{out}} \ge 0$ is the contribution to the divergence from outside $\Omega$.
For a uniformly permuted row,
\begin{equation}
\label{eq:exterior}
    \E E^\pi_{\mathrm{out}} \le \frac{C}{n} .
\end{equation}
\end{lemma}

\begin{proof}
The scalar inequality $F(z) \ge z^2/2$ gives $T^0 \ge \int_\Omega (\widehat{D}^0)^2 / (8 \widehat{m}^0)$.
Combining \eqref{eq:core-observed} with \eqref{eq:mstar-ratio} gives $\widehat{m}^0 \le q_* (1 + u_\omega + \xi_n / m_0)$ on $\Omega$, and hence the lower bound.
For a permutation, $\widehat{m}^b / q_* = 1 + O(\xi_n)$ and $\abs{\widehat{D}^b / (2 \widehat{m}^b)} = O(\xi_n)$ uniformly on the core, by \eqref{eq:core-permuted} and the floor.
The expansion $F(z) \le (z^2/2)(1 + C z^2)$ for $\abs{z} \le 1/2$ then proves the upper bound on the contribution of the core.
No such small-contrast assumption is applied to the observed labelling.

Outside the core, the narrow components vanish, so the two densities equal $\lambda$ times their broad-only bridge densities, which use the kernel $K$ in place of $W_h$.
Since the integrand of the divergence is homogeneous of degree one, $E^\pi_{\mathrm{out}}$ is at most $\lambda$ times the divergence between the two broad-only bridge densities.
The cross-pair Hellinger proof of \cref{lem:hellinger-pairs} uses only a common kernel envelope with finite integral, so it applies to $K$ with $V_K = \int \sup_{c \in B_L} K(x - c) \, \dd x < \infty$.
Where a population density is zero, its nonnegative empirical counterpart is zero almost surely, so there is no division by zero in that argument.
The hypergeometric and data-processing proof of \cref{lem:permuted-first-moment} is unchanged for this kernel and gives an expectation of at most $18 (2n-1) V_K / n^2 \le 36 V_K / n$ for the broad-only divergence.
\end{proof}

\subsubsection{Centring the Quadratic Comparison}

Use the Hilbert norm $\norm{u}_*^2 = \frac{1}{8} \int_\Omega u^2 / q_*$ and the bilinear map $\Bil_h$ restricted to the core.
Let
\begin{equation*}
    A^b = \Bil_h\big(\rhohat_0^b - \rhohat_1^b, \mu\big), \qquad Q_A^b = \norm{A^b}_*^2, \qquad \bar{Q}_A = \E_\pi\big[Q^\pi_A \mid \Dzero, \Done\big].
\end{equation*}
The broad part of the feature difference is identically zero on the core.

\begin{lemma}[Quadratic estimates on the core]
\label{lem:centring}
There is an unbiased two-sample quadratic form $U_A$ such that
\begin{equation}
\label{eq:centring}
\begin{gathered}
    Q^0_A - \bar{Q}_A = \frac{2(n-1)}{2n-1} U_A, \qquad \E U_A = S_h, \\
    \Var(U_A) \le C \big(S_h / n + e_n^2\big), \qquad \E \Var_\pi\big(Q^\pi_A \mid \Dzero, \Done\big) \le C e_n^2, \qquad 0 \le \bar{Q}_A \le C v_n .
\end{gathered}
\end{equation}
At $h = h_n \asymp n^{-2/(4s+d)}$, uniformly over the class and for all sufficiently large $n$ with $v_n \le 1$,
\begin{equation}
\label{eq:numerator-errors}
    \E \big|\widehat{Q}^0 - Q^0_A\big| \le C (S_h + e_n) \theta_n, \qquad \E \big|\widehat{Q}^\pi - Q^\pi_A\big| \le C e_n \sqrt{v_n} .
\end{equation}
\end{lemma}

\begin{proof}
The two conditional density channels average to $q_*$.
The contraction argument of \cref{lem:contraction} applies with functions supported on $\Omega$ and the factor $1/8$ in the norm, and the envelope integral on the fixed core is at most $C h^{-d}$.
Thus, the Hoeffding and balanced-sign calculations of \cref{thm:fluctuations} give the first two lines of \eqref{eq:centring}.
Each feature diagonal is also at most $C h^{-d}$, since $q_* \ge m_0$ and the $L^2$ norms of the shifted narrow kernels are at most $C h^{-d/2}$.
The exact permutation trace formula, $\bar{Q}_A = n^{-2} \frac{N}{N-1} \sum_i \norm{\psi_i - \bar{\psi}}_*^2$ for the frozen features $\psi_i$, then gives the deterministic bound on $\bar{Q}_A$.

For the numerator replacement, the identity \eqref{eq:remainder-identity} still holds with $\Bil_h$.
The core weight is bounded below, so its estimates simplify.
A shifted $f^{[\alpha]} * k_h$ has $L^2$ norm at most $C \omega$, whereas a shifted $f^{[\beta]} * k_h$ has $L^2$ norm at most $\norm{f}_1 \norm{k_h}_2 \le C \omega h^{-d/2}$.
Independence and centring, therefore, give
\begin{equation*}
    \E \big\|\widehat{D}^0 - A^0\big\|_*^2 \le C \big(\omega^2 v_n + e_n^2\big), \qquad \E \big\|A^0 - D_*\big\|_*^2 \le C v_n .
\end{equation*}
The doubly centred pair term contributes $C e_n^2$, because its raw Hilbert second moment is at most $C h^{-d}$.
For a permuted row, \cref{lem:permuted-remainder} with $\Gamma_h \le C h^{-d}$ gives an error moment $C e_n^2$ and $\E \norm{A^\pi}_*^2 \le C v_n$.
Finally, \eqref{eq:signal-core} implies $\omega^2 \le C (S_h + e_n)$ at the prescribed bandwidth, since $h_n^{2s} \asymp e_n$.
The Cauchy--Schwarz inequality for the difference of two squared norms bounds the observed relative error by
\begin{equation*}
    C \big\{\sqrt{v_n} + \sqrt{e_n} + v_n / \sqrt{e_n} + v_n + e_n\big\} \le C \theta_n
\end{equation*}
when $h \le 1$ and $v_n \le 1$.
The permuted bound follows from $C(\sqrt{v_n} e_n + e_n^2) \le C e_n \sqrt{v_n}$.
\end{proof}

\subsubsection{Completion of the Power Proof}

Fix a small confidence parameter $\eta > 0$.
Let $S = S_h$, and suppose that $S \ge A e_n$, where $A \ge 1$ will be chosen sufficiently large.
Then \eqref{eq:signal-core} and \eqref{eq:mstar-ratio} imply
\begin{equation}
\label{eq:denominator-cost}
    \frac{u_\omega \bar{Q}_A}{S} \le C \frac{\omega v_n}{S} \le C \theta_n, \qquad \frac{\xi_n \bar{Q}_A}{S} \le C \frac{\xi_n v_n}{e_n} .
\end{equation}
This controls the denominator cost relative to the shared permutation baseline.
It does not bound the cost by a constant times the observed signal, nor does it assume that the alternatives are local.

Chebyshev's inequality applied to \eqref{eq:centring} shows that, except with probability at most $C/(nS) + C e_n^2 / S^2$, the centred observed form $Q^0_A - \bar{Q}_A$ exceeds a fixed positive multiple of $S$, while $\abs{Q^\pi_A - \bar{Q}_A}$ is smaller than an arbitrarily specified fixed multiple of $S$.
The constants in this probability bound may depend on that multiple.
Markov's inequality and \eqref{eq:numerator-errors} control the two numerator errors by the same multiples of $S$, at an additional cost $C \theta_n$.
The exterior error is controlled similarly, at cost $C/(nS)$, by \eqref{eq:exterior}.

On the core-concentration event, the ordering $T^\pi \ge T^0$ would imply, by \cref{lem:nonlinear},
\begin{equation*}
    \widehat{Q}^0 - \widehat{Q}^\pi \le C (u_\omega + \xi_n) \widehat{Q}^\pi + C E^\pi_{\mathrm{out}} .
\end{equation*}
The preceding bounds and \eqref{eq:denominator-cost} exclude this ordering when their fixed error multiples are chosen sufficiently small and $n$ is sufficiently large.
More explicitly, the argument yields
\begin{equation}
\label{eq:ordering-bound}
    \P(T^\pi \ge T^0) \le \eta + C \bigg\{\frac{1}{nS} + \frac{e_n^2}{S^2} + \theta_n + \frac{\xi_n v_n}{e_n}\bigg\} ,
\end{equation}
which may be obtained for all large $n$ by bounding the indicator that the deterministic costs in \eqref{eq:denominator-cost} exceed a constant times those costs.
At $h \asymp n^{-2/(4s+d)}$,
\begin{gather*}
    \theta_n \longrightarrow 0 \quad \text{if } s > d/2, \\
    \frac{\xi_n v_n}{e_n} \le C_{B, \eta} \bigg\{\sqrt{\frac{\log n}{n h^{2d}}} + \frac{\log n}{n h^{3d/2}} + n^{-2} h^{-d/2}\bigg\} \longrightarrow 0 \quad \text{if } s > 3d/4 .
\end{gather*}
Choose $\eta$ small relative to $\delta k_B / B$, then $A$ large, and finally $n$ large, so that the right-hand side of \eqref{eq:ordering-bound} is at most $\delta k_B / B$.
\Cref{lem:pairwise-order} then proves the claimed type II guarantee.
The population bound \eqref{eq:signal-core} ensures $S \ge A e_n$ when $\omega \ge C n^{-2s/(4s+d)}$ with $C$ large enough.
The permutation-level theorem applies directly to this fixed-noise channel, thereby proving the exact level.

For the Monte Carlo version, the same good events give a positive gap between the observed and the permuted scores of a fixed multiple of $S$.
The importance domination and the bounded summands in \cref{prop:mc} hold for any noise density, including one that vanishes off its support, and at points sampled from the proposal, the denominator is positive almost surely.
Taking the numerical error of every row smaller than a fixed multiple of $e_n$ preserves the required ordering.
The uniform Hoeffding bound \eqref{eq:mc-bound} with $J \gtrsim e_n^{-2} \log(2(B+1)/(\kappa \delta))$ gives the additional error budget, after increasing constants and splitting $\delta$ between the errors of the exact statistic and the numerical errors.
Since $e_n^{-2} = n^{8s/(4s+d)}$ at the prescribed bandwidth, this proves \eqref{eq:mc-budget} and completes the proof of \cref{thm:minimax}.
\qed

%% file: sections/04d_noise_family.tex
\subsection{An Explicit, Directly Samplable Noise Family}
\label{app:noise-family}

The noise in \cref{thm:minimax} need not be left as an existence construction.
This subsection provides radial densities that satisfy its hypotheses, with exact samplers constructed from normal, uniform, and beta random variables.
The broad component has a smooth edge, so the $C_c^1$ assumption is retained.

Write $v_d = \pi^{d/2} / \Gamma(d/2 + 1)$ for the volume of the unit ball.
For the narrow component, take
\begin{equation}
\label{eq:narrow-kernel}
    k(x) = \frac{\Gamma(d/2 + 3)}{2 \pi^{d/2}} \big(1 - \abs{x}^2\big)_+^2 .
\end{equation}
This density is nonnegative, belongs to $C_c^1(\R^d)$ and has support $B_1$.
An exact draw is $U = \sqrt{T} \, \Theta$, where $T \sim \Beta(d/2, 3)$ and $\Theta$ is an independent uniform direction on the unit sphere; a direction is obtained by normalising a standard Gaussian vector.
Radial integration verifies that \eqref{eq:narrow-kernel} integrates to one.

For the broad component, set $A = 2L + 1$ and choose a shell width $w > 0$; the convenient default $w = A/d$ keeps the shell modest relative to the plateau.
Define
\begin{gather}
\label{eq:cd}
    C_d(A, w) = \sum_{j=0}^{d} \binom{d}{j} A^{d-j} w^j \frac{6}{(j+2)(j+3)}, \qquad c_K = \{v_d \, C_d(A, w)\}^{-1}, \\
\label{eq:broad-kernel}
    K(x) = c_K \times
    \begin{cases}
        1, & \abs{x} \le A, \\
        (1 - u)^2 (1 + 2u), \quad u = (\abs{x} - A)/w, & A < \abs{x} < A + w, \\
        0, & \abs{x} \ge A + w .
    \end{cases}
\end{gather}
The taper and its first derivative match the constant plateau at $u = 0$ and the zero exterior at $u = 1$.
Thus $K \in C_c^1(\R^d)$, and $K$ is equal to $c_K > 0$ on the required ball $B_{2L+1}$.

\begin{proposition}[Normalisation and exact broad sampler]
\label{prop:sampler}
Let the integer $I_0 \in \{0, \dots, d\}$ have probabilities
\begin{equation*}
    \P(I_0 = j) = \frac{\binom{d}{j} A^{d-j} w^j \, 6 / \{(j+2)(j+3)\}}{C_d(A, w)} .
\end{equation*}
Conditional on $I_0 = j$, draw $U_0 \sim \Beta(j + 2, 2)$.
Independently, draw $S_0$ uniformly on $(0, 1)$ and a uniform direction $\Theta_0$ on the unit sphere.
Then $V = (A + w U_0) S_0^{1/d} \Theta_0$ has density $K$ in \eqref{eq:broad-kernel}.
In particular, $K$ is a normalised probability density.
\end{proposition}

\begin{proof}
Let $U_{\mathrm{base}} \sim \Beta(2, 2)$, with density $6u(1-u)$ on $(0, 1)$.
The binomial expansion gives
\begin{equation*}
    \E (A + w U_{\mathrm{base}})^d = C_d(A, w), \qquad \E U_{\mathrm{base}}^j = \frac{6}{(j+2)(j+3)} .
\end{equation*}
The stated mixture of beta laws therefore gives $U_0$ the density $(A + wu)^d \, 6u(1-u) / C_d(A, w)$.
Conditional on $U_0 = u$, the vector $V$ is uniform in the ball of radius $A + wu$.
Its density at a point of radius $r$ is consequently
\begin{equation*}
    \frac{1}{v_d \, C_d(A, w)} \int_0^1 6u(1-u) \, \ind\{r \le A + wu\} \, \dd u .
\end{equation*}
This equals $c_K$ for $r \le A$, zero for $r \ge A + w$, and $c_K (1 - 3z^2 + 2z^3)$ at $z = (r - A)/w \in (0, 1)$.
The identity $1 - 3z^2 + 2z^3 = (1-z)^2 (1 + 2z)$ completes the proof.
\end{proof}

For the noise variable in the interpolant \eqref{eq:interpolant}, draw independently
\begin{equation}
\label{eq:zh-sampler}
    z_h =
    \begin{cases}
        U, & \text{with probability } 1 - \lambda, \\
        V/h, & \text{with probability } \lambda .
    \end{cases}
\end{equation}
Its density is exactly $(1 - \lambda) k(z) + \lambda h^d K(hz)$, and its additive noise at an $(h, \sigma)$-channel has density $(1 - \lambda) k_h + \lambda K$.
No rejection sampler or numerical normalisation is required.
In code, the normalising sum and the mixture probabilities should be evaluated with log-sum-exp arithmetic when $d$ is large.
Compactly supported components have log density equal to $-\infty$ outside their support, which the importance sampler must handle explicitly.

This law should not be rescaled to unit covariance, since that would change the effective narrow and broad scales required by the theorem.
The law is directly implementable, but it is not ordinary Gaussian or Student-$t$ noise, and its dependence on $h$ and on the support bound $L$ remains explicit.
For a fixed channel, the law of $z_h$ is held fixed throughout its interpolation path, although different bandwidth channels use different laws.

For completeness, the narrow density satisfies
\begin{equation*}
    \norm{\nabla k}_1 = \frac{d (d+2) (d+4)}{(d+1)(d+3)} .
\end{equation*}
Its symmetry and support imply $\widehat{k}(\xi) = \E \cos(\xi \cdot U) \ge 1 - \abs{\xi}^2/2$.
Thus, the choices $r_0 = 1$, $c_k = 1/2$ and $c_* = 1/4$ are valid in \eqref{eq:sigma-k}, which supplies a fully explicit sufficient amplitude condition as well as an explicit noise sampler.
These constants suffice for the proof and are not asserted to be practically optimal.

%% file: sections/04e_adaptive_proof.tex
\subsection{Rank Fusion and the Proof of the Adaptive Theorem}
\label{app:adaptive}

This subsection first proves the remaining finite-sample and deterministic properties of rank fusion (\cref{app:fusion-proofs}).
It then establishes a conditional margin for a single compact channel at a shrinking target (\cref{app:shrinking-target} and the subsections before it), and finally combines the two parts (\cref{app:grid,app:adaptive-completion}).
Throughout, the statistic matrix $(T_{bc})$ has rows $b = 0, \dots, B$, of which row $0$ is the observed labelling, and the inclusive ranks $r_{bc}$ are those of \eqref{eq:ranks}.

\subsubsection{Rank Fusion: Power Transfer and Permutation Budgets}
\label{app:fusion-proofs}

The next proposition turns a conditional margin for one channel into power for any aggregate that rejects when the channel's rank is sufficiently small.

\begin{proposition}[Finite-permutation and numerical power transfer]
\label{prop:power-transfer}
Fix a channel $c$ and a target $u \in (0, 1]$ such that the aggregate rejects whenever $r_{0c} \le u$.
Suppose that, with probability at least $1 - \delta/2$ over the observations, there is $S > 0$ such that
\begin{equation}
\label{eq:channel-margin}
    \P_\pi\big\{T_{\pi c} \ge T_{0c} - c_{\mathrm{gap}} S \,\big|\, \text{observations}\big\} \le u/4,
\end{equation}
where $c_{\mathrm{gap}} > 0$ is fixed and $T_{\pi c}$ is the exact statistic of channel $c$ on a uniformly permuted row.
Suppose that the numerical scores, if used, satisfy $\abs{\Tmc_{bc} - T_{bc}} \le c_{\mathrm{num}} e_h$ simultaneously for all rows of that channel, with failure probability at most $\zeta \le \delta/4$, and that $2 c_{\mathrm{num}} e_h < c_{\mathrm{gap}} S$ on the favourable observation event.
Then the aggregate has type II error at most $\delta$ provided that
\begin{equation}
\label{eq:binomial-check}
    k_u = \lfloor u (B+1) \rfloor \ge 1, \qquad \P\{\Bin(B, u/4) \ge k_u\} \le \delta/4 .
\end{equation}
For exact integration, the numerical conditions are unnecessary.
\end{proposition}

\begin{proof}
On the simultaneous accuracy event, $\Tmc_{bc} \ge \Tmc_{0c}$ implies $T_{bc} \ge T_{0c} - 2 c_{\mathrm{num}} e_h > T_{0c} - c_{\mathrm{gap}} S$.
Conditionally on observations satisfying \eqref{eq:channel-margin}, the number of independent permutation rows that satisfy the exact margin event is stochastically dominated by $\Bin(B, u/4)$.
Nonrejection of the aggregate implies $r_{0c} > u$, and the inclusive rank then implies at least $\lfloor u (B+1) \rfloor$ permutation exceedances.
Summing the three failure probabilities $\delta/2$, $\delta/4$ and $\zeta \le \delta/4$ proves the claim.
The binomial argument applies to the exact margin events, and no independence of the shared numerical errors is assumed.
\end{proof}

By \cref{lem:rank-moments}, rank fusion rejects whenever $r_{0c} \le u_c(B)$, so \cref{prop:power-transfer} applies with $u = u_c(B)$, and \eqref{eq:binomial-check} is then exactly \eqref{eq:fusion-budget}.
When a single budget must cover all channels, \eqref{eq:binomial-check} should be checked for every actual target: the floor creates jumps, so a check at the smallest target is not in general equivalent to checks at every target.

\begin{lemma}[Explicit sufficient permutation budgets]
\label{lem:budget}
Let $c_{\mathrm{B}} = \frac{1}{2}\log 2 - \frac{1}{4} > 0$, $K_\delta = \max\{2, c_{\mathrm{B}}^{-1} \log(4/\delta)\}$ and $w_{\min} = \min_c w_c$.
A sufficient condition for \eqref{eq:binomial-check} at every channel is $(B+1) \min_c u_c \ge K_\delta$.
For rank fusion with the targets $u_c = u_c(B)$, it holds in each of the following cases.
\begin{enumerate}
    \item[(i)] If $0 < \tau < 1$ and $B + 1 \ge K_\delta [\kappa (1 - \tau) w_{\min}]^{-1/\tau}$; for uniform weights, $B = O(C^{1/\tau})$.
    \item[(ii)] If $\tau = 1$ and $(B+1)/(1 + \log(B+1)) \ge K_\delta / (\kappa w_{\min})$; thus $B = O(w_{\min}^{-1} \log(1/w_{\min}))$, or $O(C \log C)$ for uniform weights.
    A fully explicit, loose choice is $B + 1 = \lceil 4 A_\delta \log(e A_\delta) \rceil$ with $A_\delta = K_\delta / (\kappa w_{\min}) \ge 1$.
    \item[(iii)] If $\tau > 1$ and $B + 1 \ge \frac{\tau}{\tau - 1} K_\delta^\tau / (\kappa w_{\min})$; for uniform weights, $B = O(C)$. For every $\tau>0$, it suffices that $(B+1)\,\kappa\,w_{\min}^{1/\tau}\ge C\,K_\delta$,
  because $u_c(B)\ge\kappa w_c^{1/\tau}/C$; for uniform weights, $B=O(C^{1+1/\tau})$. This
  condition is monotone in $B$, so every larger budget also satisfies \eqref{eq:binomial-check}.
\end{enumerate}
\end{lemma}

\begin{proof}
Write $m = B + 1$.
If $m u \ge 2$, then $\lfloor m u \rfloor \ge m u / 2 \ge 1$.
For $X \sim \Bin(m - 1, u/4)$, the exponential Markov inequality with parameter $\log 2$ gives
\begin{equation*}
    \P(X \ge m u / 2) \le \exp\Big(-\frac{m u \log 2}{2} + \frac{(m-1) u}{4}\Big) \le e^{-c_{\mathrm{B}} m u},
\end{equation*}
which is at most $\delta/4$ when $m u \ge c_{\mathrm{B}}^{-1} \log(4/\delta)$.
For (i), \cref{lem:rank-moments} gives $u_c(B) \ge [\kappa (1 - \tau) w_c]^{1/\tau}$.
For (ii), $\Arank_m(1) = \sum_{j=1}^m j^{-1} \le 1 + \log m$, so $m u_c(B) \ge \kappa w_c m / (1 + \log m)$; the explicit choice satisfies $m / (1 + \log m) \ge A_\delta$ because $m \mapsto m/(1 + \log m)$ is increasing and $4(1 + \log A_\delta) \ge 1 + \log 4 + \log A_\delta + \log(1 + \log A_\delta)$ for $A_\delta \ge 1$.
For (iii), $\Arank_m(\tau) \le m^{\tau - 1} \sum_{j \ge 1} j^{-\tau} \le \frac{\tau}{\tau - 1} m^{\tau - 1}$, so $m u_c(B) \ge m^{1/\tau} [\kappa w_c (\tau - 1)/\tau]^{1/\tau}$, and the stated condition gives $m u_c(B) \ge K_\delta$.
\end{proof}

These orders hold with $\kappa$, $\delta$ and $\tau$ fixed.
Exact binomial checks of \eqref{eq:binomial-check} can yield smaller budgets, and constants may grow substantially with the temperature, so the order $O(C)$ alone does not recommend a large $\tau$.
When $B$ changes, $\Arank_{B+1}(\tau)$ and the targets must be recomputed.

\subsubsection{A Conditional Exponential Bound with the Empirical Weight}

For one compact channel, let $\muhat_N = N^{-1} \sum_{i=1}^N \delta_{V_i}$ be the pooled empirical measure, $\mu = (\rho_0 + \rho_1)/2$, and
\begin{equation*}
    q_N = \big(\muhat_N^{[\alpha]} * \muhat_N^{[\beta]}\big) * W_h, \qquad q_* = \big(\mu^{[\alpha]} * \mu^{[\beta]}\big) * W_h .
\end{equation*}
As before, $\Omega = B_{L+1}$, and on this core $q_N, q_* \ge m_0 = \lambda c_K$ for every sample.
Define
\begin{gather*}
    \psi_{N, y}(z) = \big(\muhat_N^{[\beta]} * W_h\big)(z - \alpha y) - \big(\muhat_N^{[\alpha]} * W_h\big)(z - \beta y), \\
    \langle u, u' \rangle_N = \frac{1}{8} \int_\Omega \frac{u u'}{q_N}, \qquad \Gram_N(i, j) = \big\langle \psi_{N, V_i}, \psi_{N, V_j} \big\rangle_N, \\
    Q^b_N = \frac{1}{8} \int_\Omega \frac{(\widehat{D}^b)^2}{q_N}, \qquad \bar{Q}_N = \E_\pi\big[Q^\pi_N \mid V_1, \dots, V_N\big] .
\end{gather*}
The quadratic form is only a device for the proof; it is not the computed test statistic.
Set $e_h = (n h^{d/2})^{-1}$ and $v_h = (n h^d)^{-1}$.

\begin{lemma}[Samplewise permutation tail]
\label{lem:perm-tail}
For every pooled sample in $B_L$ and every $x \ge 1$,
\begin{equation}
\label{eq:perm-tail}
    \P_\pi\Big(Q^\pi_N > \bar{Q}_N + C\big\{e_h \sqrt{x} + x/n\big\} \,\Big|\, V_1, \dots, V_N\Big) \le 2 e^{-x}, \qquad 0 \le \bar{Q}_N \le C v_h .
\end{equation}
The constants depend only on the fixed class and the noise design.
\end{lemma}

\begin{proof}
The two conditional density channels average to $q_N$ under $\muhat_N$.
The conditional-expectation contraction proof of \cref{lem:contraction} therefore applies to functions supported on $\Omega$, with the factor $1/8$ in the inner product.
If $C_N = N^{-1} \sum_i \psi_{N, V_i} \otimes \psi_{N, V_i}$, then $\norm{C_N}_{\mathrm{op}} \le C$.
Each conditional kernel is bounded on the fixed core by $C h^{-d}$, and the same trace argument gives $\tr C_N \le C h^{-d}$.
The nonzero eigenvalues of $\Gram_N / N$ equal those of $C_N$, so
\begin{equation}
\label{eq:gram-norms}
    \norm{\Gram_N}_{\mathrm{op}} \le C N, \qquad \norm{\Gram_N}_{\mathrm{HS}}^2 = N^2 \tr(C_N^2) \le C N^2 h^{-d} .
\end{equation}
Also $\sum_i \psi_{N, V_i} = 0$, hence $\Gram_N \ind = 0$.
For uniform balanced signs $\varepsilon \in \{-1, 1\}^{2n}$, the pooled-feature identity gives $Q^\pi_N = n^{-2} \varepsilon^\top \Gram_N \varepsilon$.
Since $\E_\pi \varepsilon_i \varepsilon_j = -1/(N-1)$ for $i \ne j$,
\begin{equation*}
    \bar{Q}_N = \frac{1}{n^2} \cdot \frac{N}{N-1} \tr \Gram_N \le C v_h .
\end{equation*}
For a deterministic symmetric matrix $\mathbf{A}$ with zero diagonal, the quadratic specialisation of the multislice concentration theorem \citep[Theorem~1.2]{sambale_concentration_2022} states
\begin{equation*}
    \P_\pi\big(\abs{\varepsilon^\top \mathbf{A} \varepsilon - \E_\pi \varepsilon^\top \mathbf{A} \varepsilon} \ge u\big) \le 2 \exp\bigg[-c \min\bigg\{\frac{u^2}{\norm{\mathbf{A}}_{\mathrm{HS}}^2}, \frac{u}{\norm{\mathbf{A}}_{\mathrm{op}}}\bigg\}\bigg] .
\end{equation*}
Its assumptions hold here: the signs have exactly $n$ entries of each value, each sign has mean zero, and the matrix is fixed after conditioning on the pool.
Remove the diagonal of $\Gram_N$, whose quadratic contribution is constant.
Positive semidefiniteness gives $\norm{\Gram_N - \diag \Gram_N}_{\mathrm{op}} \le 2 \norm{\Gram_N}_{\mathrm{op}}$, while removing the diagonal decreases the Hilbert--Schmidt norm.
Apply the displayed inequality and \eqref{eq:gram-norms}, and divide by $n^2$.
\end{proof}

\subsubsection{Concentration of the Pooled Weight and the Actual Densities}

\begin{lemma}[Uniform core bounds at a specified confidence]
\label{lem:core-confidence}
For $h$ of polynomial order in $n$ and $0 < \eta < 1/2$, put
\begin{equation}
\label{eq:xi-h}
    \xi_h(\eta) = C \bigg\{\sqrt{\frac{\log(n/\eta)}{n h^d}} + \frac{\log(n/\eta)}{n h^d} + n^{-2}\bigg\} .
\end{equation}
With probability at least $1 - \eta$, $\norm{q_N - q_*}_{\infty, \Omega} \le \xi_h(\eta)$.
The observed pair average satisfies $\norm{\widehat{m}^0 - m_*}_{\infty, \Omega} \le \xi_h(\eta)$ with the same confidence.
For one independent balanced permutation,
\begin{equation*}
    \norm{\widehat{m}^\pi - q_*}_{\infty, \Omega} + \norm{\widehat{D}^\pi}_{\infty, \Omega} \le C \xi_h(\eta)
\end{equation*}
with unconditional probability at least $1 - \eta$.
A change of the constant absorbs any fixed number of these events.
\end{lemma}

\begin{proof}
The broad part is identically $\lambda c_K$ on $\Omega$, for both the empirical and the population mixtures.
Suppress the factor $1 - \lambda$ in the narrow part, whose exact difference is
\begin{equation*}
    (\muhat_N - \mu)^{[\alpha]} * \muhat_N^{[\beta]} * k_h + \mu^{[\alpha]} * (\muhat_N - \mu)^{[\beta]} * k_h .
\end{equation*}
The first term is bounded in sup norm by $\norm{(\muhat_N - \mu)^{[\alpha]} * k_h}_\infty$ on a fixed enlarged ball.
This is an ordinary kernel-density empirical average of $2n$ independent observations, half from $\rho_0$ and half from $\rho_1$, centred at its mean.
Its envelope is $C h^{-d}$, and its pointwise variance proxy is $C h^{-d}$, because each endpoint density is bounded by $M$ and $\alpha \ge 1/2$.
For the second term, the empirical functions are $(\mu^{[\alpha]} * k_h)(x - \beta V_i)$, with uniform bound $2^d M$ and a constant variance proxy.
Bernstein's inequality applies without identical distributions.
The spatial Lipschitz constants grow at most polynomially in $1/h$, and a net of mesh $n^{-2} h^{d+1}$ together with a union bound gives \eqref{eq:xi-h}.
Compact support limits the first supremum to a fixed enlarged ball.
This calculation includes all empirical diagonal pairs; it never treats them as independent.

For the two-sample pair mixtures, condition on one group as in \cref{lem:core}, and use the group with coefficient $\alpha$ for the variance bound $C h^{-d}$.
The reflected density is handled by interchanging the groups.
The conditional means concentrate over the other, independent group, and the same spatial net applies.
After a permutation, condition on the assignment of the original indices.
If $J_P$ original $\rho_0$-observations enter its first group, the population contrast is $a_P D_*$ and its pair average is $q_* - a_P^2 (f^{[\alpha]} * f^{[\beta]}) * W_h / 4$, with $a_P = 2 J_P / n - 1$.
For this hypergeometric count, $\P(\abs{a_P} > z) \le 2 e^{-n z^2/2}$; one proof exposes the $n$ draws without replacement, uses the exponential-moment bound by independent Bernoulli draws with success probability $1/2$, and then applies the usual Chernoff argument.
Thus $\abs{a_P} \le C \sqrt{\log(1/\eta)/n}$ at the required confidence.
The population contrast and the average corrections are uniformly bounded on the core, so these terms are absorbed into \eqref{eq:xi-h}.
\end{proof}

\subsubsection{Observed Signal above the Permutation Baseline}

We use the deterministic-weight forms $\widehat{Q}^b$, $Q^b_A$ and $\bar{Q}_A$ of \cref{app:minimax}, and let
\begin{equation*}
    S = S_h = \frac{1}{8} \int_\Omega \frac{D_*^2}{q_*}, \qquad \omega = \norm{\rho_0 - \rho_1}_2, \qquad \theta_h = \frac{v_h}{\sqrt{e_h}} = (n h^{3d/2})^{-1/2} .
\end{equation*}
We will choose a channel for which
\begin{equation}
\label{eq:good-channel}
    S \ge A e_h, \qquad \omega^2 \le C S, \qquad v_h \le 1, \qquad \theta_h \to 0, \qquad \xi_h \to 0, \qquad \xi_h v_h / S \to 0 .
\end{equation}
The second inequality follows from the population signal \eqref{eq:signal-core} whenever the separation dominates the bias $h^{2s}$ by a sufficiently large constant.

The numerator estimates extend to this channel without imposing the exact bandwidth balance of \cref{app:minimax}.
In fact,
\begin{equation}
\label{eq:numerator-errors-adaptive}
    \E \big|\widehat{Q}^0 - Q^0_A\big| \le C S \theta_h, \qquad \E \big|\widehat{Q}^\pi - Q^\pi_A\big| \le C e_h \sqrt{v_h} .
\end{equation}
To check the first bound directly, the Hilbert second-moment estimates of \cref{lem:centring} and the Cauchy--Schwarz inequality give
\begin{equation*}
    \E \big|\widehat{Q}^0 - Q^0_A\big| \le C \sqrt{S + v_h} \sqrt{\omega^2 v_h + e_h^2} + C \big(\omega^2 v_h + e_h^2\big) .
\end{equation*}
Dividing by $S$ and using $S \ge e_h$ and $\omega^2 \le C S$ bounds the terms by $C\{\sqrt{v_h} + \sqrt{e_h} + v_h/\sqrt{e_h} + v_h + e_h\} \le C \theta_h$.
The permuted estimate uses exactly the earlier conditioning argument for independent observations; it does not require $h^{2s} \asymp e_h$.

On the pooled-weight event of \cref{lem:core-confidence}, $\abs{Q^b_N - \widehat{Q}^b} \le C \xi_h \widehat{Q}^b$ for every labelling $b$.
Write $R^{\mathrm{err}}_0 = \abs{\widehat{Q}^0 - Q^0_A}$ and $R^{\mathrm{err}}_\pi = \E_\pi[\abs{\widehat{Q}^\pi - Q^\pi_A} \mid V]$.
Taking conditional expectations and subtracting yields
\begin{equation}
\label{eq:centred-signal}
    Q^0_N - \bar{Q}_N \ge (1 - C \xi_h)\big(Q^0_A - \bar{Q}_A\big) - 2 C \xi_h \bar{Q}_A - R^{\mathrm{err}}_0 - (1 + C \xi_h) R^{\mathrm{err}}_\pi .
\end{equation}
As $\bar{Q}_A \le C v_h$, its deterministic cost is $O(\xi_h v_h)$.
The exact centring identity $Q^0_A - \bar{Q}_A = 2(n-1) U_A / (2n-1)$ and $\Var(U_A) \le C (S/n + e_h^2)$ show that this centred signal exceeds a fixed positive multiple of $S$, except with probability at most $C/(nS) + C e_h^2 / S^2$.
Markov's inequality and \eqref{eq:numerator-errors-adaptive} control the two remainder terms, at failure cost $C \theta_h$, for any fixed and sufficiently small error multiple of $S$.
Consequently, under \eqref{eq:good-channel},
\begin{equation}
\label{eq:observed-quadratic}
    Q^0_N \ge \bar{Q}_N + c S
\end{equation}
with arbitrarily high fixed confidence, for large $n$ and sufficiently large $A$.

This is also a signal bound for the actual divergence.
Indeed, $m_* / q_* \le 1 + u_\omega$ with $u_\omega = \min\{C \omega, 1\}$ by \eqref{eq:mstar-ratio}, and the two core concentration bounds and the positive floor give
\begin{equation*}
    T^0 \ge \frac{Q^0_N}{1 + u_\omega + C \xi_h} .
\end{equation*}
Furthermore,
\begin{equation*}
    \frac{u_\omega \bar{Q}_N}{S} \le C \frac{\omega v_h}{S} \le C \frac{v_h}{\sqrt{S}} \le C \theta_h .
\end{equation*}
Thus, by reducing the fixed signal constant if necessary,
\begin{equation}
\label{eq:observed-js}
    T^0 \ge \bar{Q}_N + c_{\mathrm{sig}} S
\end{equation}
with arbitrarily high fixed confidence.
This step retains a fixed fraction of the signal even for nonlocal alternatives, since only $u_\omega \le 1$ is needed, and it does not assume that the observed empirical contrast is uniformly small.

\subsubsection{Permutation Comparison at a Shrinking Channel Target}
\label{app:shrinking-target}

Let $u \in (0, 1]$ be the target of this channel, that is, a level such that the aggregate rejects whenever the inclusive rank of the observed row in this channel is at most $u$, and put $x = \log(16/u)$.
Choose $\eta = u \delta / C_0$ in the core bounds, where $C_0$ is a sufficiently large fixed constant.
For a good permuted row, the entropy expansion used in \cref{lem:nonlinear}, now with $q_N$ as the weight of the comparison, gives
\begin{equation*}
    T^\pi \le (1 + C \xi_h) Q^\pi_N + E^\pi_{\mathrm{out}} .
\end{equation*}
The broad-only exterior still has an expectation of at most $C/n$.
For any fixed small $c_{\mathrm{out}} > 0$, the unconditional probability that the core event fails or that $E^\pi_{\mathrm{out}} > c_{\mathrm{out}} S$ is at most $\eta + C/(nS)$.
Let $b(V)$ be this bad-row probability conditional on all observations.
Markov's inequality gives
\begin{equation}
\label{eq:bad-rows}
    \P\{b(V) > u/8\} \le \frac{8}{u} \Big(\eta + \frac{C}{nS}\Big) .
\end{equation}
When $u n S \to \infty$, this is an arbitrarily small fixed error budget after choosing $C_0$ and taking $n$ large.
Outside a further conditional event of probability $u/8$, the samplewise tail bound \eqref{eq:perm-tail} gives $Q^\pi_N \le \bar{Q}_N + C(e_h \sqrt{x} + x/n)$.
It follows that good rows satisfy
\begin{equation*}
    T^\pi \le \bar{Q}_N + C\big\{e_h \sqrt{x} + x/n + \xi_h v_h\big\} + c_{\mathrm{out}} S .
\end{equation*}
Combine this with \eqref{eq:observed-js}.
If
\begin{equation}
\label{eq:channel-conditions}
    S \ge A \big(e_h \sqrt{x} + x/n\big), \qquad \xi_h v_h = o(S), \qquad u n S \longrightarrow \infty,
\end{equation}
then $A$ can be chosen large and $c_{\mathrm{out}}$ small so that, with probability at least $1 - \delta/2$ over the observations,
\begin{equation}
\label{eq:conditional-ordering}
    \P_\pi\big\{T^\pi \ge T^0 - c_{\mathrm{gap}} S \,\big|\, V\big\} \le u/4
\end{equation}
for a fixed $c_{\mathrm{gap}} > 0$.
The allocation $\delta/2$ includes the failures of the observed signal and \eqref{eq:bad-rows}; the constants can be tightened to any specified smaller fixed allocation.
This conditional statement is the key improvement over bounding the unconditional ordering probability and then multiplying by $1/u$: only the bad permuted rows need a failure probability of order $u$.

\subsubsection{The Dyadic Grid Contains a Suitable Channel}
\label{app:grid}

Let $u_n=u_c(B_n)$ denote the channel target \eqref{eq:fusion-target} of rank fusion with
uniform weights $w_c=1/C_n$. Since $\mathcal A_m(\tau)\ge1$ for every $m$,
\[
  \kappa\,C_n^{-1-1/\tau}\ \le\ u_n\ \le\
  \max\bigl\{(\kappa/C_n)^{1/\tau},\ \kappa\,C_n^{-1-1/\tau}\bigr\}
\]
for every permutation budget $B_n$. Hence, for fixed $\kappa$ and $\tau$,
$x_n=\log(16/u_n)\asymp\log C_n\asymp\log\log n$, whatever the value of $B_n$.
For a fixed $s > 3d/4$, define, for the purposes of the proof only,
\begin{equation*}
    h_* = \bigg(\frac{\sqrt{x_n}}{n}\bigg)^{2/(4s+d)} .
\end{equation*}
The algorithm never uses $h_*$.
It tends to zero and eventually lies strictly between the largest and the smallest grid scales.
Indeed, $2/(4s+d) < 1/(2d)$, whereas the smallest grid scale is of order $n^{-1/(2d)}$.
Thus a grid element satisfies $h_*/2 < h \le h_*$, and at this channel
\begin{equation}
\label{eq:grid-balance}
    h^{2s} \asymp e_h \sqrt{x_n}, \qquad \big(r^{\mathrm{ad}}_{n,s}\big)^2 \asymp e_h \sqrt{x_n} .
\end{equation}
The population signal bound then gives $S \ge A e_h \sqrt{x_n}$ and $\omega^2 \le C S$ if $\omega \ge A_s r^{\mathrm{ad}}_{n,s}$ with $A_s$ sufficiently large.
All error conditions in \eqref{eq:good-channel} and \eqref{eq:channel-conditions}, with $u = u_n$, can now be checked explicitly:
\begin{gather*}
    \frac{x_n / n}{e_h \sqrt{x_n}} = h^{d/2} \sqrt{x_n} \longrightarrow 0, \qquad
    \theta_h = (n h^{3d/2})^{-1/2} \longrightarrow 0, \\
    \frac{\xi_h v_h}{S} \le C \bigg\{\sqrt{\frac{\log n}{n h^{2d} x_n}} + \frac{\log n}{n h^{3d/2} \sqrt{x_n}} + \frac{n^{-2} h^{-d/2}}{\sqrt{x_n}}\bigg\} \longrightarrow 0, \\
    u_n n S \ge c \, u_n h^{-d/2} \sqrt{x_n} \longrightarrow \infty .
\end{gather*}
The last limit holds because $u_n$ is bounded below by an inverse fixed power of $\log n$, up to logarithmic factors, whereas $h^{-d/2}$ grows as a positive power of $n$ for each fixed finite $s$.
The first limit also holds for every fixed finite $s$.
The numerator limit holds for $s > d/2$, and the strongest denominator condition is $s > 3d/4$.
Here $\log(n/\eta) = O(\log n)$, because $\eta$ is of order $u_n$.
Thus \eqref{eq:conditional-ordering} holds with $u = u_n$, at a channel selected only in the proof.
Adding a fixed number of noise designs changes $C_n$ by a constant factor while preserving these conclusions.

\subsubsection{Completion of the Proofs}
\label{app:adaptive-completion}

\begin{proof}[Proof of \cref{thm:adaptive}]
By \cref{app:shrinking-target,app:grid}, the conditional margin \eqref{eq:conditional-ordering} holds at the favourable channel, with $u = u_c(B_n)$ and with probability at least $1 - \delta/2$ over the observations; this is \eqref{eq:channel-margin}.
By \cref{lem:rank-moments}, rank fusion rejects whenever $r_{0c} \le u_c(B_n)$, so \cref{prop:power-transfer}, with $B_n$ satisfying \eqref{eq:fusion-budget}, proves \eqref{eq:adaptive} with exact integration.

For numerical integration, choose a fixed $c_{\mathrm{num}} > 0$ and $0 < \zeta \le \delta/4$, and let $e_{h_c} = 1/(n h_c^{d/2})$ for the bandwidth $h_c$ of channel $c$.
The balanced-sample summands of the pooled-proposal estimate lie in $[0, 2 \log 2]$, so \cref{prop:mc} gives the explicit sufficient budgets
\begin{equation}
\label{eq:mc-budget-adaptive}
    J_c \ge \frac{(2 \log 2)^2}{2 c_{\mathrm{num}}^2 e_{h_c}^2} \log \frac{2 C_n (B_n + 1)}{\zeta} .
\end{equation}
Hoeffding's inequality and a union bound then give simultaneous errors at most $c_{\mathrm{num}} e_{h_c}$ for all rows and channels, with probability at least $1 - \zeta$.
At the favourable channel, $S / e_h \gtrsim \sqrt{x_n} \to \infty$, so $2 c_{\mathrm{num}} e_h < c_{\mathrm{gap}} S$ eventually, and \cref{prop:power-transfer} completes the numerical result.
The level follows separately from \cref{thm:fusion-level}; it needs neither smoothness nor accurate numerical integration.
Constants and the starting sample size may depend on $s$ and $\tau$; we do not claim uniformity as $s \downarrow 3d/4$, as $s \to \infty$, or for a temperature that varies with $n$.
\end{proof}


The integration budget \eqref{eq:mc-budget-adaptive} is deliberately conservative, and its largest values can be expensive.
A numerical implementation must not silently cap this budget while invoking the theorem's power guarantee.
Smaller budgets retain permutation level, but they do not automatically retain the power theorem, and their finite-sample power must be assessed separately.

%% file: sections/04f_gaussian.tex
\subsection{Gaussian Noise: Conservative Rate}
\label{app:gaussian}

This section concerns the unchanged empirical test with a single standard-Gaussian noise term.
It provides a rigorous but conservative power guarantee and records the partial results that a sharp Gaussian rate would build on.

\subsubsection{A Conservative Power Guarantee}

Fix one channel at time $t$ with $\alpha + \beta \le 1$ and $z \sim \mathcal{N}(0, \Id_d)$, assume $n_0, n_1 \ge 2$, and write $S = J_t(\rho_0, \rho_1) = \JS(\rho_t, \rho_{1-t})$ and $T^0 = \JS(\rhohat_t, \rhohat_{1-t})$.
Let $T^\pi$ be the same statistic after a uniformly random reassignment, as in \cref{lem:permuted-first-moment}, and recall $V_h$ from \eqref{eq:vh-gauss} and $\nu_n = (n_0 + n_1 - 1)/(n_0 n_1)$.
The argument uses the original Euclidean bridge, all empirical endpoint pairs, and no lower bound on the densities.

\begin{proposition}[Ordering probabilities]
\label{prop:ordering}
For $S > 0$,
\begin{equation}
\label{eq:ordering-exact}
    \P(T^\pi \ge T^0) \le \min\bigg\{1, \ \frac{304 \, \nu_n V_h}{S}\bigg\} .
\end{equation}
For the shared importance-sampling implementation with $J$ integration points and summand bound $L_w = \log 2 / (2 w_N)$, $w_N = n_0 n_1 / N^2$, the corresponding scores satisfy
\begin{equation}
\label{eq:ordering-mc}
    \P(\Tmc^\pi \ge \Tmc^0) \le \min\bigg\{1, \ \frac{592 \, \nu_n V_h}{S} + \exp\Big(-\frac{J S}{128 L_w}\Big)\bigg\} .
\end{equation}
\end{proposition}

\begin{proof}
Let $\mathsf{E} = H(\rhohat_t, \rho_t) + H(\rhohat_{1-t}, \rho_{1-t})$.
\Cref{lem:hellinger-pairs} implies $\E \mathsf{E}^2 \le 4 \nu_n V_h$.
Since $H(\rho_t, \rho_{1-t}) \ge \sqrt{S}$ by \eqref{eq:divergence-chain}, on the event $\mathsf{E} \le \sqrt{S}/2$ the empirical Hellinger distance is at least $\sqrt{S}/2$, and hence $T^0 \ge S/16$.
Markov's inequality gives $\P(T^0 < S/16) \le 16 \nu_n V_h / S$.
Combining this with \cref{lem:permuted-first-moment} and $\P(T^\pi \ge S/16) \le 288 \nu_n V_h / S$ proves \eqref{eq:ordering-exact}.

For independent variables $Z_\ell \in [0, L_w]$ with mean $m_Z$, the multiplicative Chernoff bound gives $\P(J^{-1} \sum_\ell Z_\ell < m_Z / 2) \le \exp(-J m_Z / (8 L_w))$.
Indeed, convexity gives $\E e^{-\lambda Z / L_w} \le 1 + (e^{-\lambda} - 1) m_Z / L_w$; take $\lambda = \log 2$ and apply Markov's inequality.
Conditionally on the pooled observations, this applies to the observed Monte Carlo score, so
\begin{equation*}
    \P\big(\Tmc^0 < S/32\big) \le \frac{16 \nu_n V_h}{S} + e^{-J S / (128 L_w)} .
\end{equation*}
Conditional unbiasedness and \cref{lem:permuted-first-moment} give $\E \Tmc^\pi \le 18 \nu_n V_h$, so its upper tail at $S/32$ is bounded by $576 \nu_n V_h / S$.
No independence between permutation rows, or between their shared numerical errors, is used.
\end{proof}

\begin{theorem}[Conservative power on the unchanged Gaussian bridge]
\label{thm:gaussian-conservative}
Use $B$ uniform permutation rows and the inclusive p-value \eqref{eq:pvalue}, and set $k = \lfloor \kappa (B+1) \rfloor \ge 1$.
For the exact statistic,
\begin{equation}
\label{eq:gaussian-nonrejection}
    \P_{\rho_0, \rho_1}(\text{nonrejection}) \le \min\bigg\{1, \ \frac{304 \, B \, \nu_n V_h}{k S}\bigg\} .
\end{equation}
For the shared Monte Carlo implementation, the ratio inside the braces is replaced by
\begin{equation*}
    \frac{B}{k} \bigg\{\frac{592 \, \nu_n V_h}{S} + \exp\Big(-\frac{J S}{128 L_w}\Big)\bigg\} .
\end{equation*}
Now impose the assumptions of \cref{thm:signal}.
For $n_0 \asymp n_1 \asymp n$, known $s > 0$ and fixed positive $L, M, R, \kappa, \delta$, take $h_n \asymp n^{-1/(2s+d)}$ and let $t_n$ be an $(h_n, \sigma)$-channel with $\sigma \ge \sigma_G$.
For all sufficiently large $n$, there is a finite constant $C$ such that the exact test has type II error at most $\delta$ whenever
\begin{equation}
\label{eq:gaussian-rate}
    \norm{\rho_0 - \rho_1}_2 \ge C n^{-s/(2s+d)} .
\end{equation}
The shared Monte Carlo test has the same guarantee if, in addition, $J \ge C' n^{2s/(2s+d)}$ with a sufficiently large constant $C'$.
For arbitrary fixed $\kappa$, $\delta$ and $B$, these constants may depend on $B/k$; for balanced samples, $L_w = 2 \log 2$.
\end{theorem}

\begin{proof}
Let $C_B$ count the rows whose score is at least the observed score.
Nonrejection implies $C_B \ge k$, and Markov's inequality bounds its probability by $B/k$ times the pairwise ordering probability of \cref{prop:ordering}.
This proves the finite-sample statements, without an independence assumption on the row scores.
\Cref{cor:bandwidth} implies $S \ge e^{-1} \omega^2 / (512 M)$ whenever $R (2h)^s \le \omega/4$.
Thus, it suffices to have $\omega \ge C_1 R h^s + C_2 \sqrt{\nu_n V_h}$, for constants determined by the stated parameters.
Since $\nu_n \asymp n^{-1}$ and $V_h = O(h^{-d})$ as $h \downarrow 0$, the specified bandwidth balances these two terms and proves \eqref{eq:gaussian-rate}.
Under this separation, $S \gtrsim n^{-2s/(2s+d)}$, so choosing $J$ as stated, with a large enough constant, makes the exponential contribution at most $\delta/2$.
\end{proof}

\begin{remark}
This is a uniform power bound for the exact proposed statistic and for its implemented numerical approximation; neither the path nor the statistic has been replaced.
It holds for every $s > 0$ under bounded support and an upper density bound, but not under a lower density bound.
Its separation exponent $s/(2s+d)$ is, however, smaller than the minimax exponent $2s/(4s+d)$.
The loss occurs because the proof controls the full density-estimation error and the uncentred permutation score; it does not exploit the cancellation of the shared null fluctuation at the scale $1/(n h^{d/2})$.
Proving that cancellation for the nonlinear, label-dependent empirical divergence remains the sharp-rate problem for Gaussian noise.
The constants above were chosen for transparent validity, not for numerical sharpness.
\end{remark}

%% file: sections/05_experiments.tex
\section{Experiments}
\label{sec:experiments}

 We follow the protocol of \citet{jmlr_mmdagg} on two standard benchmarks, where $n$ denotes
the sample size per distribution: Blob \citep{liu2020learning}, with $n \in \{90, 180, 270, 360\}$, and
MNIST with $n = 200$ under five alternatives $Q_1$--$Q_5$ of increasing difficulty.
Each rejection rate is estimated from $1000$ repetitions at level $\kappa = 0.05$.
We compare with MMDAgg \citep{jmlr_mmdagg}, MMD-FUSE \citep{biggs_mmd-fuse_2023},
single-kernel MMD tests with Gaussian and Laplace kernels whose bandwidth is set by the
median heuristic, selected on a held-out half, or selected by an oracle with $n$ additional
samples, and two classifier tests, C2ST with random forests \citep{lopezpaz2017revisiting}
and AutoTST \citep{kubler2022autotst}. Full numerical results are given in
Tables~\ref{tab:blob_full} and~\ref{tab:mnist_full}, and implementation details in
Appendix~\ref{app:experiments}.

\begin{figure*}[t]

  \centering
  \includegraphics[width=\textwidth]{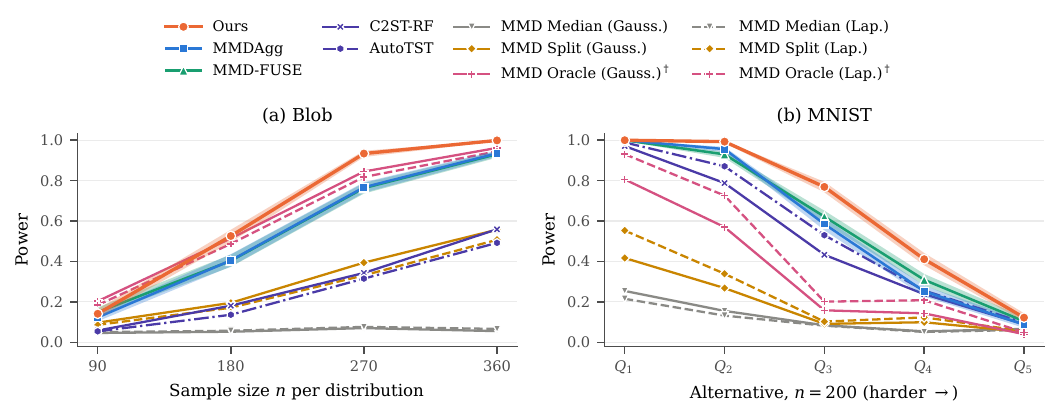}
  \caption{Power (rejection rate under the alternative) over 1000 repetitions at level $\kappa = 0.05$.
    \textbf{(a)}~Blob, as a function of the sample size $n$ per distribution.
    \textbf{(b)}~MNIST with $n = 200$, under alternatives $Q_1$--$Q_5$ of increasing difficulty.
    Shaded regions are 95\% Wilson confidence intervals for Ours, MMDAgg and MMD-Fuse.
    Solid and dashed lines denote Gaussian and Laplace kernels, respectively, for the single-kernel MMD tests.
    C2ST-RF (random-forest classifier two-sample test) and AutoTST (non-neural tabular models) train on half
    of the sample and test on the other half.
    $^\dagger$Oracle methods use $n$ additional samples for kernel selection.}
  \label{fig:power}
\end{figure*}

\begin{figure*}[t]
  \centering
  \includegraphics[width=\textwidth]{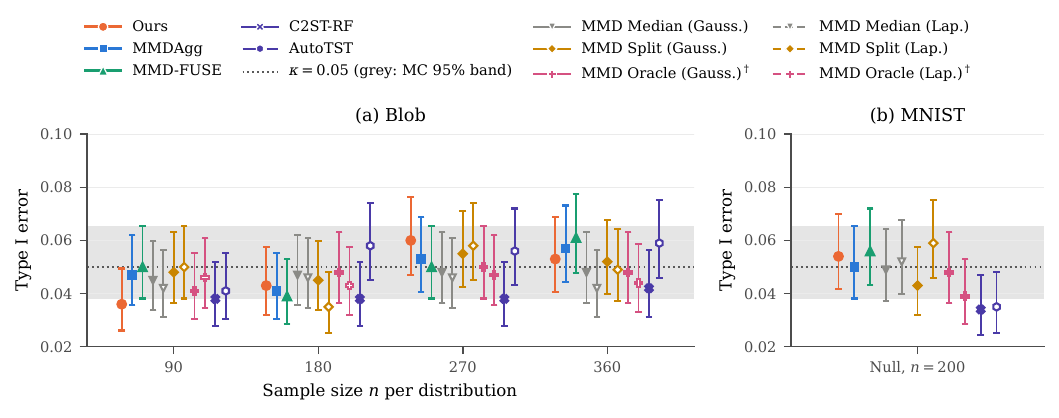}
  \caption{Type~I error (rejection rate under the null) over 1000 repetitions at level $\kappa = 0.05$.
    \textbf{(a)}~Blob, as a function of the sample size $n$ per distribution.
    \textbf{(b)}~MNIST with $n = 200$.
    Error bars are 95\% Wilson confidence intervals. The dotted line marks $\kappa$, and the grey band is
    the range of rejection rates consistent with an exact level-$\kappa$ test over 1000 repetitions
    (95\% Monte Carlo interval). Filled and hollow markers
    denote Gaussian and Laplace kernels, respectively, for the single-kernel MMD tests (and C2ST-RF and
    AutoTST, respectively, for the classifier-based tests).
    $^\dagger$Oracle methods use $n$ additional samples for kernel selection.}
  \label{fig:typeI}
\end{figure*}


\begin{table*}
\centering
\small
\caption{Blob experiment: full results. Rejection rates over 1000 repetitions at $\kappa=0.05$. The 95\% Wilson interval half-width is at most $\pm0.014$ near $0.05$ and at most $\pm0.031$ elsewhere. Bold: highest power among methods that do not use extra data. $^{*}$: 95\% interval excludes $\kappa$ (above: liberal; below: conservative). $^\dagger$Oracle methods use $n$ additional samples for kernel selection. $^\ddagger$Classifier-based tests (random-forest C2ST; AutoTST with non-neural tabular models) train on half of the sample and test on the other half.}
\label{tab:blob_full}
\setlength{\tabcolsep}{3.5pt}
\begin{tabular}{lcccccccc}
\toprule
& \multicolumn{4}{c}{Type I error (Blob-S)} & \multicolumn{4}{c}{Power (Blob-D)} \\
\cmidrule(lr){2-5} \cmidrule(lr){6-9}
Method & $n=90$ & $n=180$ & $n=270$ & $n=360$ & $n=90$ & $n=180$ & $n=270$ & $n=360$ \\
\midrule
\textbf{Ours} & 0.036$^{*}$ & 0.043 & 0.060 & 0.053 & 0.141 & \textbf{0.526} & \textbf{0.934} & \textbf{0.999} \\
MMDAgg & 0.047 & 0.041 & 0.053 & 0.057 & 0.123 & 0.405 & 0.767 & 0.935 \\
MMD-Fuse & 0.050 & 0.039 & 0.050 & 0.061 & \textbf{0.149} & 0.403 & 0.763 & 0.931 \\
\midrule
C2ST-RF$^\ddagger$ & 0.038 & 0.038 & 0.038 & 0.042 & 0.058 & 0.181 & 0.343 & 0.559 \\
AutoTST$^\ddagger$ & 0.041 & 0.058 & 0.056 & 0.059 & 0.054 & 0.136 & 0.315 & 0.492 \\
\midrule
MMD Median (Gauss.) & 0.045 & 0.047 & 0.048 & 0.048 & 0.046 & 0.052 & 0.070 & 0.055 \\
MMD Median (Lap.) & 0.042 & 0.046 & 0.046 & 0.042 & 0.050 & 0.057 & 0.076 & 0.066 \\
MMD Split (Gauss.) & 0.048 & 0.045 & 0.055 & 0.052 & 0.097 & 0.195 & 0.394 & 0.557 \\
MMD Split (Lap.) & 0.050 & 0.035$^{*}$ & 0.058 & 0.049 & 0.086 & 0.171 & 0.331 & 0.507 \\
\midrule
MMD Oracle (Gauss.)$^\dagger$ & 0.041 & 0.048 & 0.050 & 0.048 & 0.203 & 0.514 & 0.845 & 0.962 \\
MMD Oracle (Lap.)$^\dagger$ & 0.046 & 0.043 & 0.047 & 0.044 & 0.184 & 0.487 & 0.818 & 0.945 \\
\bottomrule
\end{tabular}
\end{table*}
 
\begin{table*}
\centering
\small
\caption{MNIST experiment ($n=200$): full results. Rejection rates over 1000 repetitions at $\kappa=0.05$. The 95\% Wilson interval half-width is at most $\pm0.014$ near $0.05$ and at most $\pm0.031$ elsewhere. Bold: highest power among methods that do not use extra data. $^{*}$: 95\% interval excludes $\kappa$ (above: liberal; below: conservative). $^\dagger$Oracle methods use $n$ additional samples for kernel selection. $^\ddagger$Classifier-based tests (random-forest C2ST; AutoTST with non-neural tabular models) train on half of the sample and test on the other half.}
\label{tab:mnist_full}
\begin{tabular}{lcccccc}
\toprule
& Type I error & \multicolumn{5}{c}{Power} \\
\cmidrule(lr){2-2} \cmidrule(lr){3-7}
Method & Null & $Q_1$ & $Q_2$ & $Q_3$ & $Q_4$ & $Q_5$ \\
\midrule
\textbf{Ours} & 0.054 & \textbf{1.000} & \textbf{0.993} & \textbf{0.769} & \textbf{0.411} & \textbf{0.122} \\
MMDAgg & 0.050 & \textbf{1.000} & 0.956 & 0.585 & 0.252 & 0.087 \\
MMD-Fuse & 0.056 & \textbf{1.000} & 0.930 & 0.622 & 0.309 & 0.104 \\
\midrule
C2ST-RF$^\ddagger$ & 0.034$^{*}$ & 0.971 & 0.788 & 0.433 & 0.238 & 0.089 \\
AutoTST$^\ddagger$ & 0.035$^{*}$ & 0.989 & 0.871 & 0.530 & 0.256 & 0.096 \\
\midrule
MMD Median (Gauss.) & 0.049 & 0.254 & 0.155 & 0.084 & 0.054 & 0.066 \\
MMD Median (Lap.) & 0.052 & 0.215 & 0.132 & 0.082 & 0.051 & 0.061 \\
MMD Split (Gauss.) & 0.043 & 0.417 & 0.268 & 0.091 & 0.099 & 0.051 \\
MMD Split (Lap.) & 0.059 & 0.553 & 0.339 & 0.102 & 0.123 & 0.055 \\
\midrule
MMD Oracle (Gauss.)$^\dagger$ & 0.048 & 0.806 & 0.570 & 0.158 & 0.143 & 0.040 \\
MMD Oracle (Lap.)$^\dagger$ & 0.039 & 0.931 & 0.726 & 0.201 & 0.208 & 0.049 \\
\bottomrule
\end{tabular}
\end{table*}

All tests control the Type~I error: every rejection rate under the null lies within Monte
Carlo error of $\kappa$ or below it (Figure~\ref{fig:typeI}).
On Blob (Figure~\ref{fig:power}a), every test has power below $0.21$ at $n = 90$.
From $n = 180$ onwards, our test is the most powerful among those that use no extra data,
ahead of MMDAgg and MMD-FUSE by about $12$ points at $n = 180$ and $17$ points at $n = 270$;
at $n = 270$ it also exceeds the oracle MMD tests, which use twice as much data
($0.93$ against $0.85$).
On MNIST (Figure~\ref{fig:power}b), our test is at least as powerful as every competitor on
every alternative, with the largest gains on the intermediate ones: $0.77$ against $0.62$
for MMD-FUSE on $Q_3$, and $0.41$ against $0.31$ on $Q_4$.
The classifier tests, which train on half of the sample, trail the adaptive kernel tests
on Blob and our test on every MNIST alternative.

%% file: sections/06_related_work.tex
\section{Related Work}
\label{app:related}

\textbf{Two-sample tests with kernels and learned statistics:}
Kernel tests based on the maximum mean discrepancy (MMD) compare mean embeddings of the two distributions in a reproducing kernel Hilbert space \citep{gretton2012kernel}, and their power depends strongly on the kernel and on the representation of the data.
Most adaptive tests, therefore, learn part of the statistics from data, and they differ in how this learning is separated from calibration.
Sample-splitting methods fit the statistic on one part of the data and evaluate it on the other: kernels chosen to maximise a power criterion \citep{gretton_optimal_2012, sutherland_generative_2021}, deep kernels \citep{liu2020learning}, classifiers whose held-out accuracy or logits serve as the statistic \citep{lopezpaz2017revisiting, cheng2022c2st, kim_classification_2020}, and witness functions whose mean difference is tested on held-out data \citep{kubler_witness_2022, kubler2022autotst}.
Other methods avoid the split by learning representations on auxiliary data \citep{kirchler_two-sample_2020}, by accounting for kernel selection through selective inference \citep{kubler_learning_2020}, or by aggregating or fusing tests over a collection of kernels with jointly calibrated thresholds \citep{jmlr_mmdagg, biggs_mmd-fuse_2023}.
Our test learns nothing: its statistic is fixed once the schedule and the noise law are chosen, it uses the full sample in every channel, and it adapts by fusing permutation ranks over a deterministic grid of channels.
Rank fusion applies the soft-maximum aggregation of MMD-FUSE to permutation ranks rather than to normalised kernel statistics, so it inherits neither the statistics nor the power theorem of MMD-FUSE; its level and power are proved separately in \cref{sec:fusion,sec:adaptive}.

\textbf{Minimax testing:}
Minimax separation rates for nonparametric testing go back to \citet{ingster1993asymptotically}.
For two-sample testing of densities, permutation tests based on U-statistics attain the rate $n^{-2s/(4s+d)}$ and can be made adaptive to unknown smoothness at a small cost in power \citep{kim_minimax_2022}, and MMDAgg is adaptive up to an iterated-logarithmic factor over Sobolev balls on the density difference, the same type of class as ours \citep{jmlr_mmdagg}.
However, our method is grounded in the dynamical transport view of modern generative models.
The statistic is a nonlinear divergence between two frames of a generative path, with a label-dependent denominator, and the analysis must control it directly rather than through a linear kernel statistic.
As for other nonparametric tests, the rates are stated in the ambient dimension, and power against fair alternatives can deteriorate as the dimension grows \citep{ramdas_decreasing_2015}.

\textbf{Scores, divergences and dynamics:}
Score functions have long been used for model criticism, most prominently through the kernel Stein discrepancy for goodness-of-fit testing against a model with a known score \citep{liu_kernelized_2016, chwialkowski_kernel_2016}, and more recently for hypothesis testing and change-point detection when the scores of both hypotheses are known \citep{moushegian_diffusion-based_2025}.
The link between Gaussian smoothing, score differences and divergences goes back to de Bruijn's identity and its extension to relative entropy \citep{lyu_interpretation_2012}, and to the I-MMSE relation and its mismatched-estimation counterpart \citep{guo_mutual_2004, verdu_mismatched_2010}.
The same identities underlie likelihood bounds for diffusion models \citep{song_maximum_2021, kingma_variational_2021} and information-theoretic views of them \citep{kong_information-theoretic_2023}.
\Cref{thm:velocity-score} belongs to this family but lives on a different path: it expresses a Jensen--Shannon divergence between mirrored frames of the interpolant from $\rho_0$ to $\rho_1$ as a time integral of velocity and score defects, and its empirical version involves no learned fields.

\textbf{Generative models as tools for inference:}
Beyond synthesis, diffusion and flow models have been used for likelihood evaluation through the probability flow ODE \citep{song_score-based_2020}, for out-of-distribution detection \citep{graham_denoising_2023, liu_unsupervised_2023}, and for zero-shot classification by comparing denoising errors across class-conditional models \citep{li_your_2023, clark_text--image_2023}.
Paths of intermediate distributions have also been used to estimate density ratios, by telescoping over a sequence of bridges \citep{rhodes_telescoping_2020} or by integrating infinitesimal classifiers along a continuum \citep{choi_density_2022}; our statistic likewise reads evidence along a path, but it targets a calibrated test rather than a density-ratio estimate.
Independently trained diffusion and flow models learn remarkably similar fields \citep{kadkhodaie_generalization_2024, zhang_emergence_2024, briq_amazing_2025}, and paths can differ even when their endpoints look alike \citep{english_path_2026} but can signify generalisation \citep{english_is_2026}, which supports treating paths as statistical objects.
We use the interpolant between the two empirical measures, whose marginals and fields are available in closed form; closed-form fields of the same kind have been studied for flow matching on empirical data \citep{bertrand_closed-form_2025} but for learning a generative model serving as an objective function.

\textbf{Midpoint inference from flow models:}
Closest to our work is a recent line of research built on the zero-flow criterion: for rectified flow with independently coupled endpoints, the velocity field vanishes at $t = 1/2$ if and only if the two endpoint distributions coincide.
This criterion has been turned into a midpoint divergence for training one-step generators, including a generalisation to symmetric stochastic interpolants \citep{shlenskii_midpoint_2026}, and, concurrently with our work, into a two-sample test that learns a midpoint witness on a training split and evaluates it on held-out data \citep{wang_zero-flow_2026}.
These approaches and ours rest on the idea of symmetry albeit with specific formulations: when the endpoints of a symmetric interpolant are exchangeable, its law is invariant under the time reflection $t \mapsto 1 - t$, and are exploited differently as per their applications.
Midpoint criteria read the symmetry off the velocity at a single time, i.e. the midpoint, which must be learned from pairs of observations drawn from both populations.
At the midpoint, the velocity field constructed between the same distributions should be zero. In contrast, we instead compare the marginal densities at mirrored times away from the midpoint, which are available in closed form, admit permutation calibration with finite-sample validity, and lead to minimax guarantees.
The two views meet at the midpoint: as $u \to 0$, the divergence between the frames at $1/2 \pm u$ behaves like $u^2$ times $\frac{1}{2} \int (\nabla \cdot j_{1/2})^2 / \rho_{1/2}$, whereas the gain from predicting a direction-labelled midpoint velocity equals the kinetic energy $\int \abs{j_{1/2}}^2 / \rho_{1/2}$ of the midpoint velocity, a quantity of the kind that midpoint criteria target; we do not identify either quantity with our statistic.
Also concurrently, \citet{english_two-sample_2026} construct a Gaussian encoding
interpolant whose velocity vanishes at the midpoint by construction, and test the symmetry
of the path through a Jeffreys divergence. Unlike ours, these approaches train a neural
network, which requires splitting the sample into a training part, and typically a
validation part for hyperparameter tuning, so that only the remainder is used for testing,
at a cost in power; to our knowledge, neither establishes a separation rate. Conversely,
the noise designs under which our test attains the minimax rate (\cref{sec:minimax}) may also
benefit these methods.

%% file: sections/07_discussion.tex
\section{Discussion}
\label{sec:discussion}

Drawing on dynamical-transport generative models, we have shown that a generative process
built for synthesis can also be read for inference. A stochastic interpolant built directly
between two samples is invariant under time reflection whenever the two distributions
coincide, and its marginals are explicit mixtures over all cross pairs, so testing this
symmetry requires no learned model. For Gaussian noise, the reflected Jensen--Shannon
divergence pairs the reflection defects of the velocity field and of the score
(\cref{thm:velocity-score}), so the test compares transport dynamics rather than endpoints
alone. The law of the interpolant also governs power: with a narrow-plus-broad noise law,
the test attains the minimax separation rate $n^{-2s/(4s+d)}$ for $s > 3d/4$
(\cref{thm:minimax,thm:lower-bound}), and fusing channels over a dyadic grid of noise scales
adapts to unknown smoothness at an iterated-logarithmic cost (\cref{thm:adaptive}). Other
noise laws, such as Gaussian noise, can be added to the fused channels without losing this
guarantee, provided that at least one compact-plateau design is among them. Empirically,
the test matches or outperforms state-of-the-art kernel tests. We hope that this reading of
generative processes extends to other inference problems, such as conditional independence
testing.

\textbf{Limitations and open questions:} Our focus is on statistical guarantees, and our
experiments cover two standard benchmarks, Blob and downsampled MNIST.
Since our rates, like those of other nonparametric tests, are stated in the ambient
dimension, a broader empirical study, in particular in higher dimensions, is left for
future work. Furthermore, the divergence between two mixtures has no closed form, and integrating it by
Monte Carlo multiplies the cost of a kernel statistic by the number $J$ (\cref{sec:computation}); in exchange, the test requires no training,
so it can be calibrated by permutation on the full sample, unlike tests built on trained
neural networks. Reducing this cost is a promising direction for future work.

%% file: appendix/A_experiments.tex
\section{Experimental details and additional results}
\label{app:experiments}

\subsection{Evaluation protocol and sample budgets}
\label{app:exp-protocol}

We consider balanced two-sample problems with independent samples
$\Dzero=\{x_0^{(i)}\}_{i=1}^{n}$, $x_0^{(i)}\sim\rho_0$, and
$\Done=\{x_1^{(j)}\}_{j=1}^{n}$, $x_1^{(j)}\sim\rho_1$, so that $n_0=n_1=n$.
Throughout this appendix, $n$ denotes the sample size \emph{per distribution};
The pooled sample, therefore, contains $N=2n$ observations.
For each benchmark configuration, we estimate rejection probabilities over
$N_{\mathrm{rep}}=1000$ independent repetitions at level $\kappa=0.05$.
Null experiments are conducted separately with $\rho_1=\rho_0$.

Within each repetition, every method receives the same realised datasets $\Dzero$
and $\Done$. The median MMD tests, MMDAgg, MMD-FUSE, and our test use all $n$
observations per distribution for testing. The split MMD tests and both
classifier tests share one random partition of each dataset: the first
$\lfloor n/2\rfloor$ indices of an independently generated permutation form
the selection or training sample, and the remaining indices form the test
sample. No additional split is imposed on methods that do not require it.
The oracle MMD tests receive an independent additional sample of size $n$
from \emph{each} distribution for bandwidth selection and are evaluated on
the same original datasets $\Dzero,\Done$ as the other full-sample methods.
Consequently, oracle MMD uses $2n$ observations per distribution, or $4n$
pooled observations, and is reported as an extra-data reference.

\begin{table*}[t]
\centering
\caption{Sample allocation per distribution in one repetition.
Selection includes classifier training. Additional data are independent
of the common evaluation sample.}
\label{tab:exp-sample-budgets}
\begin{tabular}{lrrrr}
\toprule
Method & Selection & Testing & Additional & Total \\
\midrule
Our test, MMDAgg, MMD-FUSE & $0$ & $n$ & $0$ & $n$ \\
MMD, median & $0$ & $n$ & $0$ & $n$ \\
MMD, split & $\lfloor n/2\rfloor$ & $\lceil n/2\rceil$ & $0$ & $n$ \\
C2ST, AutoTST & $\lfloor n/2\rfloor$ & $\lceil n/2\rceil$ & $0$ & $n$ \\
MMD, oracle & $n$ & $n$ & $n$ & $2n$ \\
\bottomrule
\end{tabular}
\end{table*}

\subsection{Datasets}
\label{app:exp-datasets}

\paragraph{Blob:}
Following the Blob construction of \citet{liu2020learning}, we use nine
components with means
\[
 (\mu_1,\ldots,\mu_9)
 =\bigl((0,0),(0,1),(0,2),(1,0),(1,1),(1,2),(2,0),(2,1),(2,2)\bigr).
\]
The two populations are
\begin{align}
 \rho_0 &= \frac{1}{9}\sum_{j=1}^{9}
       \mathcal N(\mu_j,0.03\,\mathbb{I}_2), \label{eq:exp-blob-p}\\
 \rho_1^{\mathrm D} &= \frac{1}{9}\sum_{j=1}^{9}
       \mathcal N\!\left(\mu_j,
       \begin{pmatrix}0.03&\epsilon_j\\\epsilon_j&0.03\end{pmatrix}\right),
       \label{eq:exp-blob-q}
\end{align}
where
\[
 (\epsilon_1,\ldots,\epsilon_9)=
 (-0.020,-0.022,-0.024,-0.026,0,0.020,0.022,0.024,0.026).
\]
Thus, the alternative changes within-component correlations while preserving
component means, mixture weights, and coordinate variances. Blob-S denotes
the null $\rho_1=\rho_0$ and Blob-D denotes $\rho_1=\rho_1^{\mathrm D}$.
We use $n\in\{90,180,270,360\}$ per distribution.
Each observation's component label is drawn independently and uniformly from
$\{1,\ldots,9\}$; the realized component counts are not fixed.
The nominal counts $n/9\in\{10,20,30,40\}$ are expected counts per component.
Samples are used in their original coordinates, without clipping or
coordinate standardisation.

\paragraph{MNIST:}
We use the empirical MNIST \citep{lecun1998mnist} image pools employed by the benchmark loader.
Raw intensities are divided by $255$, and each $28\times28$ image is reduced
to $7\times7$ by averaging nonoverlapping $4\times4$ blocks and flattened
to a vector in $[0,1]^{49}$. No learned representation or additional feature
standardisation is used. Let $\mathcal U_S$ denote the empirical distribution
that is uniform over the cached images whose digit labels belong to $S$.
We set $\rho_0=\mathcal U_{\{0,\ldots,9\}}$ and define the alternatives
$\rho_1=Q_1,\ldots,Q_5$ in
Table~\ref{tab:exp-mnist-alternatives}.
Sampling uniformly over images retains the cache's empirical digit
frequencies; it does not enforce uniform weights across digit classes.
For every repetition, $n=200$ images are sampled independently with
replacement from each pool. Under the null, both samples are drawn from
the all-digit pool. Digit labels determine the pools but are not provided
as input features to any test or classifier.
The reported probabilities are therefore conditional on the common
empirical MNIST cache.

\begin{table}[t]
\centering
\caption{MNIST alternatives, ordered from fewer to more retained digit
classes. Each $Q_j$ is uniform over images in its indicated empirical pool.}
\label{tab:exp-mnist-alternatives}
\begin{tabular}{cl}
\toprule
Alternative & Retained digits \\
\midrule
$Q_1$ & $\{1,3,5,7,9\}$ \\
$Q_2$ & $\{0,1,3,5,7,9\}$ \\
$Q_3$ & $\{0,1,2,3,5,7,9\}$ \\
$Q_4$ & $\{0,1,2,3,4,5,7,9\}$ \\
$Q_5$ & $\{0,1,2,3,4,5,6,7,9\}$ \\
\bottomrule
\end{tabular}
\end{table}

\subsection{The implemented rank-fusion test}
\label{app:exp-ours}

The method reported as ``ours'' is the Gaussian-plus-baseline configuration.
It applies the rank-normalised fusion rule with noise
bank, retaining the protected compact-plateau channels and Gaussian
schedule--time channels. Alternatively, the auxiliary Laplace, product Student-$t$,
logistic, and generalised-normal channels can also be used.

\paragraph{Gaussian schedule--time channels:}
For a schedule with coefficients $\alpha_t$, $\beta_t=\alpha_{1-t}$, and $\gamma_t$,
and $z\sim\mathcal N(0,\mathbb{I}_d)$, the reflected pair uses
\[
 \alpha_t x_0+\beta_t x_1+\gamma_t z,
 \qquad
 \beta_t x_0+\alpha_t x_1+\gamma_t z.
\]
The nine schedules are given in Table~\ref{tab:exp-schedules}, yielding
$9$ Gaussian channels. Define
\[
 \bar\gamma_t=\tfrac12\sin^2(\pi t),\qquad
 \vartheta(t)=\frac{t^2}{t^2+(1-t)^2},\qquad
 S_5(t)=6t^5-15t^4+10t^3.
\]
The coefficient formulas specify the implemented finite collection; no
schedule derivatives or velocity-field training enter these experiments.
For each channel, the forward and reflected densities $\rhohat_t$ and $\rhohat_{1-t}$
are evaluated as
empirical convolution mixtures over all $n_0n_1=n^2$ cross-sample pairs.
Their divergence $\JS(\rhohat_t,\rhohat_{1-t})$ is computed by direct density evaluation and Monte
Carlo integration, without fitting a discriminator or density model.

\begin{table*}[t]
\centering
\small
\caption{Gaussian channel schedules; $\beta_t=\alpha_{1-t}$ in every row.}
\label{tab:exp-schedules}
\begin{tabular}{lll}
\toprule
Schedule & $\alpha_t$ & $\gamma_t$ \\
\midrule
Linear & $1-t$ & $\bar\gamma_t$ \\
Endpoint power, $2$ & $(1-t)^2$ & $\bar\gamma_t$ \\
Power ratio, $2$ & $\dfrac{(1-t)^2}{(1-t)^2+t^2}$ & $\bar\gamma_t$ \\
Quintic smoothstep & $1-S_5(t)$ & $\bar\gamma_t$ \\
Normalized sigmoid, $10$
 & $\dfrac12\!\left[1-\dfrac{\tanh(5(t-1/2))}{\tanh(5/2)}\right]$ & $\bar\gamma_t$ \\
$\ell_2$ normalized & $\dfrac{1-t}{\sqrt{(1-t)^2+t^2}}$ & $\bar\gamma_t$ \\
Damped trigonometric
 & $\sqrt{1-\bar\gamma_t^{\,2}}\cos(\pi t/2)$ & $\bar\gamma_t$ \\
Quartic encoder--decoder
 & $\cos^4(\pi t)\,\mathbf 1\{t<1/2\}$ & $\bar\gamma_t$ \\
Warped trigonometric, $2$
 & $\cos(\pi\vartheta(t)/2)$ & $\bar\gamma_{\vartheta(t)}$ \\
\bottomrule
\end{tabular}
\end{table*}

\paragraph{Protected baseline channels:}
The internal baseline is a compact-plateau bridge-channel grid, distinct
from the external MMD comparison methods. With
\[
 \sigma_n=8+\log n,\qquad
 \mathcal G_n=\{2^{-j}:j=0,\ldots,J_n\},\qquad
 J_n=\lceil \log_2(n)/(2d)\rceil,
\]
each $h\in\mathcal G_n$ defines the smaller root $t_h\in(0,1/2)$ of
$\sigma_n t_h(1-t_h)=h$, and uses coefficients
$(\alpha_{t_h},\beta_{t_h},\gamma_{t_h})=(1-t_h,t_h,h)$.
The compact-plateau mixture parameter is $\lambda=1/4$.
Equivalently, the perturbation $\gamma_{t_h}z_h=hz_h$ has density
$W_h=(1-\lambda)k_h+\lambda K=\tfrac34 k_h+\tfrac14 K$,
where $k$ is proportional to $(1-|u|^2)_+^2$ on the unit ball.
The radial density of $K$ is constant up to radius $A=2L+1$, then follows
the normalised cubic collar $(1-v)^2(1+2v)$ for
$v=(|x|-A)/(A/d)\in(0,1)$, and vanishes outside the collar.
This noise is not rescaled to unit covariance.

\paragraph{Integration, normalisation, and fusion:}
Each channel statistic $\widehat T_t=\JS(\rhohat_t,\rhohat_{1-t})$ is approximated with
$J=256$ Monte Carlo proposal points $U_1,\ldots,U_J$
per channel. For a channel with coefficients $(\alpha_t,\beta_t,\gamma_t)$, proposals are drawn
from the label-independent pooled mixture $q_N$, the law of
$\alpha_t V_k+\beta_t V_l+\gamma_t z$, where $V_k,V_l$ are independent draws from the
empirical pooled sample $\{V_1,\ldots,V_N\}$. A channel's proposal points are reused for its observed statistic and
all relabelings. We use $B=1999$ balanced label permutations, shared across
channels, and double-precision arithmetic.

Let $T_{bc}$ be the statistic of channel $c$ for row $b$, with $b=0$ the
observed labeling and $b=1,\ldots,B$ the randomized labelings. Inclusive
upper-tail ranks and the fusion score are
\begin{align}
 r_{bc}&=\frac{1}{B+1}\sum_{j=0}^{B}\mathbf 1\{T_{jc}\ge T_{bc}\},\\
 \Lambda_b&=\sum_c w_c r_{bc}^{-\tau},
       \qquad \tau=0.5.
\end{align}
The protected channels share total weight $1/2$ uniformly; the $9$
Gaussian channels share the remaining $1/2$ uniformly, so each has weight
$w_c=1/18$. These group weights are preserved when the Gaussian subset is
extracted from the full-bank statistic matrix.
The final permutation $p$-value is
\[
 \widehat p_{\mathrm{fuse}}=\frac{1}{B+1}\sum_{b=0}^{B}\mathbf 1\{\Lambda_b\ge \Lambda_0\},
\]
and the test rejects when $\widehat p_{\mathrm{fuse}}\le\kappa=0.05$.
This is a rank fusion of JS channels, rather than an invocation of the
published MMD-FUSE kernel-test implementation.

\subsection{Comparison methods}
\label{app:exp-baselines}

\paragraph{MMDAgg and MMD-FUSE:}
We call the authors' implementations of MMDAgg \citep{jmlr_mmdagg} and
MMD-FUSE \citep{biggs_mmd-fuse_2023} on the common full samples.
MMDAgg uses $B_1=B_2=1999$ resamples and $B_3=50$ bisection iterations;
for equal sample sizes, the configured implementation uses its wild-bootstrap
calibration. MMD-FUSE uses $1999$ permutations.
Kernel collections, bandwidth construction, and remaining tuning arguments
are left at the defaults of the installed implementations.
The two MMDAgg resampling banks are distinct, so these settings do not equate the total resampling work of MMDAgg and the permutation-based tests.

\paragraph{Single-kernel MMD:}
For each of the median, split, and oracle choices, we evaluate Gaussian
and Laplace kernels. The coordinate-wise base scales are
\[
 \nu_k=\max\{\operatorname{median}_{i<j}|V_{i,k}-V_{j,k}|,10^{-4}\},
 \qquad k=1,\ldots,d,
\]
computed from the indicated unlabeled pooled sample $\{V_i\}$. With multiplier $\theta$, the kernel conventions are
\begin{align}
 \psi^{\mathrm G}_{\theta,\nu}(x,y)
   &=\exp\!\left[-\sum_{k=1}^{d}
                    \frac{(x_k-y_k)^2}{\theta^2\nu_k^2}\right],\\
 \psi^{\mathrm L}_{\theta,\nu}(x,y)
   &=\exp\!\left[-\sum_{k=1}^{d}
                    \frac{|x_k-y_k|}{\theta\nu_k}\right].
\end{align}
The Gaussian convention has no factor $1/2$ in the exponent.
Median MMD uses $\theta=1$, estimating $\nu$ from the full test pool.

The split and oracle rules select a multiplier by maximising the
regularised empirical signal-to-noise criterion used in the MMDAgg
benchmark \citep{jmlr_mmdagg}. For selection samples of size $n_{\mathrm{sel}}$ per
distribution, define
\begin{gather*}
 \Psi_{ij}(\theta)=\psi_\theta(x_0^{(i)},x_0^{(j)})+\psi_\theta(x_1^{(i)},x_1^{(j)})
             -\psi_\theta(x_0^{(i)},x_1^{(j)})-\psi_\theta(x_0^{(j)},x_1^{(i)}),\\
 \overline\Psi_i(\theta)=\frac{1}{n_{\mathrm{sel}}}\sum_j \Psi_{ij}(\theta).
\end{gather*}
The selection criterion is
\[
 \frac{\displaystyle\frac{1}{n_{\mathrm{sel}}(n_{\mathrm{sel}}-1)}\sum_{i\ne j}\Psi_{ij}(\theta)}
 {\displaystyle\sqrt{4\!\left[
       \frac{1}{n_{\mathrm{sel}}}\sum_i\overline\Psi_i(\theta)^2
       -\left(\frac{1}{n_{\mathrm{sel}}}\sum_i\overline\Psi_i(\theta)\right)^2
       \right]+10^{-8}}}.
\]
The multiplier grids are $\{0.1,0.2,\ldots,1.0\}$ for Blob and
$\{2^{10},2^{11},\ldots,2^{20}\}$ for MNIST.
The Blob grid is an adaptation of the paper's synthetic-data grid.
For selection, the coordinate scales use the unlabeled union of the
selection and testing pools; only the selection labels enter the
signal-to-noise criterion. After selecting $\theta$, the coordinate scales
are recomputed on the testing pool, and the selected multiplier is applied
to those scales. Oracle selection uses its independent extra samples;
Split selection uses the common training half.

These six tests use the paired MMD statistic
$[n_{\mathrm{te}}(n_{\mathrm{te}}-1)]^{-1}\sum_{i\ne j}\Psi_{ij}$ on their testing samples of size $n_{\mathrm{te}}$ per distribution and
$500$ wild-bootstrap resamples with independent Rademacher signs.
Calibration includes the observed statistic, giving the upper-tail
$p$-value $(1+\sum_{b=1}^{500}\mathbf 1\{T_b\ge T_0\})/501$.
The methods are implemented as NumPy/SciPy  \citep{harris_array_2020,virtanen_scipy_2020} adaptations of the paper's
benchmark routines, with the above conventions stated explicitly.

\paragraph{C2ST with random forests:}
We use a non-neural variant of the classifier two-sample test of
\citet{lopezpaz2017revisiting}. A random forest \citep{breiman_random_2001} with $200$ trees,
\texttt{max\_features=sqrt}, and minimum
leaf size $1$ is trained on the common training half to predict
distribution membership (label $1$ for $\Dzero$, label $0$ for $\Done$). Held-out hard-label accuracy is
the test statistic. We use the exact balanced-label permutation tail
rather than a binomial approximation: conditional on $N_{\mathrm{te}}$ held out
observations, $n_+$ true positive labels, and $\widehat n_+$ predicted positive labels,
the overlap $O$ has null law $\operatorname{Hypergeom}(N_{\mathrm{te}},\widehat n_+,n_+)$.
The $p$-value is $\Pr(O\ge O_{\mathrm{obs}})$, since accuracy is an
increasing function of this overlap. Constant predictions give $p=1$.
The forest is not refitted during held-out calibration, and no held-out labels are used to tune the classifier.

\paragraph{AutoTST:}
We use the authors' \texttt{autotst} implementation
\citep{kubler2022autotst}, with an AutoGluon tabular witness fitted on
the same training half. Unlike C2ST's hard-label accuracy statistic,
AutoTST learns a squared-loss regression witness and tests its held-out
mean discrepancy using the package's permutation procedure.
The AutoML search is restricted to random forests (RF) \citep{breiman_random_2001}, extremely
randomised trees (XT) \citep{geurts_extremely_2006}, and LightGBM (GBM) \citep{ke_lightgbm_2017}, with neural model families
excluded. We request $60$ seconds of fitting per repetition, disable
bagging and stacking, and use one CPU worker per model.
Any intrinsic model-selection validation uses only training data.
The held-out calibration uses $10000$ permutations and returns the
package's $p$-value unchanged. This is a specified non-neural AutoTST
configuration, rather than an unrestricted reproduction of its default
model search. The fitting limit is an AutoML time request; prediction,
calibration, and framework overhead are additional.

\subsection{Reproducibility and Monte Carlo uncertainty}
\label{app:exp-reproducibility}

All experiments were run on a single node with two AMD EPYC 7313 16-core processors
(32 cores in total) and 503\,GiB of RAM; GPU computations used a single NVIDIA A100 80GB
PCIe GPU. MNIST \citep{lecun1998mnist} is distributed under the Creative Commons
Attribution-Share Alike~3.0 license. We use NumPy \citep{harris_array_2020} and SciPy
\citep{virtanen_scipy_2020} (BSD-3-Clause), scikit-learn \citep{pedregosa_scikit-learn_2011}
(BSD-3-Clause), AutoGluon \citep{erickson_autogluon-tabular_2020} (Apache-2.0), LightGBM
\citep{ke_lightgbm_2017} (MIT), and the authors' implementations of MMDAgg, MMD-FUSE and
AutoTST, all released under the MIT license.

Data generation, method randomisation, and the common training partition
use separate deterministic random-number streams indexed by dataset,
sample size, scenario, and repetition. The split stream is
shared by Gaussian and Laplace MMD-split, C2ST, and AutoTST.
Companion baseline runners replay the evaluation seeds and, when available,
the saved split indices from the original results. MNIST cache hashes,
sample hashes, split indices, method seeds, configuration metadata, and
replication-level decisions are recorded, allowing comparisons to be
matched at the repetition level. Installed package versions and the
official MMDAgg/MMD-FUSE function signatures are also recorded.
AutoML time budgets and backend scheduling can affect fitted models
even when the data, partitions, and random seeds are identical.

For method $a$, the estimated rejection probability is
\[
 \widehat\pi_a=\frac{1}{N_{\mathrm{rep}}}\sum_{k=1}^{N_{\mathrm{rep}}}
                       \mathbf 1\{\text{method }a\text{ rejects in repetition }k\}.
\]
The Monte Carlo standard error is approximately
$\sqrt{\widehat\pi_a(1-\widehat\pi_a)/N_{\mathrm{rep}}}$; at $\pi_a=0.05$ and
$N_{\mathrm{rep}}=1000$, it is approximately $0.0069$.
The result summaries use pointwise $95\%$ Wilson binomial intervals.
With $z=1.96$, their centre and half-width are
\[
 \frac{\widehat\pi_a+z^2/(2N_{\mathrm{rep}})}{1+z^2/N_{\mathrm{rep}}},\qquad
 \frac{z}{1+z^2/N_{\mathrm{rep}}}
 \sqrt{\frac{\widehat\pi_a(1-\widehat\pi_a)}{N_{\mathrm{rep}}}+\frac{z^2}{4N_{\mathrm{rep}}^2}}.
\]
These are pointwise Monte Carlo intervals, not simultaneous intervals
over all methods and scenarios. Null rejection frequencies assess
empirical calibration; they do not by themselves prove population-level
Type~I error control. Likewise, differences between paired methods should
be distinguished from formal significance claims about their power.